\documentclass{article}
\usepackage[T1]{fontenc}
\usepackage[utf8]{inputenc}
\usepackage{color}
\usepackage{mathtools}
\usepackage{url}
\usepackage{enumitem}
\usepackage{amsmath}
\usepackage{amsthm}
\usepackage{amssymb}
\usepackage{graphicx}
\usepackage[unicode=true,
 bookmarks=false,
 breaklinks=false,pdfborder={0 0 0},pdfborderstyle={},backref=false,colorlinks=true]
 {hyperref}
\hypersetup{
 pdfborderstyle={},linkcolor=blue, citecolor=blue}

\makeatletter
\theoremstyle{plain}
\newtheorem{assumption}{\protect\assumptionname}
\theoremstyle{plain}
\newtheorem{thm}{\protect\theoremname}
\theoremstyle{definition}
\newtheorem{defn}[thm]{\protect\definitionname}
\theoremstyle{remark}
\newtheorem{rem}[thm]{\protect\remarkname}
\theoremstyle{plain}
\newtheorem{prop}[thm]{\protect\propositionname}
\theoremstyle{plain}
\newtheorem{lem}[thm]{\protect\lemmaname}
\theoremstyle{plain}
\newtheorem{cor}[thm]{\protect\corollaryname}

\usepackage[nonatbib,final]{neurips_2021}

\allowdisplaybreaks[4]

\@ifundefined{showcaptionsetup}{}{%
 \PassOptionsToPackage{caption=false}{subfig}}
\usepackage{subfig}
\makeatother

\providecommand{\assumptionname}{Assumption}
\providecommand{\corollaryname}{Corollary}
\providecommand{\definitionname}{Definition}
\providecommand{\lemmaname}{Lemma}
\providecommand{\propositionname}{Proposition}
\providecommand{\remarkname}{Remark}
\providecommand{\theoremname}{Theorem}

\begin{document}
\title{Limiting fluctuation and trajectorial stability of multilayer neural
networks with mean field training}
\author{Huy Tuan Pham \\ Department of Mathematics, Stanford University \And 
Phan-Minh Nguyen \thanks{The author ordering is randomized.} \\ The
Voleon Group}
\maketitle
\begin{abstract}
The mean field theory of multilayer neural networks centers around
a particular infinite-width scaling, in which the learning dynamics
is shown to be closely tracked by the mean field limit. A random fluctuation
around this infinite-width limit is expected from a large-width expansion
to the next order. This fluctuation has been studied only in the case
of shallow networks, where previous works employ heavily technical
notions or additional formulation ideas amenable only to that case.
Treatment of the multilayer case has been missing, with the chief
difficulty in finding a formulation that must capture the stochastic
dependency across not only time but also depth.

In this work, we initiate the study of the fluctuation in the case
of multilayer networks, at any network depth. Leveraging on the neuronal
embedding framework recently introduced by Nguyen and Pham \cite{nguyen2020rigorous},
we systematically derive a system of dynamical equations, called the
\textsl{second-order mean field limit}, that captures the limiting
fluctuation distribution. We demonstrate through the framework the
complex interaction among neurons in this second-order mean field
limit, the stochasticity with cross-layer dependency and the nonlinear
time evolution inherent in the limiting fluctuation. A limit theorem
is proven to relate quantitatively this limit to the fluctuation realized
by large-width networks.

We apply the result to show a stability property of gradient descent
mean field training: in the large-width regime, along the training
trajectory, it progressively biases towards a solution with ``minimal
fluctuation'' (in fact, vanishing fluctuation) in the learned output
function, even after the network has been initialized at or has converged
(sufficiently fast) to a global optimum. This extends a similar phenomenon
previously shown only for shallow networks with a squared loss in
the empirical risk minimization setting, to multilayer networks with
a loss function that is not necessarily convex in a more general setting.
\end{abstract}

\section{Introduction\label{sec:Introduction}}

Recent literature has witnessed much interest and progresses in the
mean field theory of neural networks. In particular, it is shown that
under a suitable scaling, as the widths tend to infinity, the neural
network's learning dynamics converges to a nonlinear deterministic
limit, known as the mean field (MF) limit \cite{mei2018mean,nguyen2020rigorous}.
This line of works starts with analyses of the shallow case under
various settings and has led to a number of nontrivial exciting results
\cite{nitanda2017stochastic,mei2018mean,chizat2018,rotskoff2018neural,sirignano2018mean,javanmard2019analysis,rotskoff2019global,nitanda2020particle,wei2018margin,alex2019landscape,wojtowytsch2020convergence,lu2020mean,agazzi2020global,nguyen2021analysis}.
The generalization to multilayer neural networks, already much more
conceptually and technically challenging, has also been met with serious
efforts from different groups of authors, with various novel ideas
and insights \cite{nguyen2019mean,nguyen2020rigorous,pham2020note,araujo2019mean,sirignano2019mean,fang2020modeling}.

Since the MF limit is basically a first-order infinite-width approximation
of the neural network, it is natural to consider the next order term
in the expansion for a more faithful finite-width description. This
leads to the consideration of the fluctuation, up-scaled appropriately
with the widths, around the MF limit. On one hand, this fluctuation
should display MF interactions among neurons. On the other hand, it
is random due to the inherent stochasticity in the finite-width network
which, for instance, is randomly initialized and hence induces randomness
in the interactions among neurons.

For shallow networks, \cite{sirignano2020mean,rotskoff2018neural,chen2020dynamical}
has identified the limiting fluctuation in the form of a time-evolving
signed measure over weights, leading to a central limit theorem (CLT)
in the space of measures. In particular, \cite{sirignano2020mean}
pinpoints the fluctuation via a compactness argument in an appropriate
measure space. Avoiding the heavy technicality, \cite{chen2020dynamical}
realizes a neat trick, specific to shallow networks, where quantities
of interest are described via push-forward of the initialized values
of the weights.

There has been no similar attempt for the multilayer case, which faces
major obstacles in formulating the limiting fluctuation. Firstly,
existing formulations of the MF limit already move away from working
with measures over weights \cite{nguyen2019mean,nguyen2020rigorous,sirignano2019mean,fang2020modeling},
unless restricted conditions are assumed \cite{araujo2019mean}. This
is due to the complexity of MF interactions: the presence of middle
layers brings about simultaneous actions of multiple symmetry groups.
We face the same complexity when formulating the fluctuation. Secondly,
unlike shallow networks, the fluctuation in multilayer networks displays
stochastic dependency not only through time but also across layers.
Specifically the cross-layer dependency is propagated by both forward
and backward propagation signals, at any time instance.

\vspace{-3.5pt}

\paragraph*{Contributions.}

We tackle the challenge by leveraging on the neuronal embedding framework,
recently proposed by \cite{nguyen2020rigorous}. An important concept
that is supported by the neuronal embedding is the sampling of neurons.
We use this concept to formulate the limiting fluctuation via a decomposition
into two components with different roles in Section \ref{sec:2nd_order_MF}.
One component is a random process, which encodes CLT-like stochasticity
in the fluctuation of the sampled neurons around the MF ensemble.
The other component, named the \textsl{second-order mean field limit},
is a system of ordinary differential equations (ODEs), which displays
MF interactions in the deviation of the neural network under gradient
descent (GD) training from the sampled neurons.

This decomposition is an innovation of the work. Our formulation enjoys
the generality brought about by the neuronal embedding framework without
restrictive assumptions (e.g. i.i.d. initialization assumption made
in \cite{araujo2019mean,sirignano2019mean}) and faithfully describes
the expected stochastic dependency across time and depth. We prove
a limit theorem that establishes the connection with the fluctuation
in finite-sized networks \textbf{(Theorems \ref{thm:2nd_order_MF}
and \ref{thm:CLT})}. Unlike \cite{sirignano2020mean,chen2020dynamical},
this result is quantitative.

Using this formulation, in Section \ref{sec:Long-time-asymptotic-variance},
we show a trajectorial stability property in a large-width multilayer
network, particularly a variance reduction effect: GD training traverses
a solution trajectory that reduces and eventually eliminates the (width-corrected)
variance of the learned output function. That is, a bias towards ``minimal
fluctuation''. This holds even if the network is initialized at,
or if it is trained to converge sufficiently fast to, a global optimum
\textbf{(Theorems \ref{thm:variance-global-opt-init} and \ref{thm:variance-global-opt-fast})}.
The same phenomenon has been shown in the shallow case \cite{chen2020dynamical},
which requires a squared loss in the empirical risk minimization setup.
As demonstrated by our result, it is not an isolated phenomenon and
can indeed hold for a neural architecture where MF interactions are
more complex, a loss function that is not necessarily convex and a
training setup unrestricted to finitely many training data points.

\vspace{-3.5pt}

\paragraph*{Limitations and potential extensions.}

Let us finally mention a few limitations. Our work considers fully-connected
multilayer networks trained with GD, a setup much less general than
\cite{nguyen2020rigorous}. We expect certain extensions in this direction
are doable. We also do not treat stochastic GD and hence disregard
the stochasticity of data sampling, unlike \cite{sirignano2020mean}.
Given the broader literature on this subject (e.g. \cite{li2019stochastic}),
this extension is foreseeable. Here we focus instead on the stochasticity
that is inherent in the interactions among neurons, which is the more
interesting aspect of neural networks.

Our result on the output variance is specific to unregularized training,
unlike \cite{chen2020dynamical}, and also requires a sufficient global
convergence rate. While there have been several proven global convergence
guarantees for multilayer networks \cite{nguyen2020rigorous,pham2021global,pham2020note},
understanding of the convergence rate is still lacking. Even in the
shallow case, global convergence has bee\textcolor{black}{n studied
only for a type of sparsity-inducing regularization \cite{chizat2018,chizat2019sparse}.
Unless the convergence rate for multilayer networks is generally perilous
(an unlikely scenario in light of the experiments in \cite{nguyen2019mean}),
our result is expected to be relevant.}

\textcolor{black}{Our development is specific to the MF scaling. This
scaling allows for nonlinear feature learning, unlike the NTK scaling
\cite{jacot2018neural}. While there are other scalings that also
admit a certain sense of feature learning \cite{golikov2020dynamically,yang2021feature},
the standard parameterization in practice -- in the infinite-width
limit -- is known to degenerate into NTK-like behaviors, which are
not expected of practical finite-but-large-width neural networks \cite{mei2019mean,yang2021feature}.
In other words, }\textsl{\textcolor{black}{all}}\textcolor{black}{{}
infinite-width scalings that display feature learning are only proxies
of practical networks. This limitation motivates finite-width studies
as we pursue here.}

\paragraph*{\textcolor{black}{Notations.}}

\textcolor{black}{We use $K$ to denote a generic constant that may
change from line to line, and similarly $K_{u}$ for a finite constant
that may depend on some constant $u$. For simplicity, for $L$ the
network depth, we write $K$ in place of $K_{L}$. For $\mathbb{E}$
(resp. $\mathbf{E}$) being the expectation, we use $\mathbb{V}$
and $\mathbb{C}$ (resp. $\mathbf{V}$ and $\mathbf{C}$) for the
variance and covariance. We write $\mathbb{E}_{Z}$ for the expectation
w.r.t. data $Z=\left(X,Y\right)\sim{\cal P}$. We write $\partial_{i}f$
to denote the partial derivative w.r.t. the $i$-th variable of $f$.}

\section{Background on the MF limit for multilayer networks and assumptions\label{sec:Background}}

We describe the necessary backgrounds based on (and simplifying) the
work \cite{nguyen2020rigorous}. This section also introduces several
important notations, as well as the assumptions for our study of the
fluctuation.

\paragraph*{Multilayer network.}

Following \cite{nguyen2020rigorous}, we consider a $L$-layer fully-connected
neural network with widths $\left\{ N_{i}\right\} _{i\leq L}$ (for
$N_{L}=1$):
\begin{align}
\hat{{\bf y}}\left(t,x\right) & =\varphi_{L}\left(\mathbf{H}_{L}\left(t,1,x\right)\right),\label{eq:nn_FC_multilayer}\\
\mathbf{H}_{i}\left(t,j_{i},x\right) & =\mathbb{E}_{J}\left[{\bf w}_{i}\left(t,J_{i-1},j_{i}\right)\varphi_{i-1}\left({\bf H}_{i-1}\left(t,J_{i-1},x\right)\right)\right],\quad j_{i}\in\left[N_{i}\right],\;i=L,...,2,\nonumber \\
{\bf H}_{1}\left(t,j_{1},x\right) & =\left\langle {\bf w}_{1}\left(t,1,j_{1}\right),x\right\rangle ,\quad j_{1}\in\left[N_{1}\right],\nonumber 
\end{align}
in which $x\in\mathbb{X}=\mathbb{R}^{d}$ is the input, the weights
are ${\bf w}_{1}\left(t,1,j_{1}\right)\in\mathbb{W}_{1}=\mathbb{R}^{d}$
and ${\bf w}_{i}\left(t,j_{i-1},j_{i}\right)\in\mathbb{W}_{i}=\mathbb{R}$,
$\varphi_{i}:\;\mathbb{R}\to\mathbb{R}$ is the activation and $t\in\mathbb{T}=\mathbb{R}_{\geq0}$
the continuous time. Here we reserve the notation $J_{i}$ for the
random variable $J_{i}\sim{\rm Unif}\left(\left[N_{i}\right]\right)$
and we write $\mathbb{E}_{J}$ to denote the expectation w.r.t. these
random variables. For convenience, we take $j_{0}\in\left\{ 1\right\} $
and $N_{0}=1$.

Given an initialization $\mathbf{w}_{i}\left(0,\cdot,\cdot\right)$,
we train the network with gradient descent (GD) w.r.t. the loss ${\cal L}:\;\mathbb{R}\times\mathbb{R}\to\mathbb{R}_{\geq0}$
and the data $Z=\left(X,Y\right)\in\mathbb{X}\times\mathbb{R}$ drawn
from a training distribution ${\cal P}$:
\[
\partial_{t}{\bf w}_{i}\left(t,j_{i-1},j_{i}\right)=-\mathbb{E}_{Z}\left[\partial_{2}{\cal L}\left(Y,\hat{\mathbf{y}}\left(t,X\right)\right)\frac{\partial\hat{{\bf y}}\left(t,X\right)}{\partial{\bf w}_{i}\left(j_{i-1},j_{i}\right)}\right],\quad i\in\left[L\right],
\]
in which we define
\begin{align*}
\frac{\partial\hat{{\bf y}}\left(t,x\right)}{\partial{\bf H}_{i}\left(j_{i}\right)} & =\begin{cases}
\varphi_{L}'\left(\mathbf{H}_{L}\left(t,1,x\right)\right), & i=L,\\
\mathbb{E}_{J}\left[{\displaystyle \frac{\partial\hat{{\bf y}}\left(t,x\right)}{\partial{\bf H}_{i+1}\left(J_{i+1}\right)}}{\bf w}_{i+1}\left(t,j_{i},J_{i+1}\right)\varphi_{i}'\left({\bf H}_{i}\left(t,j_{i},x\right)\right)\right], & i<L,
\end{cases}\\
\frac{\partial\hat{{\bf y}}\left(t,x\right)}{\partial{\bf w}_{i}\left(j_{i-1},j_{i}\right)} & =\begin{cases}
{\displaystyle \frac{\partial\hat{{\bf y}}\left(t,x\right)}{\partial{\bf H}_{i}\left(j_{i}\right)}}\varphi_{i-1}\left({\bf H}_{i-1}\left(t,j_{i-1},x\right)\right), & i>1,\\
{\displaystyle \frac{\partial\hat{{\bf y}}\left(t,x\right)}{\partial{\bf H}_{1}\left(j_{1}\right)}}x, & i=1.
\end{cases}
\end{align*}
One may recognize that these ``derivative'' quantities are defined
in a perturbative fashion; for instance, $\frac{\partial\hat{{\bf y}}\left(t,x\right)}{\partial{\bf H}_{i}\left(j_{i}\right)}$
represents how $\hat{{\bf y}}\left(t,x\right)$ changes (rescaled
by widths) as one perturbs $\mathbf{H}_{i}\left(t,j_{i},x\right)$.

\paragraph*{Mean field limit.}

The MF limit is defined upon a given neuronal ensemble $\left(\Omega,{\cal F},P\right)=\prod_{i=0}^{L}\left(\Omega_{i},{\cal F}_{i},P_{i}\right)$
(in which $\Omega_{0}=\Omega_{L}=\left\{ 1\right\} $), which is a
product probability space. We reserve the notation $C_{i}$ for the
random variable $C_{i}\sim P_{i}$ and we write $\mathbb{E}_{C}$
to denote the expectation w.r.t. these random variables. The MF limit
that is associated with the network (\ref{eq:nn_FC_multilayer}) is
described by the evolution of $\left\{ w_{i}\left(t,\cdot,\cdot\right)\right\} _{i\in\left[L\right]}$,
given by the following MF ODEs:
\begin{align*}
\partial_{t}w_{i}\left(t,c_{i-1},c_{i}\right) & =-\mathbb{E}_{Z}\left[\partial_{2}{\cal L}\left(Y,\hat{y}\left(t,X\right)\right)\frac{\partial\hat{y}\left(t,X\right)}{\partial w_{i}\left(c_{i-1},c_{i}\right)}\right],\quad i\in\left[L\right],\;c_{i}\in\Omega_{i},\;c_{i-1}\in\Omega_{i-1},
\end{align*}
where $w_{i}:\,\mathbb{T}\times\Omega_{i-1}\times\Omega_{i}\to\mathbb{W}_{i}$
and $t\in\mathbb{T}$. Here we define the forward quantities:
\begin{align*}
\hat{y}\left(t,x\right) & =\varphi_{L}\left(H_{L}\left(t,1,x\right)\right),\\
H_{i}\left(t,c_{i},x\right) & =\mathbb{E}_{C}\left[w_{i}\left(t,C_{i-1},c_{i}\right)\varphi_{i-1}\left(H_{i-1}\left(t,C_{i-1},x\right)\right)\right],\qquad i=L,...,2,\\
H_{1}\left(t,c_{1},x\right) & =\left\langle w_{1}\left(t,c_{1}\right),x\right\rangle ,
\end{align*}
and the backward quantities:
\begin{align*}
\frac{\partial\hat{y}\left(t,x\right)}{\partial H_{i}\left(c_{i}\right)} & =\begin{cases}
\varphi_{L}'\left(H_{L}\left(t,1,x\right)\right), & i=L,\\
\mathbb{E}_{C}\left[{\displaystyle \frac{\partial\hat{y}\left(t,x\right)}{\partial H_{i+1}\left(C_{i+1}\right)}}w_{i+1}\left(t,c_{i},C_{i+1}\right)\varphi_{i}'\left(H_{i}\left(t,c_{i},x\right)\right)\right], & i<L,
\end{cases}\\
\frac{\partial\hat{y}\left(t,x\right)}{\partial w_{i}\left(c_{i-1},c_{i}\right)} & =\begin{cases}
{\displaystyle \frac{\partial\hat{y}\left(t,x\right)}{\partial H_{i}\left(c_{i}\right)}}\varphi_{i-1}\left(H_{i-1}\left(t,c_{i-1},x\right)\right), & i>1,\\
{\displaystyle \frac{\partial\hat{y}\left(t,x\right)}{\partial H_{1}\left(c_{1}\right)}}x, & i=1.
\end{cases}
\end{align*}
The evolution with time $t$ of the MF limit is generally complex
due to the nonlinear activations.

\paragraph*{Sampling of neurons, neuronal embedding, and the neural net-MF limit
connection.}

A priori there is no connection between the MF limit and the neural
network (\ref{eq:nn_FC_multilayer}); they are two separate self-contained
time-evolving systems, despite the similarity in their definitions.
To formalize their connection as in \cite{nguyen2020rigorous}, we
describe the neuron sampling procedure. In particular, we independently
sample $\left\{ C_{i}\left(j_{i}\right)\right\} _{j_{i}\in\left[N_{i}\right]}\sim_{{\rm i.i.d.}}P_{i}$
for $i=1,...,L$. The samples $C_{i}\left(j_{i}\right)$ should be
thought of as ``sampled neurons''. Let us use $\mathbf{E}$ to denote
the expectation w.r.t. the sampling randomness. Although \cite{nguyen2020rigorous}
considers more general sampling rule, this suffices for our purposes.

Now suppose that we are given functions $w_{i}^{0}:\;\Omega_{i-1}\times\Omega_{i}\to\mathbb{W}_{i}$
and set the initializations for the neural network and the MF limit:
\[
\mathbf{w}_{i}\left(0,j_{i-1},j_{i}\right)=w_{i}^{0}\left(C_{i-1}\left(j_{i-1}\right),C_{i}\left(j_{i}\right)\right),\quad w_{i}\left(0,\cdot,\cdot\right)=w_{i}^{0}\left(\cdot,\cdot\right)\quad\forall i\in\left[L\right].
\]
Then one lets them evolve according to their own respective dynamics.
The sampling of neurons hence allows to couple them on the basis of
the tuple $\big(\Omega,{\cal F},P,\left\{ w_{i}^{0}\right\} _{i\in\left[L\right]}\big)$
-- known as \textit{the neuronal embedding}, the key idea in \cite{nguyen2020rigorous}.
Note that this does not pose a major limitation on the type of initializations
for the neural network (\ref{eq:nn_FC_multilayer}); indeed it allows
both i.i.d. and non-i.i.d. initializations since we have freedom to
choose the neuronal embedding, as studied in \cite{nguyen2020rigorous,pham2020note}.

Once the coupling is done, \cite{nguyen2020rigorous} obtains the
following result, which realizes the connection in mathematical terms.
For any finite terminal time $T\in\mathbb{T}$, we have \footnote{While \cite{nguyen2020rigorous} treats stochastic GD in discrete
time, this result is implicit in the proof.}:
\begin{equation}
\max_{i\in\left[L\right]}\mathbb{E}_{J}\Big[\sup_{t\leq T}\left|\mathbf{w}_{i}\left(t,J_{i-1},J_{i}\right)-w_{i}\left(t,C_{i-1}\left(J_{i-1}\right),C_{i}\left(J_{i}\right)\right)\right|^{2}\Big]^{1/2}=\tilde{O}(N_{\min}^{-c_{*}})\label{eq:first_order_MF}
\end{equation}
with high probability, where $c_{*}>0$ is a universal constant, $N_{\min}=\min_{i\in\left[L-1\right]}N_{i}$,
$N_{\max}=\max_{i\in\left[L\right]}N_{i}$ and $\tilde{O}$ hides
the dependency on $T$, $L$ and $\log N_{\max}$. This result is
akin to the law of large numbers (LLN). In words, it says that for
large (ideally infinite) widths, the evolution of the network (\ref{eq:nn_FC_multilayer})
can be tracked closely by the MF limit. This result is fundamental
and gives a useful suggestion that is to study the neural network
via analyzing the MF limit. For example, \cite{nguyen2020rigorous}
uses it to study the optimization efficiency of the neural network.

The sampling of neurons is inspired by the propagation of chaos argument
\cite{sznitman1991topics}, but is not a mere proof device. Indeed,
in this work, we again make a crucial use of this sampling in the
formulation of the limiting fluctuation (Section \ref{subsec:Fluctuation-sampling}).

\paragraph*{Set of assumptions.}

We consider the following assumptions for the rest of the paper:
\begin{assumption}
\label{Assump:Assumption_1}We make the following assumptions:

\vspace{-7pt}
\begin{itemize}[wide, labelwidth=!, labelindent=0pt, noitemsep]
\item (Regularity) $\varphi_{i}$ is $K$-bounded for $i\in\left[L-1\right]$;
$\varphi_{i}'$ and $\varphi_{i}''$ are $K$-bounded and $K$-Lipschitz
for $i\in\left[L\right]$; $\partial_{2}{\cal L}$ and $\partial_{2}^{2}{\cal L}$
are $K$-bounded and $K$-Lipschitz in the second variable; $\left|X\right|\leq K$
almost surely (a.s.).
\item (Sub-Gaussian initialization) $\sup_{m\geq1}m^{-1/2}\mathbb{E}_{C}\big[\big|w_{i}^{0}\left(C_{i-1},C_{i}\right)\big|^{m}\big]^{1/m}\leq K$
for any $i\in\left[L\right]$.
\item (Measurability) $L^{2}(P_{i})$ is separable for any $i\in\left[L\right]$.
\item (Constant hidden widths) $N_{1}=...=N_{L-1}=N$ (with ideally $N\to\infty$).
\end{itemize}
\end{assumption}

An example for the regularity assumption is $\varphi_{L}\left(u\right)=u$
the identity, $\varphi_{i}=\tanh$ for $i\in\left[L-1\right]$ and
${\cal L}\left(u_{1},u_{2}\right)$ is a smooth version of the Huber
loss. It is noteworthy that ${\cal L}$ does not need to be convex.
The initialization assumption is also mild; it allows most common
i.i.d. initializations as well as non-i.i.d. schemes in \cite{nguyen2020rigorous,pham2020note}.
These assumptions satisfy the conditions in these works. Measurability
assumption is a technical condition needed for well-defined-ness of
the fluctuation; it is not conceptually restrictive and can accommodate
results in \cite{nguyen2020rigorous,pham2020note}.

The constant width assumption aligns with our interest in the infinite-width
regime. Our development can be extended easily to the proportional
scaling, where $N_{i}=\left\lfloor \rho_{i}N\right\rfloor $ for some
constants $\rho_{1},...,\rho_{L-1}>0$ with $N\to\infty$. These scalings
are relevant to practical setups and also to yield interesting results.
Specifically once some widths are much larger than others, we expect
the fluctuation to be dominated by a subset of layers.

\section{Limit system for the fluctuation around the MF limit\label{sec:2nd_order_MF}}

We describe our formulation for the fluctuation. As introduced, it
is composed of two components: a stochastic process that induces Gaussian
CLT-like behavior, and another process -- called the second-order
MF limit -- that signifies the MF interactions at the fluctuation
level. This is done on the basis of the sampling of neurons, on top
of the neuronal embedding, introduced in Section \ref{sec:Background}.

We introduce the first component in Section \ref{subsec:Gaussian}
and then the second-oder MF limit in \ref{subsec:Second-order-MF-limit}.
Together the two components are connected with the fluctuation around
the MF limit realized by a large-width network in Section \ref{subsec:Fluctuation-sampling},
to give firstly the analog of (\ref{eq:first_order_MF}) and secondly
a limit theorem on the fluctuation distribution. We end with a discussion
in Section \ref{subsec:A-heuristic-derivation}. Proofs are deferred
to the appendix.

\subsection{Gaussian component $\tilde{G}$\label{subsec:Gaussian}}

Recall the sampling of neurons described in Section \ref{sec:Background}.
Given the MF limit, we define the following random quantities:
\begin{align*}
\tilde{y}\left(t,x\right) & =\varphi_{L}(\tilde{H}_{L}\left(t,1,x\right)),\\
\tilde{H}_{i}\left(t,c_{i},x\right) & =\mathbb{E}_{J}\left[w_{i}\left(t,C_{i-1}\left(J_{i-1}\right),c_{i}\right)\varphi_{i-1}(\tilde{H}_{i-1}\left(t,C_{i-1}\left(J_{i-1}\right),x\right))\right],\qquad i=L,...,2,\\
\tilde{H}_{1}\left(t,c_{1},x\right) & =\left\langle w_{1}\left(t,c_{1}\right),x\right\rangle ,\\
\frac{\partial\tilde{y}\left(t,x\right)}{\partial\tilde{H}_{i}\left(c_{i}\right)} & =\begin{cases}
\varphi_{L}'(\tilde{H}_{L}\left(t,1,x\right)), & i=L,\\
\mathbb{E}_{J}\left[{\displaystyle \frac{\partial\tilde{y}\left(t,x\right)}{\partial\tilde{H}_{i+1}\left(C_{i+1}\left(J_{i+1}\right)\right)}}w_{i+1}\left(t,c_{i},C_{i+1}\left(J_{i+1}\right)\right)\varphi_{i}'(\tilde{H}_{i}\left(t,c_{i},x\right))\right], & i<L,
\end{cases}\\
\frac{\partial\tilde{y}\left(t,x\right)}{\partial w_{i}\left(c_{i-1},c_{i}\right)} & =\begin{cases}
{\displaystyle \frac{\partial\tilde{y}\left(t,x\right)}{\partial\tilde{H}_{i}\left(c_{i}\right)}}\varphi_{i-1}(\tilde{H}_{i-1}\left(t,c_{i-1},x\right)), & i>1,\\
{\displaystyle \frac{\partial\tilde{y}\left(t,x\right)}{\partial\tilde{H}_{1}\left(c_{1}\right)}}x, & i=1.
\end{cases}
\end{align*}
These quantities are analogues of the forward and backward MF quantities,
but with $\mathbb{E}_{C_{i}}$ being replaced by an empirical average
over $\left\{ C_{i}\left(j_{i}\right)\right\} _{j_{i}\in\left[N_{i}\right]}$.

Now we define 
\begin{align*}
\tilde{G}^{y}\left(t,x\right) & =\sqrt{N}\left(\tilde{y}\left(t,x\right)-\hat{y}\left(t,x\right)\right),\\
\tilde{G}_{i}^{w}\left(t,c_{i-1},c_{i},x\right) & =\sqrt{N}\left(\frac{\partial\tilde{y}\left(t,x\right)}{\partial w_{i}\left(c_{i-1},c_{i}\right)}-\frac{\partial\hat{y}\left(t,x\right)}{\partial w_{i}\left(c_{i-1},c_{i}\right)}\right),\quad i=1,...,L,
\end{align*}
and let $\tilde{G}$ denote the collection of these functions. Notice
that the randomness of $\tilde{G}$ is induced by the samples $C_{i}(j_{i})$;
thus we write the expectation w.r.t. $\tilde{G}$ via $\mathbf{E}$.
Intuitively $\tilde{G}$ is the fluctuation of the sampled neurons
around the MF limit and thus should converge to a Gaussian process.
We show that this is indeed the case in the following mode of convergence.
\begin{defn}
\label{def:conv-poly-moment}Given joint distributions over a sequence
(in $N$) of functions $\tilde{G}_{i,N}:\mathbb{T}\times\Omega_{i-1}\times\Omega_{i}\times\mathbb{X}\to\mathbb{W}_{i}$
for $i\in\left[L\right]$, we say that $\tilde{G}_{N}=\{\tilde{G}_{i,N}\}_{i\in\left[L\right]}$
\textit{converges $G$-polynomially in moment} (as $N\to\infty$)
to $\underline{G}$ if for any finite collection in $\ell$ of square-integrable
$f_{\ell}:\mathbb{T}\times\Omega_{i(\ell)-1}\times\Omega_{i(\ell)}\times\mathbb{X}\to\mathbb{W}_{i(\ell)}$
that is continuous in time and integers $\alpha_{\ell},\beta_{\ell}\geq1$,
\[
\sup_{t\left(\ell\right)\leq T\;\forall\ell}\Big|\mathbb{E}\Big[\sideset{}{_{\ell}}\prod\langle f_{\ell},\tilde{G}_{i(\ell),N}^{\alpha_{\ell}}\rangle_{t\left(\ell\right)}^{\beta_{\ell}}\Big]-\mathbb{E}\Big[\sideset{}{_{\ell}}\prod\langle f_{\ell},\underline{G}_{i(\ell)}^{\alpha_{\ell}}\rangle_{t\left(\ell\right)}^{\beta_{\ell}}\Big]\Big|=O_{D,T}\big(\sup_{t\leq T}\max_{\ell}\|f_{\ell}\|_{t}^{D}\big)N^{-1/8},
\]
for $D=\sum_{\ell}\alpha_{\ell}\beta_{\ell}$ and a constant terminal
time $T$. Here we define:
\[
\left\langle f,h\right\rangle _{t}=\mathbb{E}_{Z,C}\left[f\left(t,C_{i-1},C_{i},X\right)h\left(t,C_{i-1},C_{i},X\right)\right]
\]
for any two such functions $f$ and $h$, and $\left\Vert f\right\Vert _{t}^{2}=\left\langle f,f\right\rangle _{t}$.
As a convenient convention, if $f:\mathbb{T}\times\Omega_{0}\times\Omega_{1}\times\mathbb{X}\to\mathbb{W}_{1}$
is tested against $\tilde{G}_{1,N}$, we define $\langle f,\tilde{G}_{1,N}\rangle_{t}$
and $\left\Vert f\right\Vert _{t}$ using Euclidean inner product
and the Euclidean norm.
\end{defn}

Note that this mode of convergence gives a quantitative rate in $N$.
\begin{thm}[Gaussian component $\tilde{G}$]
\label{thm:G_tilde}For any terminal time $T\geq0$, $\tilde{G}$
converges $G$-polynomially in moment to $\underline{G}=\left\{ \underline{G}^{y},\;\underline{G}_{i}^{w},\;i\in\left[L\right]\right\} $,
which is a collection of centered Gaussian processes.
\end{thm}

The covariance structure of $\underline{G}$ in Theorem \ref{thm:G_tilde}
is explicit but lengthy; we refer to the appendix for full details.
Here let us examine the variance of $\underline{G}^{y}$, given recursively
by:
\begin{align*}
v_{2}\left(t,c_{2},x\right) & =\mathbb{V}_{C}\left[w_{2}\left(t,C_{1},c_{2}\right)\varphi_{1}\left(H_{1}\left(t,C_{1},x\right)\right)\right],\\
v_{i}\left(t,c_{i},x\right) & =\mathbb{E}_{C}\big[\left|w_{i}\left(t,C_{i-1},c_{i}\right)\varphi_{i-1}'\left(H_{i-1}\left(t,C_{i-1},x\right)\right)\right|^{2}v_{i-1}\left(t,C_{i-1},x\right)\big]\\
 & \quad+\mathbb{V}_{C}\left[w_{i}\left(t,C_{i-1},c_{i}\right)\varphi_{i-1}\left(H_{i-1}\left(t,C_{i-1},x\right)\right)\right],\quad i=3,...,L,\\
\mathbf{E}\big[\left|\underline{G}^{y}\left(t,x\right)\right|^{2}\big] & =\left|\varphi_{L}'\left(H_{L}\left(t,1,x\right)\right)\right|^{2}v_{L}\left(t,1,x\right).
\end{align*}
The recursive structure shows that $\underline{G}^{y}$ compounds
the stochasticity propagated forwardly through the depth of the multilayer
network. From the appendix, a similar observation can be made in the
backward propagation direction.

\subsection{Second-order MF limit\label{subsec:Second-order-MF-limit}}

We introduce a system of ODE's for a dynamics that represents MF interactions
at the fluctuation level. As we shall see later, this system gives
a limiting description of the the deviation of the neural network
from the sampled neurons under GD training.

Let us denote by $G$ a collection of functions $G^{y}:\mathbb{T}\times\mathbb{X}\to\mathbb{R}$,
$G_{i}^{w}:\mathbb{T}\times\Omega_{i-1}\times\Omega_{i}\times\mathbb{X}\to\mathbb{W}_{i}$
for $i\in\left[L\right]$. For $p\geq1$, let 
\[
{\cal G}=\left\{ G:\;\left\Vert G\right\Vert _{T}<\infty\;\forall T\geq0\right\} ,\quad\left\Vert G\right\Vert _{T,2p}^{2p}=\sup_{t\le T}\mathbb{E}_{Z,C}\Big[|G^{y}\left(t,X\right)|^{2p}+\sideset{}{_{i=1}^{L}}\sum|G_{i}^{w}(t,C_{i-1},C_{i},X)|^{2p}\Big],
\]
with $\left\Vert G\right\Vert _{T}\equiv\left\Vert G\right\Vert _{T,2}$.
We define processes $R_{i}:{\cal G}\times\mathbb{T}\times\Omega_{i-1}\times\Omega_{i}\to\mathbb{W}_{i}$
for $i\in\left[L\right]$ as follows. The processes $R_{i}$ are initialized
at $R_{i}\left(\cdot,0,\cdot,\cdot\right)=0$ for all $i$, and for
each fixed $G\in{\cal G}$, solve the differential equations:
\begin{align}
 & \partial_{t}R_{i}\left(G,t,c_{i-1},c_{i}\right)=\nonumber \\
 & -\mathbb{E}_{Z}\bigg[\frac{\partial\hat{y}\left(t,X\right)}{\partial w_{i}\left(c_{i-1},c_{i}\right)}\partial_{2}^{2}{\cal L}\left(Y,\hat{y}\left(t,X\right)\right)\bigg(G^{y}\left(t,X\right)+\sum_{r=1}^{L}\mathbb{E}_{C}\bigg[R_{r}\left(G,t,C_{r-1},C_{r}\right)\frac{\partial\hat{y}\left(t,X\right)}{\partial w_{r}\left(C_{r-1},C_{r}\right)}\bigg]\bigg)\bigg]\nonumber \\
 & -\mathbb{E}_{Z}\bigg[\partial_{2}{\cal L}\left(Y,\hat{y}\left(t,X\right)\right)\mathbb{E}_{C}\bigg[\sum_{r=1}^{L}\bigg[R_{r}\left(G,t,a,b\right)\frac{\partial^{2}\hat{y}\left(t,X\right)}{\partial w_{r}\left(a,b\right)\partial w_{i}\left(c_{i-1},c_{i}\right)}\bigg]_{a\coloneqq C_{r-1},\;b\coloneqq C_{r}}\bigg]\bigg]\nonumber \\
 & -\mathbb{E}_{Z}\bigg[\partial_{2}{\cal L}\left(Y,\hat{y}\left(t,X\right)\right)G_{i}^{w}\left(t,c_{i-1},c_{i},X\right)\bigg].\label{eq:ODE_2ndMF}
\end{align}
Here $\frac{\partial^{2}\hat{y}\left(t,X\right)}{\partial w_{r}\left(a,b\right)\partial w_{i}\left(c_{i-1},c_{i}\right)}$
is defined in a perturbative fashion similar to those in Section \ref{sec:Background}
and is hence self-explanatory; we give the full explicit definitions
in the appendix. We shall write $R_{i}\left(G,t\right)=R_{i}\left(G,t,\cdot,\cdot\right)$
for brevity. Let $R$ denote the collection $\left\{ R_{i}\right\} _{i\in\left[L\right]}$.
\begin{thm}
\label{thm:well-posed}Under Assumption \ref{Assump:Assumption_1},
for any $\epsilon>0$ and $G\in{\cal G}$ with $\|G\|_{T,2+\epsilon}<\infty$,
there exists a unique solution $t\mapsto R_{i}\left(G,t,\cdot,\cdot\right)\in L^{2}\left(P_{i-1}\times P_{i}\right)$
which is continuous in time. Furthermore, for each $t$, $R\left(G,t\right)$
is a continuous linear functional in $G$.
\end{thm}

The process $R$ is called the \textsl{second-order MF limit} for
a reason we shall see in Section \ref{subsec:A-heuristic-derivation}.
\begin{rem}
Care should be taken in the derivation of the ``second derivative''
quantity in Eq. (\ref{eq:ODE_2ndMF}). To illustrate, consider a simplified
problem of taking a perturbation w.r.t. $w_{r}\left(a,b\right)$ of
the quantity
\[
g\left(w_{i}\left(c_{i-1},c_{i}\right),\mathbb{E}_{C}\left[f_{1}\left(w_{i}\left(C_{i-1},c_{i}\right)\right)\right],\mathbb{E}_{C}\left[f_{2}\left(w_{i}\left(c_{i-1},C_{i}\right)\right)\right],\mathbb{E}_{C}\left[f_{3}\left(w_{i}\left(C_{i-1},C_{i}\right)\right)\right]\right)\equiv g\left(c_{i-1},c_{i}\right).
\]
For instance, when $r=i$, we have:
\begin{align*}
 & \bigg[R_{i}\left(G,t,a,b\right)\frac{\partial g\left(c_{i-1},c_{i}\right)}{\partial w_{r}\left(a,b\right)}\bigg]_{a\coloneqq C_{i-1},\;b\coloneqq C_{i}}\\
 & =R_{i}\left(G,t,c_{i-1},c_{i}\right)\partial_{1}g\left(c_{i-1},c_{i}\right)+R_{i}\left(G,t,C_{i-1},c_{i}\right)\partial_{2}g\left(c_{i-1},c_{i}\right)\partial f_{1}\left(w_{i}\left(C_{i-1},c_{i}\right)\right)\\
 & \quad+R_{i}\left(G,t,c_{i-1},C_{i}\right)\partial_{3}g\left(c_{i-1},c_{i}\right)\partial f_{2}\left(w_{i}\left(c_{i-1},C_{i}\right)\right)+R_{i}\left(G,t,C_{i-1},C_{i}\right)\partial_{4}g\left(c_{i-1},c_{i}\right)\partial f_{3}\left(w_{i}\left(C_{i-1},C_{i}\right)\right).
\end{align*}
\end{rem}

\subsection{Fluctuation around the MF limit via sampling of neurons\label{subsec:Fluctuation-sampling}}

We are interested in a characterization of the following quantity,
which represents the deviation of the weights of the neural network
under GD training from the sampled neurons:
\[
{\bf R}_{i}\left(t,j_{i-1},j_{i}\right)=\sqrt{N}\left({\bf w}_{i}\left(t,j_{i-1},j_{i}\right)-w_{i}\left(t,C_{i-1}\left(j_{i-1}\right),C_{i}\left(j_{i}\right)\right)\right).
\]
Let $\mathbf{R}$ denote the collection $\left\{ \mathbf{R}_{i}\right\} _{i\in\left[L\right]}$.
We now state our first main result of the section.
\begin{thm}[Second-order MF]
\label{thm:2nd_order_MF}Under Assumption \ref{Assump:Assumption_1},
we have, for any $t\leq T$,
\[
\mathbf{E}\mathbb{E}_{J}\Big[\left|{\bf R}_{i}\left(t,J_{i-1},J_{i}\right)-R_{i}(\tilde{G},t,C_{i-1}\left(J_{i-1}\right),C_{i}\left(J_{i}\right))\right|^{2}\Big]\leq K_{T}/N.
\]
\end{thm}

Upon the correct choice $\tilde{G}$, one deduces from the theorem
the second-order expansion:
\[
{\bf w}_{i}\left(t,j_{i-1},j_{i}\right)\approx w_{i}\left(t,C_{i-1}\left(j_{i-1}\right),C_{i}\left(j_{i}\right)\right)+N^{-1/2}R_{i}(\tilde{G},t,C_{i-1}\left(j_{i-1}\right),C_{i}\left(j_{i}\right)),
\]
where the first-order term is directly from Eq. (\ref{eq:first_order_MF}).
Importantly this relation is on the basis of the sampled neurons $\left\{ C_{i}\left(j_{i}\right):\;j_{i}\in\left[N_{i}\right],\;i\in\left[L\right]\right\} $.
The dynamics of $R_{i}$ at the $i$-th layer is dependent on all
$\left\{ R_{k}\right\} _{k\in\left[L\right]}$ and $G^{y}$, exemplifying
the cross-layer interaction of fluctuations in multilayer networks.

Although Theorem \ref{thm:2nd_order_MF} involves $\tilde{G}$, as
shown in the derivation of Section \ref{subsec:A-heuristic-derivation},
we actually do not utilize the specific structure of $\tilde{G}$.
Incidentally we study the limiting structure of $\tilde{G}$ separately
in Theorem \ref{thm:G_tilde}. In other words, our decomposition of
the fluctuation via $\tilde{G}$ and $R$ allows relatively independent
treatments of the two components (which have different natures, to
be discussed in Section \ref{subsec:A-heuristic-derivation}).

Combining the two theorems, we obtain the following CLT for the limiting
output fluctuation.
\begin{thm}[CLT for output function]
\label{thm:CLT}Under Assumption \ref{Assump:Assumption_1}, the
fluctuation $\sqrt{N}(\hat{{\bf y}}(t,x)-\hat{y}(t,x))$ converges
weakly to the Gaussian process $\hat{G}$ indexed by $\mathbb{T}\times\mathbb{X}$:
\[
\hat{G}\left(t,x\right)=\sum_{i=1}^{L}\mathbb{E}_{C}\left[R_{i}(\underline{G},t,C_{i-1},C_{i})\frac{\partial\hat{y}(t,x)}{\partial w_{i}(C_{i-1},C_{i})}\right]+\underline{G}^{y}(t,x),
\]
where $\underline{G}$ is the Gaussian process described in Theorem
\ref{thm:G_tilde}. Specifically, for any integer $m\geq1$, $t\leq T$,
and $1$-bounded $h:\mathbb{X}\to\mathbb{R}$,
\[
\left|\mathbf{E}\mathbb{E}_{Z}\Big[h(X)\big(\sqrt{N}(\hat{{\bf y}}(t,X)-\hat{y}(t,X))\big)^{m}\Big]-\mathbf{E}\mathbb{E}_{Z}\Big[h(X)\big(\hat{G}\left(t,X\right)\big)^{m}\Big]\right|\leq K_{T,m}N^{-1/8+o(1)}.
\]
\end{thm}

We note in passing the technicality we devise to prove Theorem \ref{thm:CLT}.
As noted, Theorems \ref{thm:2nd_order_MF} and \ref{thm:G_tilde}
separate treatments of $R$ and $\tilde{G}$, but recall we are interested
in $R(\tilde{G},t)$. This necessitates a way to describe them jointly.
We do so via the following notion of convergence, extending Definition
\ref{def:conv-poly-moment}.
\begin{defn}
\label{def:conv-joint-moment}Suppose we are given a sequence (in
$N$) of functions $\tilde{G}_{i,N}:\mathbb{T}\times\Omega_{i-1}\times\Omega_{i}\times\mathbb{X}\to\mathbb{W}_{i}$
for $i\in\left[L\right]$ such that $\tilde{G}_{N}=\{\tilde{G}_{i,N}\}_{i\in\left[L\right]}$
converges $G$-polynomially in moment to $G$. Given further $\tilde{R}_{i,N}:\mathbb{T}\times\Omega_{i-1}\times\Omega_{i}\times\mathbb{X}\to\mathbb{W}_{i}$
defined jointly with $\tilde{G}_{i,N}$ and $\tilde{R}_{N}=\{\tilde{R}_{i,N}\}_{i\in\left[L\right]}$,
we say that $(\tilde{R}_{N},\tilde{G}_{N})$ \textit{converges $G$-polynomially
in moment and $R$-linearly in moment} (as $N\to\infty$) to $\left(R,G\right)$
if for any $f_{\ell}$, $h_{\ell}$ that are $\mathbb{T}\times\Omega_{i(\ell)-1}\times\Omega_{i(\ell)}\times\mathbb{X}\to\mathbb{W}_{i(\ell)}$
mappings and continuous in time and any integers $\alpha_{\ell},\beta_{\ell}\geq1$,
\begin{align*}
 & \sup_{t\left(\ell\right)\leq T\;\forall\ell}\Big|\mathbb{E}\Big[\sideset{}{_{\ell}}\prod\langle f_{\ell},\tilde{G}_{i(\ell),N}^{\alpha_{\ell}}\rangle_{t\left(\ell\right)}^{\beta_{\ell}}\sideset{}{_{\ell}}\prod\langle h_{\ell},\tilde{R}_{i(\ell),n}\rangle_{t\left(\ell\right)}^{\beta_{\ell}}\Big]-\mathbb{E}\Big[\sideset{}{_{\ell}}\prod\langle f_{\ell},G_{i(\ell)}^{\alpha_{\ell}}\rangle_{t\left(\ell\right)}^{\beta_{\ell}}\sideset{}{_{\ell}}\prod\langle h_{\ell},R_{i(\ell)}\rangle_{t\left(\ell\right)}^{\beta_{\ell}}\Big]\Big|\\
 & =O_{D}\big(\sup_{t\leq T}\max_{\ell}\|f_{\ell}\|_{t}^{D},\;\sup_{t\leq T}\max_{\ell}\|h_{\ell}\|_{t}^{D}\big)N^{-1/8+o(1)},
\end{align*}
for $D=\sum_{\ell}\beta_{\ell}\alpha_{\ell}+\sum_{\ell}\beta_{\ell}$.
\end{defn}

\begin{prop}
\label{prop:conv-G-implies-R}Assume that $\tilde{G}$ converges $G$-polynomially
in moment to $G=\left\{ G^{y},\;G_{i}^{w},\;i\in\left[L\right]\right\} $.
Then $(R(\tilde{G},\cdot),\tilde{G})$ converges $G$-polynomially
in moment and $R$-linearly in moment to $(R(G,\cdot),G)$.
\end{prop}

\subsection{A heuristic derivation and discussion\label{subsec:A-heuristic-derivation}}

\paragraph*{A heuristic argument.}

We wish to give a heuristic to derive the relation between ${\bf R}_{i}$
and $R_{i}(\tilde{G},t)$, i.e. Theorem \ref{thm:2nd_order_MF}. Suppose
we look at the $i$-th layer:
\begin{align}
 & \partial_{t}{\bf R}_{i}\left(t,j_{i-1},j_{i}\right)\nonumber \\
 & =\sqrt{N}\partial_{t}{\bf w}_{i}\left(t,j_{i-1},j_{i}\right)-\sqrt{N}\partial_{t}w_{i}\left(t,C_{i-1}\left(j_{i-1}\right),C_{i}\left(j_{i}\right)\right)\nonumber \\
 & =\sqrt{N}\mathbb{E}_{Z}\left[\partial_{2}{\cal L}\left(Y,\hat{y}\left(t,X\right)\right)\frac{\partial\hat{y}\left(t,X\right)}{\partial w_{i}\left(C_{i-1}\left(j_{i-1}\right),C_{i}\left(j_{i}\right)\right)}\right]-\sqrt{N}\mathbb{E}_{Z}\left[\partial_{2}{\cal L}\left(Y,\hat{\mathbf{y}}\left(t,X\right)\right)\frac{\partial\hat{{\bf y}}\left(t,X\right)}{\partial{\bf w}_{i}\left(j_{i-1},j_{i}\right)}\right]\nonumber \\
 & =\mathbb{E}_{Z}\bigg[\underbrace{\sqrt{N}\Big(\partial_{2}{\cal L}\left(Y,\hat{y}\left(t,X\right)\right)-\partial_{2}{\cal L}\left(Y,\hat{\mathbf{y}}\left(t,X\right)\right)\Big)}_{A_{1}}\cdot\frac{\partial\hat{y}\left(t,X\right)}{\partial w_{i}\left(C_{i-1}\left(j_{i-1}\right),C_{i}\left(j_{i}\right)\right)}\bigg]\nonumber \\
 & \quad+\mathbb{E}_{Z}\bigg[\partial_{2}{\cal L}\left(Y,\hat{\mathbf{y}}\left(t,X\right)\right)\cdot\sqrt{N}\bigg(\frac{\partial\hat{y}\left(t,X\right)}{\partial w_{i}\left(C_{i-1}\left(j_{i-1}\right),C_{i}\left(j_{i}\right)\right)}-\frac{\partial\hat{{\bf y}}\left(t,X\right)}{\partial{\bf w}_{i}\left(j_{i-1},j_{i}\right)}\bigg)\bigg].\label{eq:heuristic_partial_i_R}
\end{align}
Let us zoom into $A_{1}$, with some care not to eliminate the fluctuation
of interest:
\begin{align*}
A_{1} & \approx\sqrt{N}\partial_{2}^{2}{\cal L}\left(Y,\hat{y}\left(t,X\right)\right)\left(\hat{y}\left(t,X\right)-\hat{\mathbf{y}}\left(t,X\right)\right)\\
 & =\partial_{2}^{2}{\cal L}\left(Y,\hat{y}\left(t,X\right)\right)\left(-\tilde{G}^{y}\left(t,X\right)+\sqrt{N}\left(\tilde{y}\left(t,X\right)-\hat{\mathbf{y}}\left(t,X\right)\right)\right).
\end{align*}
Now observe that $\tilde{y}\left(t,X\right)-\hat{\mathbf{y}}\left(t,X\right)$
is a difference between the sampled neurons and the neural network,
and hence one should be able to express it in terms of ${\bf R}$:
\begin{align*}
 & \sqrt{N}\left(\tilde{y}\left(t,X\right)-\hat{\mathbf{y}}\left(t,X\right)\right)\\
 & \approx\sum_{r=1}^{L}\mathbb{E}_{J}\bigg[\frac{\partial\tilde{y}\left(t,X\right)}{\partial w_{r}\left(C_{r-1}\left(J_{r-1}\right),C_{r}\left(J_{r}\right)\right)}\cdot\sqrt{N}\left(w_{r}\left(t,C_{r-1}\left(J_{r-1}\right),C_{r}\left(J_{r}\right)\right)-{\bf w}_{r}\left(t,J_{r-1},J_{r}\right)\right)\bigg]\\
 & =-\sum_{r=1}^{L}\mathbb{E}_{J}\bigg[\frac{\partial\tilde{y}\left(t,X\right)}{\partial w_{r}\left(C_{r-1}\left(J_{r-1}\right),C_{r}\left(J_{r}\right)\right)}\cdot{\bf R}_{r}\left(t,J_{r-1},J_{r}\right)\bigg].
\end{align*}
Furthermore in the limit $N\to\infty$, we expect from the LLN, for
a test function $f$,
\begin{equation}
\mathbb{E}_{J}\left[f\left(C_{r-1}\left(J_{r-1}\right),C_{r}\left(J_{r}\right)\right)\right]\approx\mathbb{E}_{C}\left[f\left(C_{r-1},C_{r}\right)\right].\label{eq:heuristic_sketchy}
\end{equation}
As such, one can identify the term that involves $A_{1}$ in Eq. (\ref{eq:heuristic_partial_i_R})
for $\partial_{t}{\bf R}_{i}\left(t,j_{i-1},j_{i}\right)$ with the
first term in Eq. (\ref{eq:ODE_2ndMF}) for $\partial_{t}R_{i}\left(G,t,c_{i-1},c_{i}\right)$.
One can derive similarly for the rest of the terms. With these, one
arrives at $\partial_{t}{\bf R}_{i}\left(t,j_{i-1},j_{i}\right)\approx\partial_{t}R_{i}(\tilde{G},t,C_{i-1}\left(j_{i-1}\right),C_{i}\left(j_{i}\right))$.
Recalling that ${\bf R}_{i}\left(0,\cdot,\cdot\right)=0=R_{i}(\cdot,0,\cdot,\cdot)$,
one arrives at the conclusion of Theorem \ref{thm:2nd_order_MF}.

We make two comments. Firstly, the sampling of neurons makes transparent
the derivation of $R$ as the limit of $\mathbf{R}$: it allows to
compare one-to-one a ``neuron'' of the former to a neuron of the
latter. This demonstrates an advantage of the neuronal embedding framework,
which easily accommodates the sampling of neurons in the multilayer
setup. We also see that the only involvement of $\tilde{G}$ in this
derivation is via its definition.

Secondly, $R$ arises (by replacing $\mathbf{R}$) at steps that invoke
(\ref{eq:heuristic_sketchy}). This LLN-type nature signifies MF interactions
among neurons at the fluctuation level, and hence the name second-order
MF limit. In contrast, the component $\tilde{G}$ displays a Gaussian
CLT-type nature, evident from its definition.

\paragraph*{Technical difficulties.}

So far we have seen $R$ captures the difference between the neural
network and the sampled neurons. Recall the other component in the
fluctuation is the process $\tilde{G}$, which captures the $\sqrt{N}$-scaled
difference between the sampled neurons and the MF limit. Taking into
account the whole system of fluctuation, we note a major technical
subtlety: the aforementioned two differences captured by $R$ and
$\tilde{G}$ share the same source of randomness $\left\{ C_{i}\left(j_{i}\right):\;j_{i}\in\left[N_{i}\right],\;i\in\left[L\right]\right\} $.

We mention two particular complications that arise from this subtlety.
The first complication is how one should define the process $R$ in
relation with $\tilde{G}$, given that $R$ is meant to be the infinite-width
limit of $\mathbf{R}$ while $\mathbf{R}$ and $\tilde{G}$ are stochastically
coupled. Our solution is to let $t\mapsto R(G,t)$ defined for any
$G\in{\cal G}$, not restricted to only $\tilde{G}$. This streamlines
the definition of the second-order MF limit and allows to separate
our treatments of $R$ and $\tilde{G}$, as evident from Theorems
\ref{thm:2nd_order_MF} and \ref{thm:G_tilde}.

The second complication lies with Eq. (\ref{eq:heuristic_sketchy}).
In fact, to arrive at the desired conclusion, we require the function
$f$ to depend on $\tilde{G}$. This is because the main object of
interest is $R(\tilde{G},t)$, even though we have treated $R$ and
$\tilde{G}$ separately. In a nutshell, $f$ shares randomness with
$C_{r}\left(j_{r}\right)$ and $C_{r-1}\left(j_{r-1}\right)$, $j_{r}\in\left[N_{r}\right]$,
$j_{r-1}\in\left[N_{r-1}\right]$. The random variable on the left-hand
side of Eq. (\ref{eq:heuristic_sketchy}) is therefore complex, and
it is highly questionable whether this equation should hold. The analysis
becomes delicate; without taking into account this shared randomness,
the derivation would be a mere heuristic. As we present in the proof,
we verify this equation for a relevant set of functions $f$.

\section{Asymptotic variance of the output function\label{sec:Long-time-asymptotic-variance}}

We study the following width-scaled asymptotic variance quantity:
\[
V^{*}\left(t\right)=\lim_{N\to\infty}\mathbf{E}\mathbb{E}_{Z}\big[\big|\sqrt{N}(\hat{{\bf y}}(t,X)-\hat{y}(t,X))\big|^{2}\big]=\lim_{N\to\infty}\mathbf{E}\mathbb{E}_{Z}\big[\big|\sqrt{N}(\hat{{\bf y}}(t,X)-{\bf E}[\hat{{\bf y}}(t,X)])\big|^{2}\big].
\]
The second equality holds and the limits in $N$ exist by Theorems
\ref{thm:2nd_order_MF} and \ref{thm:CLT}. We would like to understand
$V^{*}(t)$ in the long-time horizon, and specifically in a situation
of considerable interest where the MF limit converges to a global
optimum as $t\to\infty$. To that end, we assume the following.
\begin{assumption}
\label{Assump:Assumption_2}We assume $\mathbb{E}_{Z}\left[\partial_{2}^{2}{\cal L}\left(Y,f\left(X\right)\right)\middle|X\right]=K$
a positive constant a.s. for any $f$ in which $\mathbb{E}_{Z}\left[\partial_{2}{\cal L}\left(Y,f\left(X\right)\right)\middle|X\right]=0$
a.s.
\end{assumption}

Note that convexity of the loss is not required. For example, the
assumption holds for ${\cal L}\left(y,y'\right)=\ell\left(y-y'\right)$
for any quasi-convex smooth function $\ell$ when $Y=y(X)$ is a deterministic
function of $X$. It is driven by our interest in what happens at
global convergence: \cite{nguyen2020rigorous,pham2020note} show convergence
to a global optimum is attainable (with suitable initialization strategies)
if either (case 1) ${\cal L}$ is convex in the second variable, or
(case 2) $\partial_{2}{\cal L}\left(y,y'\right)=0$ implies ${\cal L}\left(y,y'\right)=0$
(recalling ${\cal L}\geq0$) and $Y=y\left(X\right)$. In both cases,
one can find reasonable loss functions that satisfy Assumption \ref{Assump:Assumption_2}.
We also note in both cases, any $f$ with $\mathbb{E}_{Z}\left[\partial_{2}{\cal L}\left(Y,f\left(X\right)\right)\middle|X\right]=0$
a.s. is a global optimizer.

Our first result indicates that even if one initializes the network
at a global optimum and hence there is no evolution with time at the
MF limit level, GD training still helps by reducing the output variance
$V^{*}\left(t\right)$ with time at the fluctuation level, in the
large-width regime.
\begin{thm}
\label{thm:variance-global-opt-init}Suppose at initialization $\mathbb{E}_{Z}\left[\partial_{2}{\cal L}\left(Y,\hat{y}\left(0,X\right)\right)\middle|X\right]=0$
a.s. Under Assumptions \ref{Assump:Assumption_1} and \ref{Assump:Assumption_2},
$V^{*}\left(t\right)$ is non-increasing and $V^{*}\left(t\right)\to0$
as $t\to\infty$.
\end{thm}

This variance reduction effect continues to hold in the long-time
horizon if, instead of initializing at a global optimum, we assume
to have global convergence at a sufficiently fast rate.
\begin{thm}
\label{thm:variance-global-opt-fast}Assume $\int_{0}^{\infty}t^{2+\delta}\mathbb{E}_{X}\left[\left|\mathbb{E}_{Z}[\partial_{2}{\cal L}(Y,\hat{y}(t,X))|X]\right|^{2}\right]^{1/2}dt<\infty$
for some $\delta>0$. Under Assumptions \ref{Assump:Assumption_1}
and \ref{Assump:Assumption_2}, $V^{*}\left(t\right)\to0$ as $t\to\infty$.
\end{thm}

These results suggest that the variance reduction effect of GD in
MF training is a phenomenon more general than the case of shallow
networks with a squared loss and finitely many training points in
\cite{chen2020dynamical}. Specifically it is shown for multilayer
networks with a loss that is not necessarily convex with arbitrary
training data distribution ${\cal P}$. Though we discuss in the context
of a global optimum as our main interest, the theorems apply to any
stationary point where Assumption \ref{Assump:Assumption_2} holds.

Let us briefly discuss the proof. For simplicity, consider the context
of Theorem \ref{thm:variance-global-opt-init}. Recall from Theorem
\ref{thm:CLT} the Gaussian process $\hat{G}$ that the output fluctuation
converges to as $N\to\infty$; we thus would like to show $\mathbf{E}\mathbb{E}_{Z}\big[\big|\hat{G}\left(t,X\right)\big|^{2}\big]\to0$
as $t\to\infty$. In this case, one can show that for $\frak{A}:L^{2}({\cal P})\to L^{2}({\cal P})$
a linear operator defined by 
\[
(\frak{A}f)(x)=-c\sum_{i}\mathbb{E}_{Z,C}\left[\frac{\partial\hat{y}(0,x)}{\partial w_{i}(C_{i-1},C_{i})}\frac{\partial\hat{y}(0,X)}{\partial w_{i}(C_{i-1},C_{i})}f(X)\right],
\]
with a constant $c>0$, we have $\partial_{t}\hat{G}\left(t,\cdot\right)=\frak{A}\hat{G}\left(t,\cdot\right)$.
In particular, $\frak{A}$ is self-adjoint and has non-positive eigenvalues
and as such, if at initialization $\hat{G}\left(0,\cdot\right)$ lies
in the range of $\frak{A}$, we reach the desired conclusion. This
is simple in the shallow case. Indeed note that $\hat{G}\left(0,\cdot\right)=G^{y}\left(0,\cdot\right)$
by the assumption on initialization of Theorem \ref{thm:variance-global-opt-init}.
Assuming $\varphi_{L}\left(u\right)=u$ for simplicity, in the shallow
case where $\hat{y}\left(0,x\right)=\mathbb{E}_{C_{1}}\left[w_{2}\left(0,C_{1},1\right)\varphi_{1}\left(\left\langle w_{1}\left(0,1,C_{1}\right),x\right\rangle \right)\right]$,
we have $G^{y}\left(0,\cdot\right)$ is zero-mean Gaussian with the
covariance $\mathbf{E}[G^{y}\left(0,x\right)G^{y}\left(0,x'\right)]$
equal to
\[
\mathbb{C}_{C_{1}}\left[w_{2}\left(0,C_{1},1\right)\varphi_{1}\left(\left\langle w_{1}\left(0,1,C_{1}\right),x\right\rangle \right);\;w_{2}\left(0,C_{1},1\right)\varphi_{1}\left(\left\langle w_{1}\left(0,1,C_{1}\right),x'\right\rangle \right)\right].
\]
Immediately from this simple covariance structure, we see that $x\mapsto G^{y}\left(0,x\right)$
lies in the span of $\left\{ x\mapsto\varphi_{1}\left(\left\langle w_{1}\left(0,1,c_{1}\right),x\right\rangle \right):\;c_{1}\in{\rm supp}\left(P_{1}\right)\right\} $,
i.e. the desired conclusion since $\frac{\partial\hat{y}(0,x)}{\partial w_{2}(C_{1},1)}=\varphi_{1}\left(\left\langle w_{1}\left(0,1,C_{1}\right),x\right\rangle \right)$.
In the multilayer case, this is no longer obvious due to the complex
covariance structure as pointed out in Section \ref{sec:2nd_order_MF},
and yet interestingly we show that it still indeed holds.

The proof of Theorem \ref{thm:variance-global-opt-fast} leverages
on Theorem \ref{thm:variance-global-opt-init}. In particular, implicit
in this proof is the fact that the variance reduction effect takes
place mostly after a global optimum is reached. Here the technical
bulk lies in uniform control on weight movements, which is delicate
in the multilayer case.

\section{Numerical illustration\label{sec:Numerical-illustration}}

\begin{figure}
\begin{centering}
\subfloat[]{\begin{centering}
\includegraphics[width=0.19\columnwidth]{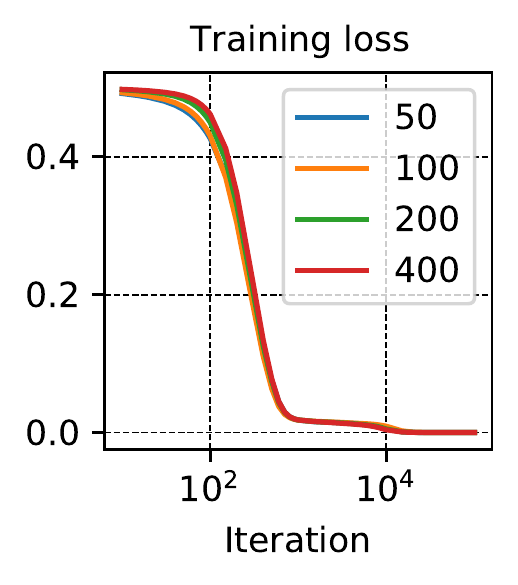}
\par\end{centering}
}\subfloat[]{\begin{centering}
\includegraphics[width=0.177\columnwidth]{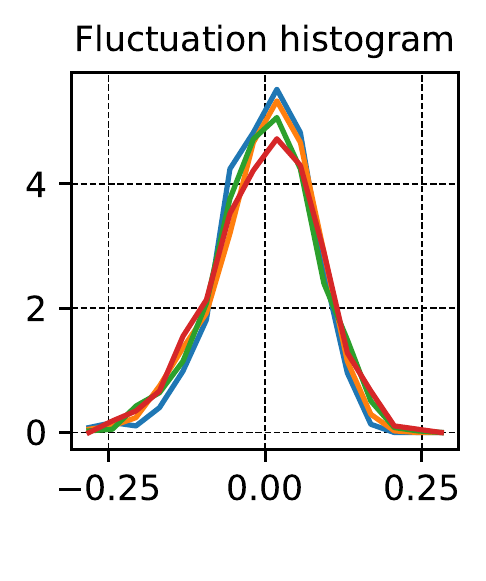}
\par\end{centering}
}\subfloat[]{\begin{centering}
\includegraphics[width=0.2\columnwidth]{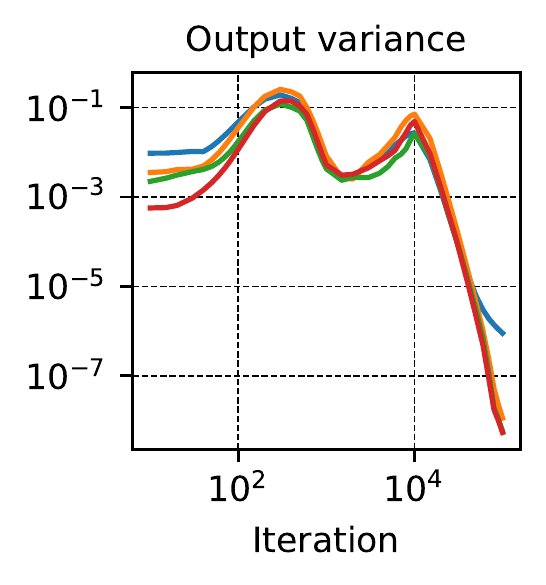}
\par\end{centering}
}\subfloat[]{\begin{centering}
\includegraphics[width=0.2\columnwidth]{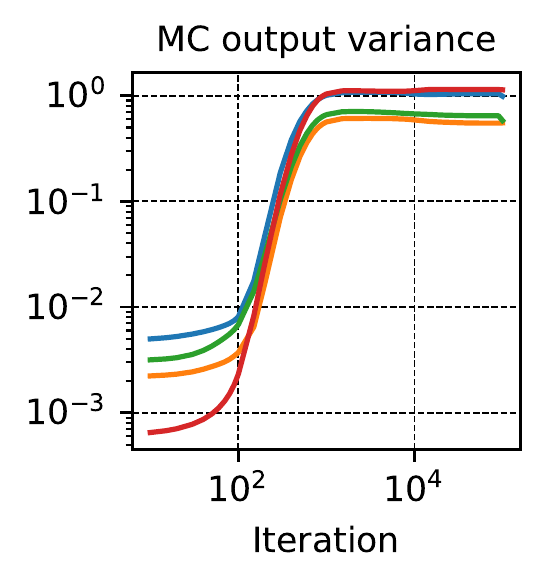}
\par\end{centering}
}\subfloat[]{\begin{centering}
\includegraphics[width=0.2\columnwidth]{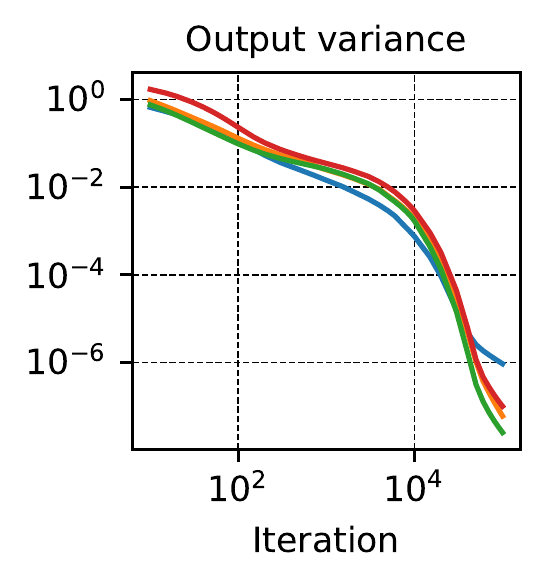}
\par\end{centering}
}
\par\end{centering}
\caption{MNIST classification of digits 0 and 4 versus 5 and 9, with full-batch
GD training on 100 images. The network has 3 layers, $\tanh$ activations
and a Huber loss. We vary the width $N\in\{50,100,200,400\}$. See
Appendix \ref{Appendix:experiments} for details. (a): The evolution
of the training loss, at a random initialization. (b): Histogram of
$\sqrt{N}(\hat{{\bf y}}(t,x)-\hat{y}(t,x))$ at iteration $10^{3}$
for a test image $x$. (c): Output variance $\mathbf{E}\mathbb{E}_{Z}\big[\big|\sqrt{N}(\hat{{\bf y}}(t,X)-\hat{y}(t,X))\big|^{2}\big]$.
(d): Monte-Carlo (MC) output variance, a.k.a. $\mathbf{E}\mathbb{E}_{Z}[|\tilde{G}^{y}\left(t,X\right)|^{2}]$.
(e): Similar to (c), but the network is initialized at a global optimum.}

\label{fig:illustration}
\end{figure}

We give several illustrations in Fig. \ref{fig:illustration} via
a simple experimental setup with MNIST \cite{lecun2010mnist} (see
the caption and also Appendix \ref{Appendix:experiments} for more
details). Fig. \ref{fig:illustration}(a) shows that the network converges
to a global optimum around iteration $10^{4}$ and displays a nonlinear
dynamics (which takes the shape of a superposition of two sigmoids).
The agreement of the different histogram plots of the output fluctuation
for varying widths $N$ in Fig \ref{fig:illustration}(b) verifies
the existence of a limiting Gaussian-like behavior, predicted by Theorems
\ref{thm:2nd_order_MF} and \ref{thm:CLT}. Fig \ref{fig:illustration}(c)
shows that the output variance $\mathbf{E}\mathbb{E}_{Z}\big[\big|\sqrt{N}(\hat{{\bf y}}(t,X)-\hat{y}(t,X))\big|^{2}\big]$
decreases with time quickly after iteration $10^{4}$, which is predicted
by Theorem \ref{thm:variance-global-opt-fast}. Fig \ref{fig:illustration}(d)
plots the Gaussian component $\tilde{G}^{y}$ in the output variance.
Note that after iteration $10^{4}$, this component no longer moves
since a global optimum is reached. Contrasting this plot with Fig
\ref{fig:illustration}(c), we see the central role of GD training
(in particular, the second-order MF limit component $R$) in reducing
the variance. As shown in Fig \ref{fig:illustration}(e), congruent
with Theorem \ref{thm:variance-global-opt-init}, when the network
is instead initialized at a global optimum, the output variance is
decreasing on the entire period. These plots also highlight an interesting
fact (previously mentioned in Section \ref{sec:Long-time-asymptotic-variance}):
the variance reduction effect takes place mostly after a global optimum
is reached.

\newpage{}

\section*{Acknowledgement}

The work of H. T. Pham is partially supported by a Two Sigma Fellowship.

\bibliographystyle{amsplain}
\bibliography{neurips}

\newpage{}

\appendix
\newgeometry{hmargin=.5in,vmargin=1in} 

Supplementary information for \textbf{``Limiting fluctuation and
trajectorial stability of multilayer neural networks with mean field
training''}:
\begin{itemize}
\item Appendix \ref{Appendix:Prelim} introduces several preliminaries:
new notations and known results from \cite{nguyen2020rigorous}.
\item Appendix \ref{Appendix:Gaussian} studies the Gaussian component $\tilde{G}$
and proves Theorem \ref{thm:G_tilde}.
\item Appendix \ref{Appendix:proof_well_posed} presents the proof of Theorem
\ref{thm:well-posed} for well-posedness of $R$.
\item Appendix \ref{Appendix:prop_chaos} proves Theorem \ref{thm:2nd_order_MF}
that connects the neural network with the second-order MF limit at
the fluctuation level.
\item Appendix \ref{Appendix:CLT_output} proves Theorem \ref{thm:CLT}
that establishes a central limit theorem for the output fluctuation.
\item Appendix \ref{Appendix:variance} studies the asymptotic width-scaled
output variance and proves that it eventually vanishes under different
conditions, i.e. Theorems \ref{thm:variance-global-opt-init} and
\ref{thm:variance-global-opt-fast}.
\item Appendix \ref{Appendix:experiments} describes the experimental details
for Section \ref{sec:Numerical-illustration}.
\end{itemize}

\section{Preliminaries\label{Appendix:Prelim}}

We introduce notations that are not in the main paper. For convenience,
for each index $i\in\left[L\right]$, we use $\mathbf{i}$ to refer
to either $i$ or the pair $\left(i-1,i\right)$, which depends on
the context, and we let $|\mathbf{i}|=1$ and $|\mathbf{i}|=2$ respectively
in each case. For instance, when we write $C_{\mathbf{i}}$, we refer
to either $C_{i}$ or $\left(C_{i-1},C_{i}\right)$. A statement that
is stated for $\mathbf{i}$ should hold in both cases.

For the MF limit:
\begin{align*}
\frac{\partial H_{i}(t,x,c_{i})}{\partial H_{j}(c_{j}')} & =\begin{cases}
0, & i\le j,\\
w_{i}(t,c_{j}',c_{i})\varphi_{j}'(H_{j}(t,x,c_{j}')), & i=j+1,\\
\mathbb{E}_{C_{j+1}'}\left[w_{j+1}(t,c_{j}',C_{j+1}')\varphi_{j}'(H_{j}(t,x,c_{j}')){\displaystyle \frac{\partial H_{i}(t,x,c_{i})}{\partial H_{j+1}(C_{j+1}')}}\right], & i>j+1,
\end{cases}\\
\frac{\partial H_{j}(t,x,c_{j})}{\partial_{*}H_{j}(c_{j})} & =1,\\
\frac{\partial H_{i}(t,x,c_{i})}{\partial w_{j}(c_{j-1}',c_{j}')} & =\begin{cases}
0, & i<j,\\
{\displaystyle \frac{\partial H_{i}(t,x,c_{i})}{\partial H_{j}(c_{j}')}}x, & 1=j<i,\\
{\displaystyle \frac{\partial H_{i}(t,x,c_{i})}{\partial H_{j}(c_{j}')}}\varphi_{j-1}(H_{j-1}(t,x,c_{j-1}')), & 1<j<i,
\end{cases}\\
\frac{\partial H_{j}(t,x,c_{j})}{\partial_{*}w_{j}(c_{j-1}',c_{j})} & =\begin{cases}
x, & j=1,\\
{\displaystyle \frac{\partial H_{j}(t,x,c_{j})}{\partial_{*}H_{j}(c_{j})}}\varphi_{j-1}(H_{j-1}(t,x,c_{j-1}')), & j>1,
\end{cases}\\
\frac{\partial^{2}\hat{y}(t,x)}{\partial w_{j}(c_{j-1}',c_{j}')\partial H_{L}(1)} & =\varphi_{L}''(H_{L}(t,x,1))\frac{\partial H_{L}(t,x,1)}{\partial w_{j}(c_{j-1}',c_{j}')},\\
\frac{\partial^{2}\hat{y}(t,x)}{\partial w_{j}(c_{j-1}',c_{j}')\partial H_{i-1}(c_{i-1})} & =\mathbb{E}_{C_{i}}\left[w_{i}(t,c_{i-1},C_{i})\varphi_{i-1}''(H_{i-1}(t,x,c_{i-1}))\frac{\partial H_{i-1}(t,x,c_{i-1})}{\partial w_{j}(c_{j-1}',c_{j}')}\frac{\partial\hat{y}(t,x)}{\partial H_{i}(C_{i})}\right]\\
 & \quad+\mathbb{E}_{C_{i}}\left[w_{i}(t,c_{i-1},C_{i})\varphi_{i-1}'(H_{i-1}(t,x,c_{i-1}))\frac{\partial^{2}\hat{y}(t,x)}{\partial w_{j}(c_{j-1}',c_{j}')\partial H_{i}(C_{i})}\right],\quad i\leq L,\\
\frac{\partial^{2}\hat{y}(t,x)}{\partial_{*}w_{i}(c_{i-1},c_{i}')\partial H_{i-1}(c_{i-1})} & =\varphi_{i-1}'(H_{i-1}(t,x,c_{i-1}))\frac{\partial\hat{y}(t,x)}{\partial H_{i}(c_{i}')},\\
\frac{\partial^{2}\hat{y}(t,x)}{\partial_{*}w_{i-1}(c_{i-2}',c_{i-1})\partial H_{i-1}(c_{i-1})} & =\begin{cases}
\mathbb{E}_{C_{i}}\left[w_{i}(t,c_{i-1},C_{i})\varphi_{i-1}''(H_{i-1}(t,x,c_{i-1}))\varphi_{i-2}(H_{i-2}(t,x,c_{i-2}')){\displaystyle \frac{\partial\hat{y}(t,x)}{\partial H_{i}(C_{i})}}\right], & 2<i\leq L,\\
\mathbb{E}_{C_{2}}\left[w_{2}(t,c_{1},C_{2})\varphi_{1}''(H_{1}(t,x,c_{1}))x{\displaystyle \frac{\partial\hat{y}(t,x)}{\partial H_{2}(C_{2})}}\right], & i=2,\\
{\displaystyle \frac{\partial^{2}\hat{y}(t,x)}{\partial w_{L}(c_{L-1}',1)\partial H_{L}(1)}}, & i=L+1,
\end{cases}\\
\frac{\partial^{2}\hat{y}(t,x)}{\partial w_{j}(c_{j-1}',c_{j}')\partial w_{i}(c_{i-1},c_{i})} & =\begin{cases}
{\displaystyle \frac{\partial^{2}\hat{y}(t,x)}{\partial w_{j}(c_{j-1}',c_{j}')\partial H_{i}(c_{i})}}\varphi_{i-1}(H_{i-1}(t,x,c_{i-1}))\\
\qquad+{\displaystyle \frac{\partial\hat{y}(t,x)}{\partial H_{i}(c_{i})}}\varphi_{i-1}'(H_{i-1}(t,x,c_{i-1})){\displaystyle \frac{\partial H_{i-1}(t,x,c_{i-1})}{\partial w_{j}(c_{j-1}',c_{j}')}}, & i>1,\\
{\displaystyle \frac{\partial^{2}\hat{y}(t,x)}{\partial w_{j}(c_{j-1}',c_{j}')\partial H_{1}(c_{1})}}x, & i=1,
\end{cases}\\
\frac{\partial^{2}\hat{y}(t,x)}{\partial_{*}w_{i-1}(c_{i-2}',c_{i-1})\partial w_{i}(c_{i-1},c_{i})} & ={\displaystyle \frac{\partial\hat{y}(t,x)}{\partial H_{i}(c_{i})}}\varphi_{i-1}'(H_{i-1}(t,x,c_{i-1})){\displaystyle \frac{\partial H_{i-1}(t,x,c_{i-1})}{\partial_{*}w_{i-1}(c_{i-2}',c_{i-1})}},\\
\frac{\partial^{2}\hat{y}(t,x)}{\partial_{*}w_{i}(c_{i-1}',c_{i})\partial w_{i}(c_{i-1},c_{i})} & =\begin{cases}
\varphi_{i-1}(H_{i-1}(t,x,c_{i-1})){\displaystyle \frac{\partial^{2}\hat{y}(t,x)}{\partial_{*}w_{i}(c_{i-1}',c_{i})\partial H_{i}(c_{i})}}, & i>1,\\
{\displaystyle \frac{\partial^{2}\hat{y}(t,x)}{\partial_{*}w_{1}(1,c_{1})\partial H_{1}(c_{1})}}x, & i=1
\end{cases}\\
\frac{\partial^{2}\hat{y}(t,x)}{\partial_{*}w_{i+1}(c_{i},c_{i+1}')\partial w_{i}(c_{i-1},c_{i})} & =\begin{cases}
\varphi_{i-1}(H_{i-1}(t,x,c_{i-1})){\displaystyle \frac{\partial^{2}\hat{y}(t,x)}{\partial_{*}w_{i+1}(c_{i},c_{i+1}')\partial H_{i}(c_{i})}}, & i>1,\\
{\displaystyle \frac{\partial^{2}\hat{y}(t,x)}{\partial_{*}w_{2}(c_{1},c_{2}')\partial H_{1}(c_{1})}}x, & i=1.
\end{cases}
\end{align*}
With this, we write the dynamics (\ref{eq:ODE_2ndMF}) for the second-order
MF limit $R$ in its complete form as follows:
\begin{align}
 & \partial_{t}R_{i}(G,t,c_{i-1},c_{i})=\nonumber \\
 & -\mathbb{E}_{Z}\bigg[\frac{\partial\hat{y}(t,X)}{\partial w_{i}(c_{i-1},c_{i})}\partial_{2}^{2}{\cal L}\left(Y,\hat{y}\left(t,X\right)\right)\bigg(G^{y}\left(t,X\right)+\sum_{r=1}^{L}\mathbb{E}_{C}\bigg[R_{r}\left(G,t,C_{r-1},C_{r}\right)\frac{\partial\hat{y}\left(t,X\right)}{\partial w_{r}\left(C_{r-1},C_{r}\right)}\bigg]\bigg)\bigg]\nonumber \\
 & -\mathbb{E}_{Z}\bigg[\partial_{2}{\cal L}\left(Y,\hat{y}\left(t,X\right)\right)\mathbb{E}_{C}\bigg[\sum_{r=1}^{L}R_{r}\left(G,t,C_{r-1},C_{r}\right)\frac{\partial^{2}\hat{y}\left(t,X\right)}{\partial w_{r}\left(C_{r-1},C_{r}\right)\partial w_{i}\left(c_{i-1},c_{i}\right)}\bigg]\bigg]\nonumber \\
 & -\mathbb{E}_{Z}\bigg[\partial_{2}{\cal L}\left(Y,\hat{y}\left(t,X\right)\right)\mathbb{E}_{C}\bigg[R_{i-1}\left(G,t,C_{i-2},c_{i-1}\right)\frac{\partial^{2}\hat{y}\left(t,X\right)}{\partial_{*}w_{i-1}\left(C_{i-2},c_{i-1}\right)\partial w_{i}\left(c_{i-1},c_{i}\right)}\bigg]\bigg]\nonumber \\
 & -\mathbb{E}_{Z}\bigg[\partial_{2}{\cal L}\left(Y,\hat{y}\left(t,X\right)\right)\mathbb{E}_{C}\bigg[R_{i}\left(G,t,C_{i-1},c_{i}\right)\frac{\partial^{2}\hat{y}\left(t,X\right)}{\partial_{*}w_{i}\left(C_{i-1},c_{i}\right)\partial w_{i}\left(c_{i-1},c_{i}\right)}\bigg]\bigg]\nonumber \\
 & -\mathbb{E}_{Z}\bigg[\partial_{2}{\cal L}\left(Y,\hat{y}\left(t,X\right)\right)\mathbb{E}_{C}\bigg[R_{i+1}\left(G,t,c_{i},C_{i+1}\right)\frac{\partial^{2}\hat{y}\left(t,X\right)}{\partial_{*}w_{i+1}\left(c_{i},C_{i+1}\right)\partial w_{i}\left(c_{i-1},c_{i}\right)}\bigg]\bigg]\nonumber \\
 & -\mathbb{E}_{Z}\bigg[\partial_{2}{\cal L}\left(Y,\hat{y}\left(t,X\right)\right)G_{i}^{w}\left(t,c_{i-1},c_{i},X\right)\bigg]\label{eq:ODE_2ndMF-alt}
\end{align}
where we take by convention that $R_{0}=R_{L+1}=0$. We also define
secondary quantities, similar to those in Section \ref{subsec:Gaussian}
e.g. $\frac{\partial^{2}\tilde{y}(t,x)}{\partial w_{j}(c_{j-1}',c_{j}')\partial w_{i}(c_{i-1},c_{i})}$,
in a similar fashion, by taking their MF counterparts and replacing
$\mathbb{E}_{C_{i}}$ being replaced by an empirical average over
$\left\{ C_{i}\left(j_{i}\right)\right\} _{j_{i}\in\left[N_{i}\right]}$.

We also recall the following bounds from \cite{nguyen2020rigorous}.
\begin{lem}
\label{lem:MF_a_priori}Under Assumption \ref{Assump:Assumption_1},
for any $T\geq0$,
\[
\max_{i\in\left[L\right]}m^{-1/2}\mathbb{E}_{C}\left[\sup_{t\leq T}\left|w_{i}\left(t,C_{i-1},C_{i}\right)\right|^{m}\right]^{1/m}\leq K_{T}.
\]
\end{lem}

Boundedness of moments of several other MF quantities at any time
$t\leq T$ are consequences of this lemma and Assumption \ref{Assump:Assumption_1}.
We omit the details.

\section{CLT for the Gaussian component $\tilde{G}$: Proof of Theorem \ref{thm:G_tilde}\label{Appendix:Gaussian}}

We recall the process $\tilde{G}$ defined in Section \ref{subsec:Gaussian}.
Recall that for each $j_{i}\in[N_{i}]$, we sample $C_{i}(j_{i})\in\Omega_{i}$
independently at random from $P_{i}$. Let $S_{i}=\{C_{i}(1),\dots,C_{i}(N_{i})\}$.
We denote by $\mathbf{E}_{i}$ for the expectation over the random
choice of $S_{i}$. We also recall that ${\bf E}$ is the expectation
over the random choice of $S_{i},\;i=1,\dots,L$.

We will use the following notation throughout. Let $\delta$ be a
function of $t\in\mathbb{T},z\in\mathbb{Z}$, $c_{1}\in\Omega_{1},c_{2}\in\Omega_{2},\dots,c_{L}\in\Omega_{L}$,
and $j_{1}\in[N_{1}],j_{2}\in[N_{2}],\dots,j_{L}\in[N_{L}]$ (where
$\delta$ may not necessarily depend on all those variables). For
$o_{N}$ decreasing in $N$, we write $\delta={\bf O}_{T}(o_{N})$
if for any $t\leq T$,
\[
\mathbf{E}\mathbb{E}_{J}\mathbb{E}_{Z,C}[\left|\delta(t,Z,C_{1},\dots,C_{L},J_{1},\dots,J_{L}\right|^{2}]\le K_{T}o_{N}^{2}
\]
for sufficiently large $N$. Thus, if we write 
\[
f(t,z,c_{1},c_{2},\dots,c_{L})=g(t,z,c_{1},c_{2},\dots,c_{L})+{\bf O}_{T}(o_{N})
\]
we mean
\[
\mathbf{E}\mathbb{E}_{J}\mathbb{E}_{Z,C}[\left|f(t,Z,C_{1},\dots,C_{L},J_{1},\dots,J_{L})-g(t,Z,C_{1},\dots,C_{L},J_{1},\dots,J_{L})\right|^{2}]\le K_{T}o_{N}^{2}
\]
for all $t\leq T$, for sufficiently large $N$.

We define additional processes as follows 
\begin{align*}
\tilde{G}_{i}^{\partial\tilde{H}}(t,c_{i},x) & =\sqrt{N}\left(\frac{\partial\tilde{y}(t,x)}{\partial\tilde{H}_{i}(c_{i})}-\frac{\partial\hat{y}(t,x)}{\partial H_{i}(c_{i})}\right),\\
\tilde{G}_{i-1}^{\partial H}(t,c_{i-1},x) & =\sqrt{N}\left(\mathbb{E}_{J_{i}}\left[w_{i}(t,c_{i-1},C_{i}(J_{i}))\frac{\partial\hat{y}(t,x)}{\partial H_{i}(C_{i}(J_{i}))}\right]-\mathbb{E}_{C_{i}}\left[w_{i}(t,c_{i-1},C_{i})\frac{\partial\hat{y}(t,x)}{\partial H_{i}(C_{i})}\right]\right)\varphi_{i-1}'\left(H_{i-1}(t,x,c_{i-1})\right),\\
\tilde{G}_{L}^{\partial H}(t,1,x) & =\tilde{G}_{L-1}^{\partial H}(t,c_{L-1},x)=0,\\
\tilde{G}_{i}^{\tilde{H}}(t,c_{i},x) & =\sqrt{N}\left(\tilde{H}_{i}(t,x,c_{i})-H_{i}(t,x,c_{i})\right),\\
\tilde{G}_{i}^{H}(t,c_{i},x) & =\sqrt{N}\left(\mathbb{E}_{J_{i-1}}[w_{i}(t,C_{i-1}(J_{i-1}),c_{i})\varphi_{i-1}(H_{i-1}(t,x,C_{i-1}(J_{i-1})))]-\mathbb{E}_{C_{i-1}}[w_{i}(t,C_{i-1},c_{i})\varphi_{i-1}(H_{i-1}(t,x,C_{i-1}))]\right),\\
\tilde{G}_{1}^{H}(t,c_{1},x) & =0.
\end{align*}

The next lemma is a key tool in the argument.
\begin{lem}
\label{lem:second-moment-dependency}For an index $\mathbf{i}$, let
$f(S,c_{{\bf i}},c_{{\bf i}'})$ be a function of $S=\{C_{\mathbf{i}}(j_{\mathbf{i}}):\;j_{\mathbf{i}}\in\left[N_{\mathbf{i}}\right]\}\cup\{S'\}$
and $c_{{\bf i}}$, $c_{{\bf i}'}$, where $S'$ denotes some random
variable that is independent of $\{C_{\mathbf{i}}(j_{\mathbf{i}}):\;j_{\mathbf{i}}\in\left[N_{\mathbf{i}}\right]\}$.
Assume that for some $\alpha>0$, 
\begin{align*}
{\bf E}\mathbb{E}_{C_{{\bf i}'}}\left[\left|f(S^{j_{\mathbf{i}}},C_{{\bf i}}(j_{{\bf i}}),C_{{\bf i}'})\right|^{2}\right] & \le K,\\
{\bf E}\mathbb{E}_{C_{{\bf i}'}}\left[\left|f(S,C_{{\bf i}}(j_{{\bf i}}),C_{{\bf i}'})-f(S^{j_{\mathbf{i}}},C_{{\bf i}}(j_{{\bf i}}),C_{{\bf i}'})\right|^{2}\right] & \le K/N^{\alpha},
\end{align*}
for all $j_{\mathbf{i}}$, where we define $S^{j_{\mathbf{i}}}$ similar
to $S$ except with $C_{{\bf i}}(j_{{\bf i}})$ in $S_{{\bf i}}$
replaced by an independent copy $C_{{\bf i}}'(j_{{\bf i}})$ and $\mathbf{E}$
denotes the expectation w.r.t $S$. Then we have 
\[
\mathbb{E}_{J_{{\bf i}}}[f(S,C_{{\bf i}}(J_{{\bf i}}),c_{{\bf i}'})]=\mathbb{E}_{C_{{\bf i}}}[f(S,C_{{\bf i}},c_{{\bf i}'})]+{\bf O}(N^{-\alpha/4}+N^{-\left|\mathbf{i}\right|/2}).
\]
\end{lem}

\begin{proof}
We have 
\begin{align*}
 & {\bf E}\mathbb{E}_{C_{{\bf i}'}}\left[\left|\mathbb{E}_{J_{{\bf i}}}[f(S,C_{{\bf i}}(J_{{\bf i}}),C_{{\bf i}'})]-\mathbb{E}_{C_{{\bf i}}}[f(S,C_{{\bf i}},C_{{\bf i}'})]\right|^{2}\right]\\
 & =\frac{1}{N^{2\left|\mathbf{i}\right|}}\sum_{j_{{\bf i}},j_{{\bf i}}'}{\bf E}\mathbb{E}_{C_{{\bf i}'}}\left[\left(f(S,C_{{\bf i}}(j_{{\bf i}}),C_{{\bf i}'})-\mathbb{E}_{C_{{\bf i}}}[f(S,C_{{\bf i}},C_{{\bf i}'})]\right)\cdot\left(f(S,C_{{\bf i}}(j_{{\bf i}}'),C_{{\bf i}'})-\mathbb{E}_{C_{{\bf i}}}[f(S,C_{{\bf i}},C_{{\bf i}'})]\right)\right].
\end{align*}
Define $S^{j_{\mathbf{i}}}$ -- with an abuse of notation -- similar
to $S$ but with the terms involving $C_{{\bf i}}(j_{{\bf i}})$ and
$C_{{\bf i}}(j_{{\bf i}}')$ replaced by independent copies $C_{{\bf i}}'(j_{{\bf i}})$
and $C_{{\bf i}}'(j_{{\bf i}}')$. (In particular, $S^{j_{\mathbf{i}}}$
has the same distribution as $S$.) We have:
\begin{align*}
 & \bigg|{\bf E}\left[\left(f(S,C_{{\bf i}}(j_{{\bf i}}),c_{{\bf i}'})-\mathbb{E}_{C_{{\bf i}}}[f(S,C_{{\bf i}},c_{{\bf i}'})]\right)\left(f(S,C_{{\bf i}}(j_{{\bf i}}'),c_{{\bf i}'})-\mathbb{E}_{C_{{\bf i}}}[f(S,C_{{\bf i}},c_{{\bf i}'})]\right)\right]\\
 & \qquad-{\bf E}\left[\left(f(S^{j_{\mathbf{i}}},C_{{\bf i}}(j_{{\bf i}}),c_{{\bf i}'})-\mathbb{E}_{C_{{\bf i}}}[f(S^{j_{\mathbf{i}}},C_{{\bf i}},c_{{\bf i}'})]\right)\left(f(S^{j_{\mathbf{i}}},C_{{\bf i}}(j_{{\bf i}}'),c_{{\bf i}'})-\mathbb{E}_{C_{{\bf i}}}[f(S^{j_{\mathbf{i}}},C_{{\bf i}},c_{{\bf i}'})]\right)\right]\bigg|\\
 & =\left|{\bf E}\left[f(S,C_{{\bf i}}(j_{{\bf i}}),c_{{\bf i}'})f(S,C_{{\bf i}}(j_{{\bf i}}'),c_{{\bf i}'})-f(S^{j_{\mathbf{i}}},C_{{\bf i}}(j_{{\bf i}}),c_{{\bf i}'})f(S^{j_{\mathbf{i}}},C_{{\bf i}}(j_{{\bf i}}'),c_{{\bf i}'})\right]\right|\\
 & \le K{\bf E}\left[\left|f(S,C_{{\bf i}}(j_{{\bf i}}),c_{{\bf i}'})\left(f(S,C_{{\bf i}}(j_{{\bf i}}'),c_{{\bf i}'})-f(S^{j_{\mathbf{i}}},C_{{\bf i}}(j_{{\bf i}}'),c_{{\bf i}'})\right)\right|\right]+K{\bf E}\left[\left|f(S^{j_{\mathbf{i}}},C_{{\bf i}}(j_{{\bf i}}'),c_{{\bf i}'})\left(f(S,C_{{\bf i}}(j_{{\bf i}}),c_{{\bf i}'})-f(S^{j_{\mathbf{i}}},C_{{\bf i}}(j_{{\bf i}}),c_{{\bf i}'})\right)\right|\right]\\
 & \le K{\bf E}\left[\left(f(S,C_{{\bf i}}(j_{{\bf i}}),c_{{\bf i}'})\right)^{2}\right]^{1/2}{\bf E}\left[\left(f(S,C_{{\bf i}}(j_{{\bf i}}'),c_{{\bf i}'})-f(S^{j_{\mathbf{i}}},C_{{\bf i}}(j_{{\bf i}}'),c_{{\bf i}'})\right)^{2}\right]^{1/2}\\
 & \qquad+K{\bf E}\left[\left(f(S^{j_{\mathbf{i}}},C_{{\bf i}}(j_{{\bf i}}'),c_{{\bf i}'})\right)^{2}\right]^{1/2}{\bf E}\left[\left(f(S,C_{{\bf i}}(j_{{\bf i}}),c_{{\bf i}'})-f(S^{j_{\mathbf{i}}},C_{{\bf i}}(j_{{\bf i}}),c_{{\bf i}'})\right)^{2}\right]^{1/2}\\
 & \le K{\bf E}\left[\left(f(S,C_{{\bf i}}(j_{{\bf i}}'),c_{{\bf i}'})-f(S^{j_{\mathbf{i}}},C_{{\bf i}}(j_{{\bf i}}'),c_{{\bf i}'})\right)^{2}\right]^{1/2}{\bf E}\left[\left(f(S,C_{{\bf i}}(j_{{\bf i}}),c_{{\bf i}'})-f(S^{j_{\mathbf{i}}},C_{{\bf i}}(j_{{\bf i}}),c_{{\bf i}'})\right)^{2}\right]^{1/2}\\
 & \qquad+K{\bf E}\left[\left(f(S,C_{{\bf i}}(j_{{\bf i}}'),c_{{\bf i}'})-f(S^{j_{\mathbf{i}}},C_{{\bf i}}(j_{{\bf i}}'),c_{{\bf i}'})\right)^{2}\right]^{1/2}{\bf E}\left[\left(f(S^{j_{\mathbf{i}}},C_{{\bf i}}(j_{{\bf i}}),c_{{\bf i}'})\right)^{2}\right]^{1/2}\\
 & \qquad+K{\bf E}\left[\left(f(S^{j_{\mathbf{i}}},C_{{\bf i}}(j_{{\bf i}}'),c_{{\bf i}'})\right)^{2}\right]^{1/2}{\bf E}\left[\left(f(S,C_{{\bf i}}(j_{{\bf i}}),c_{{\bf i}'})-f(S^{j_{\mathbf{i}}},C_{{\bf i}}(j_{{\bf i}}),c_{{\bf i}'})\right)^{2}\right]^{1/2}.
\end{align*}
Furthermore, 
\begin{align*}
 & \mathbb{E}_{C_{{\bf i}'}}\left[{\bf E}\left[\left(f(S,C_{{\bf i}}(j_{{\bf i}}'),C_{{\bf i}'})-f(S^{j_{\mathbf{i}}},C_{{\bf i}}(j_{{\bf i}}'),C_{{\bf i}'})\right)^{2}\right]^{1/2}{\bf E}\left[\left(f(S^{j_{\mathbf{i}}},C_{{\bf i}}(j_{{\bf i}}),C_{{\bf i}'})\right)^{2}\right]^{1/2}\right]\\
 & \le\mathbb{E}_{C_{{\bf i}'}}{\bf E}\left[\left(f(S,C_{{\bf i}}(j_{{\bf i}}'),C_{{\bf i}'})-f(S^{j_{\mathbf{i}}},C_{{\bf i}}(j_{{\bf i}}'),C_{{\bf i}'})\right)^{2}\right]^{1/2}\mathbb{E}_{C_{{\bf i}'}}{\bf E}\left[\left(f(S^{j_{\mathbf{i}}},C_{{\bf i}}(j_{{\bf i}}),C_{{\bf i}'})\right)^{2}\right]^{1/2}\\
 & \le K/N^{\alpha/2},
\end{align*}
and the other two terms can be bounded similarly. Now recall that
$S^{j_{\mathbf{i}}}$ is independent of $C_{{\bf i}}(j_{{\bf i}})$
and $C_{{\bf i}}(j_{{\bf i}}')$. Let ${\cal S}^{j_{\mathbf{i}}}$
be the $\sigma$-algebra generated by $S'$ and $C_{{\bf k}}(j_{{\bf k}})$
for all $j_{{\bf k}}$ except $j_{{\bf i}}$ and $j_{{\bf i}}'$.
Then if $j_{{\bf i}}\neq j_{{\bf i}}'$,
\begin{align*}
 & {\bf E}\left[\left(f(S^{j_{\mathbf{i}}},C_{{\bf i}}(j_{{\bf i}}),c_{{\bf i}'})-\mathbb{E}_{C_{{\bf i}}}[f(S^{j_{\mathbf{i}}},C_{{\bf i}},c_{{\bf i}'})]\right)\left(f(S^{j_{\mathbf{i}}},C_{{\bf i}}(j_{{\bf i}}'),c_{{\bf i}'})-\mathbb{E}_{C_{{\bf i}}}[f(S^{j_{\mathbf{i}}},C_{{\bf i}},c_{{\bf i}'})]\right)\right]\\
 & ={\bf E}_{{\cal S}}\left[{\bf E}_{C_{{\bf i}}(j_{{\bf i}}),C_{{\bf i}}(j_{{\bf i}}')}\left[\left(f(S^{j_{\mathbf{i}}},C_{{\bf i}}(j_{{\bf i}}),c_{{\bf i}'})-\mathbb{E}_{C_{{\bf i}}}[f(S^{j_{\mathbf{i}}},C_{{\bf i}},c_{{\bf i}'})]\right)\left(f(S^{j_{\mathbf{i}}},C_{{\bf i}}(j_{{\bf i}}'),c_{{\bf i}'})-\mathbb{E}_{C_{{\bf i}}}[f(S^{j_{\mathbf{i}}},C_{{\bf i}},c_{{\bf i}'})]\right)\right]\right]\\
 & =0.
\end{align*}
Thus under the assumptions in the lemma,
\[
{\bf E}\mathbb{E}_{C_{{\bf i}'}}\left[\left(\mathbb{E}_{J_{{\bf i}}}[f(S,C_{{\bf i}}(J_{{\bf i}}),C_{{\bf i}'})]-\mathbb{E}_{C_{{\bf i}}}[f(S,C_{{\bf i}},C_{{\bf i}'})]\right)^{2}\right]\le\frac{K}{N^{\alpha/2}}+\frac{K}{N^{\left|\mathbf{i}\right|}}.
\]
\end{proof}
We proceed in several sections.

\subsection{A priori moment estimates}
\begin{lem}
\label{lem:tildeH-moment}Under Assumption \ref{Assump:Assumption_1},
we have for $t\leq T$,
\[
{\bf E}\mathbb{E}_{C_{i}}\left[\left(\tilde{G}_{i}^{H}(t,x,C_{i})\right)^{2p}\right],\;{\bf E}\mathbb{E}_{J_{i}}\left[\left(\tilde{G}_{i}^{H}(t,x,C_{i}\left(J_{i}\right))\right)^{2p}\right],\;{\bf E}\mathbb{E}_{C_{i}}\left[\left(\tilde{G}_{i}^{\tilde{H}}(t,x,C_{i})\right)^{2p}\right],\;{\bf E}\mathbb{E}_{C_{i}}\left[\left(\tilde{G}_{i}^{\tilde{H}}(t,x,C_{i}\left(J_{i}\right))\right)^{2p}\right]\le K_{T,p}.
\]
\end{lem}

\begin{proof}
Recall Lemma \ref{lem:MF_a_priori}. The base case $i=1$ easily follows.
Notice that
\begin{align*}
\mathbf{E}\mathbb{E}_{J_{i-1}}\left[\left(\tilde{G}_{i-1}^{H}(t,x,C_{i-1}(J_{i-1}))\right)^{2p}\right] & =\mathbf{E}\mathbb{E}_{C_{i-1}}\left[\left(\tilde{G}_{i-1}^{H}(t,x,C_{i-1})\right)^{2p}\right],\\
\mathbf{E}\mathbb{E}_{J_{i-1}}\left[\left(\tilde{G}_{i-1}^{\tilde{H}}(t,x,C_{i-1}(J_{i-1}))\right)^{2p}\right] & =\mathbf{E}\mathbb{E}_{C_{i-1}}\left[\left(\tilde{G}_{i-1}^{\tilde{H}}(t,x,C_{i-1})\right)^{2p}\right],
\end{align*}
where the second claim is because the randomness of $\tilde{G}_{i-1}^{\tilde{H}}(t,x,c_{i-1})$
comes from $\left\{ C_{k}\left(j_{k}\right):\;j_{k}\in\left[N_{k}\right],\;k<i-1\right\} $.

Note that 
\begin{align*}
 & \tilde{G}_{i}^{\tilde{H}}(t,c_{i},x)\\
 & =\sqrt{N}\left(\tilde{H}_{i}(t,x,c_{i})-H_{i}(t,x,c_{i})\right)\\
 & =\sqrt{N}\left(\mathbb{E}_{J_{i-1}}[w_{i}(t,C_{i-1}(J_{i-1}),c_{i})\varphi_{i-1}(\tilde{H}_{i-1}(t,x,C_{i-1}(J_{i-1})))]-\mathbb{E}_{C_{i-1}}[w_{i}(t,C_{i-1},c_{i})\varphi_{i-1}(H_{i-1}(t,x,C_{i-1}))]\right)\\
 & =\tilde{G}_{i}^{H}(t,x,c_{i})+\sqrt{N}\mathbb{E}_{J_{i-1}}[w_{i}(t,C_{i-1}(J_{i-1}),c_{i})\left(\varphi_{i-1}(\tilde{H}_{i-1}(t,x,C_{i-1}(J_{i-1})))-\varphi_{i-1}(H_{i-1}(t,x,C_{i-1}(J_{i-1})))\right)].
\end{align*}
We have 
\begin{align*}
 & {\bf E}\mathbb{E}_{C_{i}}\left[\left(\tilde{G}_{i}^{H}(t,x,C_{i})\right)^{2p}\right]\\
 & \leq{\bf E}\mathbb{E}_{C_{i}}\left[\left(\sqrt{N}\left(\mathbb{E}_{J_{i-1}}[w_{i}(t,C_{i-1}(J_{i-1}),C_{i})\varphi_{i-1}(H_{i-1}(t,x,C_{i-1}(J_{i-1})))]-\mathbb{E}_{C_{i-1}}[w_{i}(t,C_{i-1},C_{i})\varphi_{i-1}(H_{i-1}(t,x,C_{i-1}))]\right)\right)^{2p}\right]\\
 & \le\mathbb{E}_{C_{i}}\left[\mathbb{V}_{C_{i-1}}[w_{i}(t,C_{i-1},C_{i})\varphi_{i-1}(H_{i-1}(t,x,C_{i-1}))]^{p}\right]+\frac{K_{T,p}}{N}\\
 & \le K_{T,p}.
\end{align*}
We also have:
\begin{align*}
 & \mathbf{E}\mathbb{E}_{C_{i}}\left[\left(\sqrt{N}\mathbb{E}_{J_{i-1}}\left[w_{i}(t,C_{i-1}(J_{i-1}),C_{i})\left(\varphi_{i-1}(\tilde{H}_{i-1}(t,x,C_{i-1}(J_{i-1})))-\varphi_{i-1}(H_{i-1}(t,x,C_{i-1}(J_{i-1})))\right)\right]\right)^{2p}\right]\\
 & \leq\mathbf{E}\mathbb{E}_{C_{i}}\left[\left(\mathbb{E}_{J_{i-1}}\left[w_{i}(t,C_{i-1}(J_{i-1}),C_{i})\tilde{G}_{i-1}^{\tilde{H}}(t,x,C_{i-1}(J_{i-1}))\right]\right)^{2p}\right]\\
 & \leq\mathbf{E}\mathbb{E}_{C_{i},J_{i-1}}\left[\left|w_{i}(t,C_{i-1}(J_{i-1}),C_{i})\right|^{4p}\right]^{1/2}\mathbf{E}\mathbb{E}_{J_{i-1}}\left[\left(\tilde{G}_{i-1}^{\tilde{H}}(t,x,C_{i-1}(J_{i-1}))\right)^{4p}\right]^{1/2}\\
 & \leq K_{T,p}\mathbf{E}\mathbb{E}_{J_{i-1}}\left[\left(\tilde{G}_{i-1}^{\tilde{H}}(t,x,C_{i-1}(J_{i-1}))\right)^{4p}\right]^{1/2}.
\end{align*}
Hence, by the induction hypothesis, we obtain 
\[
{\bf E}\mathbb{E}_{C_{i}}\left[\left(\tilde{G}_{i}^{\tilde{H}}(t,x,C_{i})\right)^{2p}\right]\le K_{T,p}.
\]
\end{proof}
\begin{lem}
\label{lem:partial-tildeH-moment}Under Assumption \ref{Assump:Assumption_1},
we have for any $t\leq T$,
\[
\mathbf{E}\mathbb{E}_{C_{i}}\left[\left(\tilde{G}_{i}^{\partial H}(t,C_{i},x)\right)^{2p}\right],\;\mathbf{E}\mathbb{E}_{J_{i-1}}\left[\left(\tilde{G}_{i-1}^{\partial H}(t,C_{i-1}\left(J_{i-1}\right),x)\right)^{2p}\right]\leq K_{T,p},
\]
\[
\mathbf{E}\mathbb{E}_{C_{i}}\left[\left(\tilde{G}_{i}^{\partial\tilde{H}}(t,C_{i},x)\right)^{2p}\right],\;\mathbf{E}\mathbb{E}_{J_{i}}\left[\left(\tilde{G}_{i}^{\partial\tilde{H}}(t,C_{i}\left(J_{i}\right),x)\right)^{2p}\right]\leq K_{T,p}.
\]
\end{lem}

\begin{proof}
Recall that $\tilde{G}_{L}^{\partial H}=0$. For $\tilde{G}_{L}^{\partial\tilde{H}}$:
\begin{align*}
\mathbf{E}\left[\left(\tilde{G}_{L}^{\partial\tilde{H}}(t,1,x)\right)^{2p}\right] & =\mathbf{E}\left[\left(\sqrt{N}\left(\frac{\partial\tilde{y}(t,x)}{\partial\tilde{H}_{L}(1)}-\frac{\partial\hat{y}(t,x)}{\partial H_{L}(1)}\right)\right)^{2p}\right]\\
 & =\mathbf{E}\left[\left(\sqrt{N}\left(\varphi_{L}'\left(\tilde{H}_{L}\left(t,x,1\right)\right)-\varphi_{L}'\left(H_{L}\left(t,x,1\right)\right)\right)\right)^{2p}\right]\\
 & \leq K\mathbf{E}\left[\left(\tilde{G}_{L}^{\tilde{H}}\left(t,1,x\right)\right)^{2p}\right]\\
 & \leq K_{T,p}
\end{align*}
by Lemma \ref{lem:tildeH-moment}.

We note that 
\begin{align*}
 & \tilde{G}_{i-1}^{\partial\tilde{H}}(t,c_{i-1},x)\\
 & =\sqrt{N}\left(\frac{\partial\tilde{y}(t,x)}{\partial\tilde{H}_{i-1}(c_{i-1})}-\frac{\partial\hat{y}(t,x)}{\partial H_{i-1}(c_{i-1})}\right)\\
 & =\sqrt{N}\left(\mathbb{E}_{J_{i}}\left[w_{i}(t,c_{i-1},C_{i}(J_{i}))\varphi_{i-1}'(\tilde{H}_{i-1}(t,x,c_{i-1}))\frac{\partial\tilde{y}(t,x)}{\partial\tilde{H}_{i}(C_{i}(J_{i}))}\right]-\mathbb{E}_{C_{i}}\left[w_{i}(t,c_{i-1},C_{i})\varphi_{i-1}'(H_{i-1}(t,x,c_{i-1}))\frac{\partial\hat{y}(t,x)}{\partial H_{i}(C_{i})}\right]\right)\\
 & =\varphi_{i-1}'(H_{i-1}(t,x,c_{i-1}))\tilde{G}_{i-1}^{\partial H}(t,c_{i-1},x)\\
 & \quad+\sqrt{N}\left(\varphi_{i-1}'(\tilde{H}_{i-1}(t,x,c_{i-1}))-\varphi_{i-1}'(H_{i-1}(t,x,c_{i-1}))\right)\mathbb{E}_{J_{i}}\left[w_{i}(t,c_{i-1},C_{i}(J_{i}))\frac{\partial\hat{y}(t,x)}{\partial H_{i}(C_{i}(J_{i}))}\right]\\
 & \quad+\varphi_{i-1}'(\tilde{H}_{i-1}(t,x,c_{i-1}))\mathbb{E}_{J_{i}}\left[w_{i}(t,c_{i-1},C_{i}(J_{i}))\tilde{G}_{i}^{\partial\tilde{H}}(t,C_{i}(J_{i}),x)\right].
\end{align*}
We bound each term. We have for $i<L$,
\begin{align*}
 & \mathbf{E}\mathbb{E}_{C_{i-1}}\left[\left(\tilde{G}_{i-1}^{\partial H}(t,C_{i-1},x)\right)^{2p}\right]\\
 & \leq K_{p}\mathbf{E}\mathbb{E}_{C_{i-1}}\left[\left(\sqrt{N}\left(\mathbb{E}_{J_{i}}\left[w_{i}(t,C_{i-1},C_{i}(J_{i}))\frac{\partial\hat{y}(t,x)}{\partial H_{i}(C_{i}(J_{i}))}\right]-\mathbb{E}_{C_{i}}\left[w_{i}(t,C_{i-1},C_{i})\frac{\partial\hat{y}(t,x)}{\partial H_{i}(C_{i})}\right]\right)\right)^{2p}\right]\\
 & \leq K_{p}\mathbb{E}_{C_{i-1}}\left[\mathbb{V}_{C_{i}}\left[w_{i}(t,C_{i-1},C_{i})\frac{\partial\hat{y}(t,x)}{\partial H_{i}(C_{i})}\right]^{p}\right]+\frac{K_{T,p}}{N}\\
 & \leq K_{T,p},
\end{align*}
and similarly,
\[
\mathbf{E}\mathbb{E}_{J_{i-1}}\left[\left(\tilde{G}_{i-1}^{\partial H}(t,C_{i-1}\left(J_{i-1}\right),x)\right)^{2p}\right]\leq K_{T,p}.
\]
The same of course holds in the case $i=L$ since $\tilde{G}_{L-1}^{\partial H}(t,c_{L-1},x)=0$.
By Lemma \ref{lem:tildeH-moment},
\begin{align*}
 & \mathbf{E}\mathbb{E}_{C_{i-1}}\left[\left(\sqrt{N}\left(\varphi_{i-1}'(\tilde{H}_{i-1}(t,x,C_{i-1}))-\varphi_{i-1}'(H_{i-1}(t,x,C_{i-1}))\right)\mathbb{E}_{J_{i}}\left[w_{i}(t,C_{i-1},C_{i}(J_{i}))\frac{\partial\hat{y}(t,x)}{\partial H_{i}(C_{i}(J_{i}))}\right]\right)^{2p}\right]\\
 & \leq K\mathbf{E}\mathbb{E}_{C_{i-1}}\left[\left(\tilde{G}_{i-1}^{\tilde{H}}(t,C_{i-1},x)\mathbb{E}_{J_{i}}\left[w_{i}(t,C_{i-1},C_{i}(J_{i}))\frac{\partial\hat{y}(t,x)}{\partial H_{i}(C_{i}(J_{i}))}\right]\right)^{2p}\right]\\
 & \leq K\mathbf{E}\mathbb{E}_{C_{i-1}}\left[\left(\tilde{G}_{i-1}^{\tilde{H}}(t,C_{i-1},x)\right)^{4p}\right]^{1/2}\mathbf{E}\mathbb{E}_{C_{i-1}}\left[\mathbb{E}_{J_{i}}\left[w_{i}(t,C_{i-1},C_{i}(J_{i}))\frac{\partial\hat{y}(t,x)}{\partial H_{i}(C_{i}(J_{i}))}\right]^{4p}\right]^{1/2}\\
 & \leq K_{T,p},
\end{align*}
and by noticing that $\mathbf{E}\mathbb{E}_{J_{i-1}}\left[\left(\tilde{G}_{i-1}^{\tilde{H}}(t,C_{i-1}\left(J_{i-1}\right),x)\right)^{4p}\right]=\mathbf{E}\mathbb{E}_{C_{i-1}}\left[\left(\tilde{G}_{i-1}^{\tilde{H}}(t,C_{i-1},x)\right)^{4p}\right]$,
we have similarly:
\begin{align*}
 & \mathbf{E}\mathbb{E}_{J_{i-1}}\left[\left(\sqrt{N}\left(\varphi_{i-1}'(\tilde{H}_{i-1}(t,x,C_{i-1}\left(J_{i-1}\right)))-\varphi_{i-1}'(H_{i-1}(t,x,C_{i-1}\left(J_{i-1}\right)))\right)\mathbb{E}_{J_{i}}\left[w_{i}(t,C_{i-1}\left(J_{i-1}\right),C_{i}(J_{i}))\frac{\partial\hat{y}(t,x)}{\partial H_{i}(C_{i}(J_{i}))}\right]\right)^{2p}\right]\\
 & \leq K_{T,p}.
\end{align*}
We also have:
\begin{align*}
 & \mathbf{E}\mathbb{E}_{C_{i-1}}\left[\mathbb{E}_{J_{i}}\left[w_{i}(t,C_{i-1},C_{i}(J_{i}))\tilde{G}_{i}^{\partial\tilde{H}}(t,C_{i}(J_{i}),x)\right]^{2p}\right]\\
 & \leq\mathbf{E}\mathbb{E}_{C_{i-1}}\mathbb{E}_{J_{i}}\left[\left|w_{i}(t,C_{i-1},C_{i}(J_{i}))\tilde{G}_{i}^{\partial\tilde{H}}(t,C_{i}(J_{i}),x)\right|^{2p}\right]\\
 & \leq\mathbf{E}\mathbb{E}_{C_{i-1}}\mathbb{E}_{J_{i}}\left[\left|w_{i}(t,C_{i-1},C_{i}(J_{i}))\right|^{4p}\right]^{1/2}\mathbf{E}\mathbb{E}_{J_{i}}\left[\left|\tilde{G}_{i}^{\partial\tilde{H}}(t,C_{i}(J_{i}),x)\right|^{4p}\right]^{1/2}\\
 & \leq K_{T,p}
\end{align*}
by the induction hypothesis, and similarly,
\[
\mathbf{E}\mathbb{E}_{J_{i-1}}\left[\mathbb{E}_{J_{i}}\left[w_{i}(t,C_{i-1}\left(J_{i-1}\right),C_{i}(J_{i}))\tilde{G}_{i}^{\partial\tilde{H}}(t,C_{i}(J_{i}),x)\right]^{2p}\right]\leq K_{T,p}.
\]
This completes the proof.
\end{proof}

\subsection{Relation among $\tilde{G}$ quantities}
\begin{lem}
\label{lem:tildeH-recursion}Under Assumption \ref{Assump:Assumption_1},
we have for $t\leq T$,
\begin{align*}
\tilde{G}_{1}^{\tilde{H}}(t,c_{1},x) & =0,\\
\tilde{G}_{i}^{\tilde{H}}(t,c_{i},x) & =\tilde{G}_{i}^{H}(t,c_{i},x)+\mathbb{E}_{C_{i-1}}\left[w_{i}(t,C_{i-1},c_{i})\varphi_{i-1}'(H_{i-1}(t,x,C_{i-1}))\tilde{G}_{i-1}^{\tilde{H}}(t,x,C_{i-1})\right]+{\bf O}_{T}(N^{-1/2}).
\end{align*}
\end{lem}

\begin{proof}
We have 
\[
\tilde{G}_{1}^{\tilde{H}}(t,c_{1},x)=\sqrt{N}(\tilde{H}_{1}(t,x,c_{1})-H_{1}(t,x,c_{1}))=0,
\]
and
\begin{align*}
 & \tilde{G}_{i}^{\tilde{H}}(t,c_{i},x)\\
 & =\sqrt{N}\left(\tilde{H}_{i}(t,x,c_{i})-H_{i}(t,x,c_{i})\right)\\
 & =\sqrt{N}\left(\mathbb{E}_{J_{i-1}}[w_{i}(t,C_{i-1}(J_{i-1}),c_{i})\varphi_{i-1}(\tilde{H}_{i-1}(t,x,C_{i-1}(J_{i-1})))]-\mathbb{E}_{C_{i-1}}[w_{i}(t,C_{i-1},c_{i})\varphi_{i-1}(H_{i-1}(t,x,C_{i-1}))]\right)\\
 & =\sqrt{N}\tilde{G}_{i}^{H}(t,c_{i},x)\\
 & \quad+\sqrt{N}\mathbb{E}_{J_{i-1}}\left[w_{i}(t,C_{i-1}(J_{i-1}),c_{i})\left(\varphi_{i-1}(\tilde{H}_{i-1}(t,x,C_{i-1}(J_{i-1})))-\varphi_{i-1}(H_{i-1}(t,x,C_{i-1}(J_{i-1})))\right)\right].
\end{align*}
We note that by the mean value theorem, for some $h_{i-1}\left(t,x,C_{i-1}(J_{i-1})\right)$
between $H_{i-1}(t,x,C_{i-1}(J_{i-1}))$ and $\tilde{H}_{i-1}(t,x,C_{i-1}(J_{i-1}))$:
\begin{align*}
 & \mathbb{E}_{C_{i}}\mathbf{E}\bigg[\mathbb{E}_{J_{i-1}}\bigg[\sqrt{N}w_{i}(t,C_{i-1}(J_{i-1}),C_{i})\left(\varphi_{i-1}(\tilde{H}_{i-1}(t,x,C_{i-1}(J_{i-1})))-\varphi_{i-1}(H_{i-1}(t,x,C_{i-1}(J_{i-1})))\right)\\
 & \qquad-w_{i}(t,C_{i-1}(J_{i-1}),C_{i})\varphi_{i-1}'(H_{i-1}(t,x,C_{i-1}(J_{i-1})))\tilde{G}_{i-1}^{\tilde{H}}(t,x,C_{i-1}(J_{i-1}))\bigg]^{2}\bigg]\\
 & \leq\mathbb{E}_{C_{i}}\mathbf{E}\bigg[\mathbb{E}_{J_{i-1}}\bigg[w_{i}(t,C_{i-1}(J_{i-1}),C_{i})\left(\varphi_{i-1}'\left(h_{i-1}\left(t,x,C_{i-1}(J_{i-1})\right)\right)-\varphi_{i-1}'\left(H_{i-1}\left(t,x,C_{i-1}(J_{i-1})\right)\right)\right)\tilde{G}_{i-1}^{\tilde{H}}(t,x,C_{i-1}(J_{i-1}))\bigg]^{2}\bigg]\\
 & \leq K\mathbb{E}_{C_{i}}\mathbf{E}\bigg[\mathbb{E}_{J_{i-1}}\bigg[\left|w_{i}(t,C_{i-1}(J_{i-1}),C_{i})\right|\left|h_{i-1}\left(t,x,C_{i-1}(J_{i-1})\right)-H_{i-1}\left(t,x,C_{i-1}(J_{i-1})\right)\right|\left|\tilde{G}_{i-1}^{\tilde{H}}(t,x,C_{i-1}(J_{i-1}))\right|\bigg]^{2}\bigg]\\
 & \leq K\mathbb{E}_{C_{i}}\mathbf{E}\bigg[\mathbb{E}_{J_{i-1}}\bigg[\left|w_{i}(t,C_{i-1}(J_{i-1}),C_{i})\right|\left(\tilde{G}_{i-1}^{\tilde{H}}(t,x,C_{i-1}(J_{i-1}))\right)^{2}\bigg]^{2}\bigg]/N\\
 & \leq K_{T}/N,
\end{align*}
by Lemmas \ref{lem:MF_a_priori} and \ref{lem:tildeH-moment}. Thus,
it remains to show that 
\begin{align*}
 & \mathbb{E}_{C_{i}}{\bf E}\bigg[\bigg(\mathbb{E}_{J_{i-1}}\left[w_{i}(t,C_{i-1}(J_{i-1}),C_{i})\varphi_{i-1}'(H_{i-1}(t,x,C_{i-1}(J_{i-1})))\tilde{G}_{i-1}^{\tilde{H}}(t,x,C_{i-1}(J_{i-1}))\right]\\
 & \qquad-\mathbb{E}_{C_{i-1}}\left[w_{i}(t,C_{i-1},C_{i})\varphi_{i-1}'(H_{i-1}(t,x,C_{i-1}))\tilde{G}_{i-1}^{\tilde{H}}(t,x,C_{i-1})\right]\bigg)^{2}\bigg]=O_{T}(1/N).
\end{align*}
This is shown by applying Lemma \ref{lem:second-moment-dependency},
where we take $\mathbf{i}=i-1$, $\mathbf{i}'=i$, $S'=\left\{ C_{k}\left(j_{k}\right):\;j_{k}\in\left[N_{k}\right],\;k<i-1\right\} $
and 
\[
f\left(S,c_{i-1},c_{i}\right)=w_{i}(t,c_{i-1},c_{i})\varphi_{i-1}'(H_{i-1}(t,x,c_{i-1}))\tilde{G}_{i-1}^{\tilde{H}}(t,x,c_{i-1}).
\]
Observe that $f\left(S,c_{i-1},c_{i}\right)$ is independent of $C_{i-1}\left(j_{i-1}\right)$.
In addition:
\[
\mathbf{E}\mathbb{E}_{C_{i}}\left[\left(f\left(S^{j_{i-1}},C_{i-1}\left(j_{i-1}\right),C_{i}\right)\right)^{2}\right]\leq K\mathbf{E}\mathbb{E}_{C_{i}}\left[\left(w_{i}(t,C_{i-1}\left(j_{i-1}\right),C_{i})\right)^{4}\right]^{1/2}\mathbf{E}\left[\left(\tilde{G}_{i-1}^{\tilde{H}}(t,x,C_{i-1}(j_{i-1}))\right)^{4}\right]^{1/2}\leq K_{T}
\]
from Lemma \ref{lem:tildeH-moment}. This shows the claim.
\end{proof}
\begin{lem}
\label{lem:partial-tildeH-recursion}Under Assumption \ref{Assump:Assumption_1},
we have for any $t\leq T$,

\begin{align*}
\tilde{G}_{L}^{\partial\tilde{H}}(t,c_{L},x) & =\tilde{G}_{L}^{\tilde{H}}\left(t,x,c_{L}\right)\varphi_{L}''(H_{L}(t,x,c_{L}))+{\bf O}_{T}(N^{-1/4}),\\
\tilde{G}_{i-1}^{\partial\tilde{H}}(t,c_{i-1},x) & =\tilde{G}_{i-1}^{\partial H}(t,c_{i-1},x)+\tilde{G}_{i-1}^{\tilde{H}}(t,c_{i-1},x)\varphi_{i-1}''(H_{i-1}(t,x,c_{i-1}))\mathbb{E}_{C_{i}}\left[w_{i}(t,c_{i-1},C_{i})\frac{\partial\hat{y}(t,x)}{\partial H_{i}(C_{i})}\right]\\
 & \qquad+\varphi_{i-1}'(H_{i-1}(t,x,c_{i-1}))\mathbb{E}_{C_{i}}\left[w_{i}(t,c_{i-1},C_{i})\tilde{G}_{i}^{\partial\tilde{H}}(t,C_{i},x)\right]+{\bf O}_{T}(N^{-1/4}).
\end{align*}
\end{lem}

\begin{proof}
We have 
\begin{align*}
 & \tilde{G}_{i-1}^{\partial\tilde{H}}(t,c_{i-1},x)\\
 & =\sqrt{N}\left(\frac{\partial\tilde{y}(t,x)}{\partial\tilde{H}_{i-1}(c_{i-1})}-\frac{\partial\hat{y}(t,x)}{\partial H_{i-1}(c_{i-1})}\right)\\
 & =\sqrt{N}\left(\mathbb{E}_{J_{i}}\left[w_{i}(t,c_{i-1},C_{i}(J_{i}))\varphi_{i-1}'(\tilde{H}_{i-1}(t,x,c_{i-1}))\frac{\partial\tilde{y}(t,x)}{\partial\tilde{H}_{i}(C_{i}(J_{i}))}\right]-\mathbb{E}_{C_{i}}\left[w_{i}(t,c_{i-1},C_{i})\varphi_{i-1}'(H_{i-1}(t,x,c_{i-1}))\frac{\partial\hat{y}(t,x)}{\partial H_{i}(C_{i})}\right]\right)\\
 & =\tilde{G}_{i-1}^{\partial H}(t,c_{i-1},x)\\
 & \quad+\sqrt{N}\left(\varphi_{i-1}'(\tilde{H}_{i-1}(t,x,c_{i-1}))-\varphi_{i-1}'(H_{i-1}(t,x,c_{i-1}))\right)\mathbb{E}_{J_{i}}\left[w_{i}(t,c_{i-1},C_{i}(J_{i}))\frac{\partial\hat{y}(t,x)}{\partial H_{i}(C_{i}(J_{i}))}\right]\\
 & \quad+\varphi_{i-1}'(H_{i-1}(t,x,c_{i-1}))\mathbb{E}_{J_{i}}\left[w_{i}(t,c_{i-1},C_{i}(J_{i}))\tilde{G}_{i}^{\partial\tilde{H}}(t,C_{i}(J_{i}),x)\right]\\
 & \quad+\left(\varphi_{i-1}'(\tilde{H}_{i-1}(t,x,c_{i-1}))-\varphi_{i-1}'(H_{i-1}(t,x,c_{i-1}))\right)\mathbb{E}_{J_{i}}\left[w_{i}(t,c_{i-1},C_{i}(J_{i}))\tilde{G}_{i}^{\partial\tilde{H}}(t,C_{i}(J_{i}),x)\right].
\end{align*}
We proceed with analyzing each term.

\paragraph*{First estimate.}

Firstly we claim that 
\begin{align*}
 & {\bf E}\mathbb{E}_{C_{i-1}}\bigg[\bigg(\sqrt{N}\left(\varphi_{i-1}'(\tilde{H}_{i-1}(t,x,C_{i-1}))-\varphi_{i-1}'(H_{i-1}(t,x,C_{i-1}))\right)\mathbb{E}_{J_{i}}\left[w_{i}(t,C_{i-1},C_{i}(J_{i}))\frac{\partial\hat{y}(t,x)}{\partial H_{i}(C_{i}(J_{i}))}\right]\\
 & \qquad-\sqrt{N}\left(\varphi_{i-1}'(\tilde{H}_{i-1}(t,x,C_{i-1}))-\varphi_{i-1}'(H_{i-1}(t,x,C_{i-1}))\right)\mathbb{E}_{C_{i}}\left[w_{i}(t,C_{i-1},C_{i})\frac{\partial\hat{y}(t,x)}{\partial H_{i}(C_{i})}\right]\bigg)^{2}\bigg]\\
 & =O_{T}(1/N).
\end{align*}
Indeed, this is obvious for $i=L$, and for $i<L$, we verify the
condition of Lemma \ref{lem:second-moment-dependency} for ${\bf i}=i$,
${\bf i}'=i-1$, $S'=\left\{ C_{k}\left(j_{k}\right):\;j_{k}\in\left[N_{k}\right],\;k<i\right\} $
and 
\[
f(S,c_{i},c_{i-1})=\sqrt{N}\left(\varphi_{i-1}'(\tilde{H}_{i-1}(t,x,c_{i-1}))-\varphi_{i-1}'(H_{i-1}(t,x,c_{i-1}))\right)w_{i}(t,c_{i-1},c_{i})\frac{\partial\hat{y}(t,x)}{\partial H_{i}(c_{i})}.
\]
Observe that $f$ is independent of $C_{i}\left(j_{i}\right)$. In
this case, we only need to bound the following:
\begin{align*}
 & {\bf E}\mathbb{E}_{C_{i-1},C_{i}}\left[\left(\sqrt{N}\left(\varphi_{i-1}'(\tilde{H}_{i-1}(t,x,C_{i-1}))-\varphi_{i-1}'(H_{i-1}(t,x,C_{i-1}))\right)w_{i}(t,C_{i-1},C_{i})\frac{\partial\hat{y}(t,x)}{\partial H_{i}(C_{i})}\right)^{2}\right]\\
 & \le K{\bf E}\mathbb{E}_{C_{i-1}}\left[\left(\sqrt{N}(\tilde{H}_{i-1}(t,x,C_{i-1})-H_{i-1}(t,x,C_{i-1}))\right)^{4}\right]+K{\bf E}\mathbb{E}_{C_{i-1},C_{i}}\left[\left(w_{i}(t,C_{i-1},C_{i})\frac{\partial\hat{y}(t,x)}{\partial H_{i}(C_{i})}\right)^{4}\right]\\
 & \le K_{T},
\end{align*}
by Lemmas \ref{lem:MF_a_priori} and \ref{lem:tildeH-moment}. This
proves the claim.

Next, we extend this claim. In particular,
\begin{align*}
 & {\bf E}\mathbb{E}_{C_{i-1},C_{i}}\bigg[\bigg(\sqrt{N}\left(\varphi_{i-1}'(\tilde{H}_{i-1}(t,x,C_{i-1}))-\varphi_{i-1}'(H_{i-1}(t,x,C_{i-1}))\right)\mathbb{E}_{C_{i}}\left[w_{i}(t,C_{i-1},C_{i})\frac{\partial\hat{y}(t,x)}{\partial H_{i}(C_{i})}\right]\\
 & \qquad-\varphi_{i-1}''(H_{i-1}(t,x,C_{i-1}))\tilde{G}_{i-1}^{\tilde{H}}\left(t,x,C_{i-1}\right)\mathbb{E}_{C_{i}}\left[w_{i}(t,C_{i-1},C_{i})\frac{\partial\hat{y}(t,x)}{\partial H_{i}(C_{i})}\right]\bigg)^{2}\bigg]\\
 & \le\frac{K}{N}{\bf E}\mathbb{E}_{C_{i-1},C_{i}}\bigg[\bigg(\left(\tilde{G}_{i-1}^{\tilde{H}}\left(t,x,C_{i-1}\right)\right)^{2}\mathbb{E}_{C_{i}}\left[w_{i}(t,C_{i-1},C_{i})\frac{\partial\hat{y}(t,x)}{\partial H_{i}(C_{i})}\right]\bigg)^{2}\bigg]\\
 & \le\frac{K_{T}}{N},
\end{align*}
by Lemmas \ref{lem:MF_a_priori} and \ref{lem:tildeH-moment}. Therefore,
\begin{align*}
 & {\bf E}\mathbb{E}_{C_{i-1}}\bigg[\bigg(\sqrt{N}\left(\varphi_{i-1}'(\tilde{H}_{i-1}(t,x,C_{i-1}))-\varphi_{i-1}'(H_{i-1}(t,x,C_{i-1}))\right)\mathbb{E}_{J_{i}}\left[w_{i}(t,C_{i-1},C_{i}(J_{i}))\frac{\partial\hat{y}(t,x)}{\partial H_{i}(C_{i}(J_{i}))}\right]\\
 & \qquad-\varphi_{i-1}''(H_{i-1}(t,x,C_{i-1}))\tilde{G}_{i-1}^{\tilde{H}}\left(t,x,C_{i-1}\right)\mathbb{E}_{C_{i}}\left[w_{i}(t,C_{i-1},C_{i})\frac{\partial\hat{y}(t,x)}{\partial H_{i}(C_{i})}\right]\bigg)^{2}\bigg]\\
 & =O_{T}(1/N).
\end{align*}

\paragraph*{Second estimate.}

We have:
\begin{align*}
 & {\bf E}\mathbb{E}_{C_{i-1}}\bigg[\bigg(\varphi_{i-1}'(H_{i-1}(t,x,C_{i-1}))\mathbb{E}_{J_{i}}\left[w_{i}(t,C_{i-1},C_{i}(J_{i}))\tilde{G}_{i}^{\partial\tilde{H}}(t,C_{i}(J_{i}),x)\right]\\
 & \qquad-\varphi_{i-1}'(H_{i-1}(t,x,C_{i-1}))\mathbb{E}_{C_{i}}\left[w_{i}(t,C_{i-1},C_{i})\tilde{G}_{i}^{\partial\tilde{H}}(t,C_{i},x)\right]\bigg)^{2}\bigg]\\
 & =O(N^{-1/2}).
\end{align*}
This is again obvious for $i=L$. For $i<L$, by applying Lemma \ref{lem:second-moment-dependency}
for ${\bf i}=i$, ${\bf i}'=i-1$ and 
\[
f(S,c_{i},c_{i-1})=\varphi_{i-1}'(H_{i-1}(t,x,c_{i-1}))w_{i}(t,c_{i-1},c_{i})\tilde{G}_{i}^{\partial\tilde{H}}(t,c_{i},x),
\]
we have the claim since firstly it is easy to see that by Lemma \ref{lem:MF_a_priori},
\[
\mathbf{E}\mathbb{E}_{C_{i-1}}\left[\left(f\left(S,C_{i}\left(j_{i}\right),C_{i-1}\right)-f\left(S^{j_{i}},C_{i}\left(j_{i}\right),C_{i-1}\right)\right)^{2}\right]=O_{T}\left(1/N\right),
\]
and secondly
\begin{align*}
 & {\bf E}\mathbb{E}_{C_{i-1},C_{i}}\left[\left(\varphi_{i-1}'(\tilde{H}_{i-1}(t,x,C_{i-1}))w_{i}(t,C_{i-1},C_{i})\tilde{G}_{i}^{\partial\tilde{H}}(t,C_{i},x)\right)^{2}\right]\\
 & \le K{\bf E}\mathbb{E}_{C_{i-1},C_{i}}\left[\left(w_{i}(t,C_{i-1},C_{i})\right)^{4}\right]+K{\bf E}\mathbb{E}_{C_{i}}\left[\left(\tilde{G}_{i}^{\partial\tilde{H}}(t,C_{i},x)\right)^{4}\right]\\
 & \le K_{T}.
\end{align*}
by Lemmas \ref{lem:MF_a_priori} and \ref{lem:partial-tildeH-moment}.

\paragraph*{Last estimate.}

Finally we have:
\begin{align*}
 & \mathbf{E}\mathbb{E}_{C_{i-1}}\left[\left(\varphi_{i-1}'(\tilde{H}_{i-1}(t,x,C_{i-1}))-\varphi_{i-1}'(H_{i-1}(t,x,C_{i-1}))\right)^{2}\mathbb{E}_{J_{i}}\left[w_{i}(t,C_{i-1},C_{i}(J_{i}))\tilde{G}_{i}^{\partial\tilde{H}}(t,C_{i}(J_{i}),x)\right]^{2}\right]\\
 & \leq K\mathbf{E}\mathbb{E}_{C_{i-1}}\left[\left(\tilde{G}_{i-1}^{\tilde{H}}(t,x,C_{i-1})\right)^{2}\mathbb{E}_{J_{i}}\left[w_{i}(t,C_{i-1},C_{i}(J_{i}))\tilde{G}_{i}^{\partial\tilde{H}}(t,C_{i}(J_{i}),x)\right]^{2}\right]/N\\
 & \leq K_{T}/N.
\end{align*}
by Lemmas \ref{lem:MF_a_priori}, \ref{lem:tildeH-moment} and \ref{lem:partial-tildeH-moment}.

\paragraph*{Combining estimates.}

Combining all above estimates, we obtain 
\begin{align*}
\tilde{G}_{i-1}^{\partial\tilde{H}}(t,c_{i-1},x) & =\tilde{G}_{i-1}^{\partial H}(t,c_{i-1},x)+\tilde{G}_{i-1}^{\tilde{H}}(t,c_{i-1},x)\varphi_{i-1}''(H_{i-1}(t,x,c_{i-1}))\mathbb{E}_{C_{i}}\left[w_{i}(t,c_{i-1},C_{i})\frac{\partial\hat{y}(t,x)}{\partial H_{i}(C_{i})}\right]\\
 & \quad+\varphi_{i-1}'(H_{i-1}(t,x,c_{i-1}))\mathbb{E}_{C_{i}}\left[w_{i}(t,c_{i-1},C_{i})\tilde{G}_{i}^{\partial\tilde{H}}(t,C_{i},x)\right]+{\bf O}_{T}(N^{-1/4}).
\end{align*}
Note that the claim for $\tilde{G}_{L}^{\partial\tilde{H}}$ follows
by the same argument.
\end{proof}
\begin{lem}
\label{lem:G-w}Under Assumption \ref{Assump:Assumption_1}, we have
for any $t\leq T$,
\begin{align*}
\tilde{G}_{1}^{w}(t,1,c_{1},x) & =\tilde{G}_{1}^{\partial\tilde{H}}(t,c_{1},x)x+{\bf O}_{T}(N^{-1/2}),\\
\tilde{G}_{i}^{w}(t,c_{i-1},c_{i},x) & =\tilde{G}_{i}^{\partial\tilde{H}}(t,c_{i},x)\varphi_{i-1}(H_{i-1}(t,x,c_{i-1}))+\varphi_{i-1}'(H_{i-1}(t,x,c_{i-1}))\frac{\partial\hat{y}(t,x)}{\partial H_{i}(c_{i})}\tilde{G}_{i-1}^{\tilde{H}}(t,c_{i-1},x)+{\bf O}_{T}(N^{-1/2}),\\
\tilde{G}^{y}\left(t,x\right) & =\varphi_{L}'\left(H_{L}\left(t,1,x\right)\right)\tilde{G}_{L}^{\tilde{H}}\left(t,1,x\right)+{\bf O}_{T}(N^{-1/2}).
\end{align*}
\end{lem}

\begin{proof}
We prove the second statement; the claims for $\tilde{G}_{1}^{w}$
and $\tilde{G}^{y}$ can be proven similarly. We have:
\begin{align*}
 & \tilde{G}_{i}^{w}(t,c_{i-1},c_{i},x)\\
 & =\sqrt{N}\left(\frac{\partial\tilde{y}(t,x)}{\partial w_{i}(c_{i-1},c_{i})}-\frac{\partial\hat{y}(t,x)}{\partial w_{i}(c_{i-1},c_{i})}\right)\\
 & =\sqrt{N}\left(\frac{\partial\tilde{y}(t,x)}{\partial\tilde{H}_{i}(c_{i})}-\frac{\partial\hat{y}(t,x)}{\partial H_{i}(c_{i})}\right)\varphi_{i-1}(H_{i-1}(t,x,c_{i-1}))+\sqrt{N}\left(\varphi_{i-1}(\tilde{H}_{i-1}(t,x,c_{i-1}))-\varphi_{i-1}(H_{i-1}(t,x,c_{i-1}))\right)\frac{\partial\hat{y}(t,x)}{\partial H_{i}(c_{i})}\\
 & \qquad+\sqrt{N}\left(\frac{\partial\tilde{y}(t,x)}{\partial\tilde{H}_{i}(c_{i})}-\frac{\partial\hat{y}(t,x)}{\partial H_{i}(c_{i})}\right)\left(\varphi_{i-1}(\tilde{H}_{i-1}(t,x,c_{i-1}))-\varphi_{i-1}(H_{i-1}(t,x,c_{i-1}))\right)\\
 & =\tilde{G}_{i}^{\partial\tilde{H}}(t,c_{i},x)\varphi_{i-1}(H_{i-1}(t,x,c_{i-1}))+\varphi_{i-1}'(H_{i-1}(t,x,c_{i-1}))\frac{\partial\hat{y}(t,x)}{\partial H_{i}(c_{i})}\tilde{G}_{i-1}^{\tilde{H}}(t,c_{i-1},x)+{\bf O}_{T}(N^{-1/2}).
\end{align*}
Here, we have used that by Lemmas \ref{lem:MF_a_priori} and \ref{lem:tildeH-moment},
\begin{align*}
 & {\bf E}\mathbb{E}_{C_{i-1},C_{i}}\bigg[\bigg(\sqrt{N}\left(\varphi_{i-1}(\tilde{H}_{i-1}(t,x,C_{i-1}))-\varphi_{i-1}(H_{i-1}(t,x,C_{i-1}))\right)\frac{\partial\hat{y}(t,x)}{\partial H_{i}(C_{i})}\\
 & \qquad-\varphi_{i-1}'(H_{i-1}(t,x,C_{i-1}))\frac{\partial\hat{y}(t,x)}{\partial H_{i}(C_{i})}\tilde{G}_{i-1}^{\tilde{H}}(t,C_{i-1},x)\bigg)^{2}\bigg]\\
 & \le\frac{K}{N}{\bf E}\mathbb{E}_{C_{i-1},C_{i}}\bigg[\left(\tilde{G}_{i-1}^{\tilde{H}}(t,C_{i-1},x)\frac{\partial\hat{y}(t,x)}{\partial H_{i}(C_{i})}\right)^{2}\bigg]\\
 & \le\frac{K_{T}}{N},
\end{align*}
 and that by Lemmas \ref{lem:tildeH-moment} and \ref{lem:partial-tildeH-moment},
\begin{align*}
 & {\bf E}\mathbb{E}_{C_{i-1},C_{i}}\bigg[\bigg(\sqrt{N}\left(\frac{\partial\tilde{y}(t,x)}{\partial\tilde{H}_{i}(C_{i})}-\frac{\partial\hat{y}(t,x)}{\partial H_{i}(C_{i})}\right)\left(\varphi_{i-1}(\tilde{H}_{i-1}(t,x,C_{i-1}))-\varphi_{i-1}(H_{i-1}(t,x,C_{i-1}))\right)\bigg)^{2}\bigg]\\
 & \le\frac{K}{N}{\bf E}\mathbb{E}_{C_{i-1},C_{i}}\left[\left(\tilde{G}_{i-1}^{\tilde{H}}(t,C_{i-1},x)\tilde{G}_{i}^{\partial\tilde{H}}(t,x,C_{i})\right)^{2}\right]\\
 & \le\frac{K_{T}}{N}.
\end{align*}
\end{proof}

\subsection{Structure of the limiting Gaussian process and the proof of Theorem
\ref{thm:G_tilde}}

Each process $\tilde{G}_{i}^{\partial H},\;\tilde{G}_{i}^{H},\;\tilde{G}_{i}^{\partial\tilde{H}},\;\tilde{G}_{i}^{\tilde{H}},\;\tilde{G}_{i}^{w},\;\tilde{G}^{y}$
on $t\leq T$ can be viewed as an element of ${\cal G}_{T}=C([0,T],L^{2}(\mathbb{X}\times\tilde{P}))$
where $\tilde{P}$ is $P_{i}$ or $P_{i-1}\times P_{i}$. Equip ${\cal G}_{T}$
with the norm 
\[
\|g\|_{T}=\sup_{t\in[0,T]}\mathbb{E}\left[\left\Vert g\left(t\right)\right\Vert ^{2}\right]^{1/2}.
\]
We define the following Gaussian processes $\underline{G}^{\partial H},\;\underline{G}^{H}$
with zero mean and covariance given by 
\begin{align*}
 & \mathbf{E}\left[\underline{G}_{i}^{H}(t,c_{i},x)\underline{G}_{j}^{H}(t',c_{j}',x')\right]\\
 & =\begin{cases}
\delta_{ij}\mathbb{C}_{C_{i-1}}\left[w_{i}(t,C_{i-1},c_{i})\varphi_{i-1}(H_{i-1}(t,C_{i-1},x));\;w_{j}(t',C_{i-1},c'_{j})\varphi_{i-1}(H_{i-1}(t',C_{i-1},x'))\right], & i>1,\\
0, & {\rm otherwise,}
\end{cases}\\
 & \mathbf{E}\left[\underline{G}_{i}^{\partial H}(t,c_{i},x)\underline{G}_{j}^{\partial H}(t',c_{j}',x')\right]\\
 & =\begin{cases}
\delta_{ij}\mathbb{C}_{C_{i+1}}\left[\varphi_{i}'(H_{i}(t,x,c_{i}))w_{i}(t,c_{i},C_{i+1}){\displaystyle \frac{\partial\hat{y}(t,x)}{\partial H_{i+1}(C_{i+1})}};\;\varphi_{i}'(H_{i}(t',x',c_{j}'))w_{i}(t',c_{j}',C_{i+1}){\displaystyle \frac{\partial\hat{y}(t',x')}{\partial H_{i+1}(C_{i+1})}}\right], & i<L,\\
0, & {\rm otherwise,}
\end{cases}\\
 & \mathbf{E}\left[\underline{G}_{i}^{H}(t,c_{i},x)\underline{G}_{j}^{\partial H}(t',c_{j}',x')\right]\\
 & =\begin{cases}
\delta_{i,j+2}\mathbb{C}_{C_{i-1}}\left[\varphi_{i-2}'(H_{i-2}(t,x,c_{j}'))w_{i-1}(t,c_{j}',C_{i-1}){\displaystyle \frac{\partial\hat{y}(t,x)}{\partial H_{i-1}(C_{i-1})}};\;w_{i}(t,C_{i-1},c_{i})\varphi_{i-1}(H_{i-1}(t,x,C_{i-1}))\right], & i>2,\\
0, & {\rm otherwise.}
\end{cases}
\end{align*}
As a side remark, note that the covariance of $\underline{G}^{\partial H}$
and $\underline{G}^{H}$ can be recognized through the following simple
rule: for functions $f(C_{i})$ and $g(C_{j})$, we have 
\begin{equation}
{\bf E}\left[\sqrt{N}(\mathbb{E}_{J_{i}}[f(C_{i}(J_{i}))]-\mathbb{E}_{C_{i}}[f(C_{i})])\cdot\sqrt{N}(\mathbb{E}_{J_{j}}[g(C_{j}(J_{j}))]-\mathbb{E}_{C_{j}}[g(C_{j})])\right]=\delta_{ij}(\mathbb{E}_{C_{i}}[f(C_{i})g(C_{i})]-\mathbb{E}_{C_{i}}[f(C_{i})]\mathbb{E}_{C_{i}}[g(C_{i})]).\label{eq:rule-covariance}
\end{equation}
Define $\underline{G}^{\tilde{H}},\underline{G}^{\partial\tilde{H}},\underline{G}^{w},\underline{G}^{y}$
by the following linear combination of $\underline{G}^{H}$ and $\underline{G}^{\partial H}$:
\begin{align*}
\underline{G}_{1}^{\tilde{H}}(t,c_{1},x) & =0,\\
\underline{G}_{i}^{\tilde{H}}(t,c_{i},x) & =\underline{G}_{i}^{H}(t,c_{i},x)+\mathbb{E}_{C_{i-1}}\left[w_{i}(t,C_{i-1},c_{i})\varphi_{i-1}'(H_{i-1}(t,x,C_{i-1}))\underline{G}_{i-1}^{\tilde{H}}(t,x,C_{i-1})\right],\\
\underline{G}_{L}^{\partial\tilde{H}}(t,c_{L},x) & =\varphi_{L}''(H_{L}(t,x,c_{L}))\underline{G}_{L}^{\tilde{H}}(t,x,c_{L}),\\
\underline{G}_{i-1}^{\partial\tilde{H}}(t,c_{i-1},x) & =\underline{G}_{i-1}^{\partial H}(t,c_{i-1},x)+\underline{G}_{i-1}^{\tilde{H}}(t,c_{i-1},x)\varphi_{i-1}''(H_{i-1}(t,x,c_{i-1}))\mathbb{E}_{C_{i}}\left[w_{i}(t,c_{i-1},C_{i})\frac{\partial\hat{y}(t,x)}{\partial H_{i}(C_{i})}\right]\\
 & \qquad+\varphi_{i-1}'(H_{i-1}(t,x,c_{i-1}))\mathbb{E}_{C_{i}}\left[w_{i}(t,c_{i-1},C_{i}(J_{i}))\underline{G}_{i}^{\partial\tilde{H}}(t,C_{i}(J_{i}),x)\right],\\
\underline{G}_{1}^{w}(t,1,c_{1},x) & =\underline{G}_{1}^{\partial\tilde{H}}\left(t,c_{1},x\right)x,\\
\underline{G}_{i}^{w}(t,c_{i-1},c_{i},x) & =\underline{G}_{i}^{\partial\tilde{H}}(t,c_{i},x)\varphi_{i-1}(H_{i-1}(t,x,c_{i-1}))+\varphi_{i-1}'(H_{i-1}(t,x,c_{i-1}))\frac{\partial\hat{y}(t,x)}{\partial H_{i}(c_{i})}\underline{G}_{i-1}^{\tilde{H}}(t,c_{i-1},x),\\
\underline{G}^{y}(t,x) & =\varphi_{L}'(H_{L}(t,x,1))\underline{G}_{L}^{\tilde{H}}(t,x,1).
\end{align*}

Recall that we say $\tilde{G}$ converges $G$-polynomially in moment
to a Gaussian process $\underline{G}$ if for any square-integrable
$f_{j}:\;\mathbb{T}\times\Omega_{i(j)-1}\times\Omega_{i(j)}\times\mathbb{X}\to\mathbb{R}$
which is continuous in time,
\[
\sup_{t_{j}\leq T}\bigg|\mathbf{E}\bigg[\prod_{j}\langle f_{j},\tilde{G}_{i(j)}^{\alpha_{j}}\rangle_{t_{j}}^{\beta_{j}}\bigg]-\mathbf{E}\bigg[\prod_{j}\langle f_{j},\underline{G}_{i(j)}^{\alpha_{j}}\rangle_{t_{j}}^{\beta_{j}}\bigg]\bigg|=O_{D,T}\big(\sup_{t\leq T}\max_{j}\|f_{j}\|_{t}^{D}\big)\cdot N^{-1/8},
\]
where $D=\sum_{j}\alpha_{j}\beta_{j}$. (For $\tilde{G}^{H}$, $\tilde{G}^{\tilde{H}}$,
$\tilde{G}^{\partial H}$and $\tilde{G}^{\partial\tilde{H}}$, we
take $f_{j}:\mathbb{T}\times\Omega_{i(j)}\times\mathbb{X}\to\mathbb{R}$.)
We restate and prove Theorem \ref{thm:G_tilde}.

\textsl{}
\begin{thm}[Theorem \ref{thm:G_tilde} restated]
\label{thm:clt-G}We have $\tilde{G}$ converges $G$-polynomially
in moment to $\underline{G}$, where $\underline{G}$ is the Gaussian
process with mean and covariance structure as given above.
\end{thm}

\begin{proof}
First, we show the following:
\[
\left|\mathbf{E}\left[\prod_{\ell}\left\langle f_{\ell},\left(\tilde{G}_{i(\ell)}^{\tau(\ell)}\right)^{\alpha_{\ell}}\right\rangle _{t_{\ell}}^{\beta_{\ell}}\right]-\mathbf{E}\left[\prod_{\ell}\left\langle f_{\ell},\left(\underline{G}_{i(\ell)}^{\tau(\ell)}\right)^{\alpha_{\ell}}\right\rangle _{t_{\ell}}^{\beta_{\ell}}\right]\right|=O_{D}(\max_{\ell}\|f_{\ell}\|_{t_{\ell}}^{D})\cdot N^{-1/2},
\]
where for each $\ell$, $\tau(\ell)$ is either $H$ or $\partial H$.
By independence of $C_{i}(j_{i})$ for distinct $i$'s, it suffices
to consider the case where $i(\ell)=i$ for all $\ell$.

Consider the case $\tau(\ell)=H$ and $i\left(\ell\right)=i$, we
have 
\begin{align*}
\left\langle f_{\ell},\left(\tilde{G}_{i}^{H}\right)^{\alpha_{\ell}}\right\rangle _{t_{\ell}} & =\mathbb{E}_{Z,C_{i}}\left[f_{\ell}(t_{\ell},C_{i},X)\left(\sqrt{N}\mathbb{E}_{J_{i-1}}\left[Z_{i}\left(t_{\ell},J_{i-1},C_{i},X\right)\right]\right)^{\alpha_{\ell}}\right],
\end{align*}
where we denote
\[
Z_{i}(t_{\ell},j_{i-1},c_{i},x)=w_{i}(t_{\ell},C_{i-1}(j_{i-1}),c_{i})\varphi_{i-1}(H_{i-1}(t_{\ell},x,C_{i-1}(j_{i-1})))-\mathbb{E}_{C_{i-1}}[w_{i}(t_{\ell},C_{i-1},c_{i})\varphi_{i-1}(H_{i-1}(t_{\ell},x,C_{i-1}))].
\]
One can write $\prod_{\ell}\left\langle f_{\ell},\left(\tilde{G}_{i}^{H}\right)^{\alpha_{\ell}}\right\rangle _{t_{\ell}}^{\beta_{\ell}}$
as a sum of terms of the form:
\[
N^{-D/2}\prod_{\ell}\prod_{r\leq\beta_{\ell}}\mathbb{E}_{Z,C_{i}}\bigg[f_{\ell}(t_{\ell},C_{i},X)\prod_{h\le\alpha_{\ell}}Z_{i}(t_{\ell},j_{\ell,r,h},C_{i},X)\bigg].
\]
Note that ${\bf E}\mathbb{E}_{J_{i-1}}\left[Z_{i}(\cdot,J_{i-1},C_{i},\cdot)\right]=0$.
Thus, if there is any index $j_{i-1}$ for which $j_{i-1}$ appears
exactly once among $j_{\ell,r,h}$, then the above term after taking
$\mathbf{E}$ vanishes. The number of terms with each $j_{i-1}$ appearing
either zero or exactly twice among $j_{\ell,r,h}$ is $\Theta(N^{D/2})$.
Notice that these terms make up the quantity $\mathbf{E}\bigg[\prod_{\ell}\Big<f_{\ell},\left(\underline{G}_{i(\ell)}^{\tau(\ell)}\right)^{\alpha_{\ell}}\Big>_{t_{\ell}}^{\beta_{\ell}}\bigg]$,
which is to be subtracted away. The number of remaining terms where
no $j_{i-1}$ appears exactly once is at most $O(N^{(D-1)/2})$. Thus
we are done upon bounding
\[
{\bf E}\bigg[\bigg|\prod_{\ell}\prod_{r\leq\beta_{\ell}}\mathbb{E}_{Z,C_{i}}\bigg[f_{\ell}(t_{\ell},C_{i},X)\prod_{h\le\alpha_{\ell}}Z_{i}(t_{\ell},j_{\ell,r,h},C_{i},X)\bigg]\bigg|\bigg].
\]
By Lemma \ref{lem:MF_a_priori}, it is easy to see that
\[
{\bf E}\mathbb{E}_{X,C_{i}}\left[|Z_{i}(t_{\ell},j_{\ell,r,h},C_{i},X)|^{D}\right]\leq K_{T,D}.
\]
Then:
\begin{align*}
 & {\bf E}\bigg[\bigg|\prod_{\ell}\prod_{r\leq\beta_{\ell}}\mathbb{E}_{Z,C_{i}}\bigg[f_{\ell}\left(t_{\ell},C_{i},X\right)\prod_{h\le\alpha_{\ell}}Z_{i}\left(t_{\ell},j_{\ell,r,h},C_{i},X\right)\bigg]\bigg|\bigg]\\
 & \le\prod_{\ell}\prod_{r\leq\beta_{\ell}}\mathbb{E}_{Z,C_{i}}\bigg[\Big|f_{\ell}\left(t_{\ell},C_{i},X\right)\Big|^{2}\bigg]^{1/2}{\bf E}\bigg[\mathbb{E}_{Z,C_{i}}\bigg[\prod_{h\le\alpha_{\ell}}\bigg|Z_{i}\left(t_{\ell},j_{\ell,r,h},C_{i},X\right)\bigg|^{2}\bigg]^{1/2}\bigg]\\
 & \leq\max_{\ell}\|f_{\ell}\|_{t_{\ell}}^{D}\prod_{\ell}\prod_{r\leq\beta_{\ell}}{\bf E}\bigg[\mathbb{E}_{Z,C_{i}}\bigg[\prod_{h\le\alpha_{\ell}}\bigg|Z_{i}\left(t_{\ell},j_{\ell,r,h},C_{i},X\right)\bigg|^{2}\bigg]^{1/2}\bigg]\\
 & \leq\max_{\ell}\|f_{\ell}\|_{t_{\ell}}^{D}\prod_{\ell}\prod_{r\leq\beta_{\ell}}{\bf E}\bigg[\prod_{h\le\alpha_{\ell}}\mathbb{E}_{Z,C_{i}}\bigg[\bigg|Z_{i}\left(t_{\ell},j_{\ell,r,h},C_{i},X\right)\bigg|^{2\alpha_{\ell}}\bigg]^{1/\left(2\alpha_{\ell}\right)}\bigg]\\
 & \leq\max_{\ell}\|f_{\ell}\|_{t_{\ell}}^{D}\prod_{\ell}\prod_{r\leq\beta_{\ell}}\prod_{h\le\alpha_{\ell}}{\bf E}\mathbb{E}_{Z,C_{i}}\bigg[\bigg|Z_{i}\left(t_{\ell},j_{\ell,r,h},C_{i},X\right)\bigg|^{2\alpha_{\ell}}\bigg]^{1/\left(2\alpha_{\ell}\right)}\\
 & \leq\max_{\ell}\|f_{\ell}\|_{t_{\ell}}^{D}\cdot K_{T,D}.
\end{align*}
This completes the proof for the case $\tau(\ell)=H$ and $i\left(\ell\right)=i$.

The involvement of terms with $\tau(\ell)=\partial H$ can be dealt
with similarly. To deal with the case $\tau\left(\ell\right)\in\left\{ H,\partial H\right\} $,
we note that the same argument holds with an appropriate modification
of the definition of $Z_{i}$; in particular, this function now depends
on $\ell$ via $\tau\left(\ell\right)$, but this detail does not
affect the argument since all we need is that ${\bf E}\mathbb{E}_{J_{i-1}}\left[Z_{i}(\cdot,J_{i-1},C_{i},\cdot)\right]=0$
and ${\bf E}\mathbb{E}_{X,C_{i}}\left[|Z_{i}(t,\cdot,C_{i},X)|^{D}\right]\leq K_{T,D}.$
This completes the argument that $(\tilde{G}^{\partial H},\tilde{G}^{H})\to(\underline{G}^{\partial H},\underline{G}^{H})$
in $G$-polynomial moment.

Finally, observe that $\tilde{G}_{i}^{\tilde{H}},\;\tilde{G}_{i}^{\partial\tilde{H}},\;\tilde{G}_{i}^{w},\;\tilde{G}^{y}$
can be written as a linear combination of $\tilde{G}_{i}^{H}$ and
$\tilde{G}_{i}^{\partial H}$ by Lemmas \ref{lem:tildeH-recursion},
\ref{lem:partial-tildeH-recursion} and \ref{lem:G-w}, from which
it is easy to verify that $\tilde{G}\to\underline{G}$ in $G$-polynomial
moment.
\end{proof}

\subsection{Dynamical form of the limiting Gaussian process}

We also have the following property of the limiting process $\underline{G}$
that will be useful later. Define further auxiliary processes: 
\begin{align*}
 & \tilde{G}_{i-1}^{\partial_{t}(\partial H)}(t,c_{i-1},x)\\
 & =\sqrt{N}\left(\mathbb{E}_{J_{i}}\left[w_{i}(t,c_{i-1},C_{i}(J_{i}))\frac{\partial\hat{y}(t,x)}{\partial H_{i}(C_{i}(J_{i}))}\right]-\mathbb{E}_{C_{i}}\left[w_{i}(t,c_{i-1},C_{i})\frac{\partial\hat{y}(t,x)}{\partial H_{i}(C_{i})}\right]\right)\varphi_{i-1}''(H_{i-1}(t,x,c_{i-1}))\partial_{t}H_{i-1}(t,x,c_{i-1}),\\
 & \quad+\sqrt{N}\left(\mathbb{E}_{J_{i}}\left[\partial_{t}w_{i}(t,c_{i-1},C_{i}(J_{i}))\frac{\partial\hat{y}(t,x)}{\partial H_{i}(C_{i}(J_{i}))}\right]-\mathbb{E}_{C_{i}}\left[\partial_{t}w_{i}(t,c_{i-1},C_{i})\frac{\partial\hat{y}(t,x)}{\partial H_{i}(C_{i})}\right]\right)\varphi_{i-1}'(H_{i-1}(t,x,c_{i-1}))\\
 & \quad+\sqrt{N}\left(\mathbb{E}_{J_{i}}\left[\partial_{t}w_{i}(t,c_{i-1},C_{i}(J_{i}))\partial_{t}\left(\frac{\partial\hat{y}(t,x)}{\partial H_{i}(C_{i}(J_{i}))}\right)\right]-\mathbb{E}_{C_{i}}\left[\partial_{t}w_{i}(t,c_{i-1},C_{i})\partial_{t}\left(\frac{\partial\hat{y}(t,x)}{\partial H_{i}(C_{i})}\right)\right]\right)\varphi_{i-1}'(H_{i-1}(t,x,c_{i-1})),
\end{align*}
and
\begin{align*}
 & \tilde{G}_{i}^{\partial_{t}H}(t,c_{i},x)\\
 & =\sqrt{N}\left(\mathbb{E}_{J_{i-1}}[\partial_{t}w_{i}(t,C_{i-1}(J_{i-1}),c_{i})\varphi_{i-1}(H_{i-1}(t,x,C_{i-1}(J_{i-1})))]-\mathbb{E}_{C_{i-1}}[\partial_{t}w_{i}(t,C_{i-1},c_{i})\varphi_{i-1}(H_{i-1}(t,x,C_{i-1}))]\right)\\
 & \quad+\sqrt{N}\Big(\mathbb{E}_{J_{i-1}}[w_{i}(t,C_{i-1}(J_{i-1}),c_{i})\varphi_{i-1}'(H_{i-1}(t,x,C_{i-1}(J_{i-1})))\partial_{t}H_{i-1}(t,x,C_{i-1}(J_{i-1}))]\\
 & \qquad\qquad-\mathbb{E}_{C_{i-1}}[w_{i}(t,C_{i-1},c_{i})\varphi_{i-1}'(H_{i-1}(t,x,C_{i-1}))\partial_{t}H_{i-1}(t,x,C_{i-1})]\Big),
\end{align*}
with $\tilde{G}_{L-1}^{\partial_{t}(\partial H)}=\tilde{G}_{L}^{\partial_{t}(\partial H)}=\tilde{G}_{1}^{\partial_{t}H}=0$.
As above, we can show that jointly with previously defined processes
$\tilde{G}^{H},\tilde{G}^{\partial H}$, $\tilde{G}^{\partial_{t}(\partial H)},\tilde{G}^{\partial_{t}H}$
converges to Gaussian processes $\underline{G}^{\partial_{t}(\partial H)},\underline{G}^{\partial_{t}H}$
whose covariance structure (jointly with the previously defined processes
$\underline{G}$) can be deduced by following the rule (\ref{eq:rule-covariance}).
\begin{thm}
\label{thm:clt-partial_t-G}The limiting Gaussian process $(\underline{G}^{H},\underline{G}^{\partial H})$
is the solution to the ODE 
\begin{align*}
\partial_{t}\underline{G}_{i}^{H}(t,c_{i},x) & =\underline{G}_{i}^{\partial_{t}H}(t,c_{i},x),\\
\partial_{t}\underline{G}_{i}^{\partial H}(t,c_{i},x) & =\underline{G}_{i}^{\partial_{t}(\partial H)}(t,c_{i},x).
\end{align*}
Similarly we can express $\partial_{t}\underline{G}_{i}^{\tilde{H}}$
and $\partial_{t}\underline{G}_{i}^{\partial\tilde{H}}$ in terms
of $\partial_{t}G_{i}^{H},\partial_{t}G_{i}^{\partial H}$ using the
expression of $G^{\tilde{H}}$ and $G^{\partial\tilde{H}}$ in terms
of $G^{H}$ and $G^{\partial H}$. Finally, we have 
\[
\partial_{t}\underline{G}^{y}(t,x)=\varphi_{L}'(H_{L}(t,1,x))\partial_{t}\underline{G}_{L}^{\tilde{H}}(t,x)+\partial_{t}H_{L}(t,1,x)\varphi_{L}''(H_{L}(t,1,x))\underline{G}_{L}^{\tilde{H}}(t,x).
\]
\end{thm}

\begin{proof}
Defining $\overline{G}_{i}^{H}(t,c_{i},x),\overline{G}_{i}^{\partial H}(t,c_{i},x)$
as solutions to 
\begin{align*}
\partial_{t}\overline{G}_{i}^{H}(t,c_{i},x) & =G_{i}^{\partial_{t}H}(t,c_{i},x),\\
\partial_{t}\overline{G}_{i}^{\partial H}(t,c_{i},x) & =G_{i}^{\partial_{t}(\partial H)}(t,c_{i},x),
\end{align*}
it is straightforward to verify that $\overline{G}$ has the same
covariance structure as $G$. Since $G^{\tilde{H}},G^{\partial\tilde{H}},G^{w},G^{y}$
are linear combinations of $G^{H},G^{\partial H}$, the conclusions
for $\partial_{t}G$ easily follow.
\end{proof}
\begin{rem}
We can in fact show that jointly with $\tilde{G}^{H}$ and $\tilde{G}^{\partial H}$,
we have $(\partial_{t}\tilde{G}^{H},\partial_{t}\tilde{G}^{\partial H})$
converges in $G$-polynomial moment to $(\partial_{t}G^{H},\partial_{t}G^{\partial H})$,
and this extends by linearity to $\partial_{t}\tilde{G}^{\tilde{H}},\partial_{t}\tilde{G}^{\partial\tilde{H}},\partial_{t}\tilde{G}^{w},\partial_{t}\tilde{G}^{y}$,
by following the procedure leading to Theorem \ref{thm:clt-G}. Since
we will not need to use this fact, we omit the details of the proof.
\end{rem}

\section{Existence and uniqueness of $R$: Proof of Theorem \ref{thm:well-posed}\label{Appendix:proof_well_posed}}

For $G\in{\cal G}$ and $p\ge1$, define 
\[
\|G\|_{T,2p}^{2p}=\sup_{t\le T}\mathbb{E}_{Z,C}\left[\left|G^{y}(t,X)\right|^{2p}+\sum_{i=1}^{L}\left|G_{i}^{w}(t,C_{i-1},C_{i},X)\right|^{2p}\right].
\]

\begin{thm}
\label{thm:exist-R}Under Assumption \ref{Assump:Assumption_1}, for
any $\epsilon>0$ and all $G\in{\cal G}$ with $\|G\|_{T,2+\epsilon}<\infty$,
there exists a unique solution $R_{i}(G,\cdot,\cdot,\cdot)\in L^{2}(P_{i-1}\times P_{i})$
which is continuous in time. Furthermore, for each $t\le T$, $R(G,t)$
is a continuous linear functional in $G$ and $\|R\left(G\right)\|_{T,2}\le K_{T,\epsilon}\|G\|_{T,2+\epsilon}$.
In fact, for each $B$ sufficiently large in $T$, there exists a
sequence in $B$ of processes $R^{B}$ which is a continuous linear
functional in $G$ with 
\[
\|R^{B}(G)\|_{T,2}^{2}\le\exp(K_{T}B)\|G\|_{T,2}^{2}
\]
for all $G$ with $\|G\|_{T,2}<\infty$, and for $G$ with $\|G\|_{T,2+\epsilon}<\infty$,
\[
\|R(G)-R^{B}(G)\|_{T,2}^{2}\le\|G\|_{T,2+\epsilon}^{2}\exp(-K_{T}\epsilon B^{2}/\left(2+\epsilon\right)).
\]
Here we define
\[
\left\Vert R\left(G\right)\right\Vert _{T,2}^{2}=\sup_{t\leq T}\mathbb{E}_{C}\left[\sum_{i=1}^{L}\left|R_{i}\left(G,t,C_{i-1},C_{i}\right)\right|^{2}\right],
\]
and similar for $\|R(G)-R^{B}(G)\|_{T,2}^{2}$.
\end{thm}

The main technical difficulty in the following proof lies in the fact
that the weights $w_{i}(t,c_{i-1},c_{i})$ could be unbounded, in
which case the linear operator $R\mapsto\partial_{t}R$ is unbounded.
Our proof of both existence and uniqueness follows from a delicate
truncation scheme. This scheme requires careful treatment: there is
no a priori bound on $R$ again due to the unboundedness problem,
and as such, usual truncation argument does not work. We remark that
this problem is unique to the multilayer structure; we do not encounter
the same problem for shallow networks.
\begin{proof}[Proof of Theorem \ref{thm:exist-R}]
Let us define:
\begin{align*}
\Delta F_{i}^{\left(1\right)}(R)(G,t,c_{i-1},c_{i}) & =\mathbb{E}_{Z}\bigg[\frac{\partial\hat{y}(t,X)}{\partial w_{i}(c_{i-1},c_{i})}\partial_{2}^{2}{\cal L}\left(Y,\hat{y}\left(t,X\right)\right)\sum_{r=1}^{L}\mathbb{E}_{C}\bigg[R_{r}\left(G,t,C_{r-1},C_{r}\right)\frac{\partial\hat{y}\left(t,X\right)}{\partial w_{r}\left(C_{r-1},C_{r}\right)}\bigg]\bigg]\\
 & \quad+\mathbb{E}_{Z}\bigg[\partial_{2}{\cal L}\left(Y,\hat{y}\left(t,X\right)\right)\mathbb{E}_{C}\bigg[\sum_{r=1}^{L}R_{r}\left(G,t,C_{r-1},C_{r}\right)\frac{\partial^{2}\hat{y}\left(t,X\right)}{\partial w_{r}\left(C_{r-1},C_{r}\right)\partial w_{i}\left(c_{i-1},c_{i}\right)}\bigg]\bigg]\\
 & \quad+\mathbb{E}_{Z}\bigg[\partial_{2}{\cal L}\left(Y,\hat{y}\left(t,X\right)\right)\mathbb{E}_{C}\bigg[R_{i-1}\left(G,t,C_{i-2},c_{i-1}\right)\frac{\partial^{2}\hat{y}\left(t,X\right)}{\partial_{*}w_{i-1}\left(C_{i-2},c_{i-1}\right)\partial w_{i}\left(c_{i-1},c_{i}\right)}\bigg]\bigg]\\
 & \quad+\mathbb{E}_{Z}\bigg[\partial_{2}{\cal L}\left(Y,\hat{y}\left(t,X\right)\right)\mathbb{E}_{C}\bigg[R_{i}\left(G,t,C_{i-1},c_{i}\right)\frac{\partial^{2}\hat{y}\left(t,X\right)}{\partial_{*}w_{i}\left(C_{i-1},c_{i}\right)\partial w_{i}\left(c_{i-1},c_{i}\right)}\bigg]\bigg]\\
 & \quad+\mathbb{E}_{Z}\bigg[\partial_{2}{\cal L}\left(Y,\hat{y}\left(t,X\right)\right)\mathbb{E}_{C}\bigg[R_{i+1}\left(G,t,c_{i},C_{i+1}\right)\frac{\partial^{2}\hat{y}\left(t,X\right)}{\partial_{*}w_{i+1}\left(c_{i},C_{i+1}\right)\partial w_{i}\left(c_{i-1},c_{i}\right)}\bigg]\bigg],\\
\Delta F_{i}^{\left(2\right)}(R)(G,t,c_{i-1},c_{i}) & =\mathbb{E}_{Z}\bigg[\frac{\partial\hat{y}(t,X)}{\partial w_{i}(c_{i-1},c_{i})}\partial_{2}^{2}{\cal L}\left(Y,\hat{y}\left(t,X\right)\right)G^{y}\left(t,X\right)\bigg]\\
 & \quad+\mathbb{E}_{Z}\bigg[\partial_{2}{\cal L}\left(Y,\hat{y}\left(t,X\right)\right)G_{i}^{w}\left(t,c_{i-1},c_{i},X\right)\bigg],\\
\Delta F_{i}(R)(G,t,c_{i-1},c_{i}) & =\Delta F_{i}^{\left(1\right)}(R)(G,t,c_{i-1},c_{i})+\Delta F_{i}^{\left(2\right)}(R)(G,t,c_{i-1},c_{i}).
\end{align*}
In this notation, $\partial_{t}R_{i}(G,t,c_{i-1},c_{i})=-\Delta F_{i}(R)(G,t,c_{i-1},c_{i})$.
Let 
\begin{align*}
|R|_{t,2p}^{2p} & =\sum_{j=1}^{L}\mathbb{E}_{C}[\left|R_{j}(G,t,C_{j-1},C_{j})\right|^{2p}],\\
|R_{i+1}(c_{i})|_{t,2p}^{2p} & =\mathbb{E}_{C_{i+1}}[\left|R_{i+1}(G,t,c_{i},C_{i+1})\right|^{2p}],\\
|R_{i}(c_{i})|_{t,2p}^{2p} & =\mathbb{E}_{C_{i-1}}[\left|R_{i}(G,t,C_{i-1},c_{i})\right|{}^{2p}].
\end{align*}
We define $|w_{i+1}(c_{i})|_{t,2p},\;|w_{i}(c_{i})|_{t,2p},\;|G_{i}^{w}(c_{i})|_{t,2p},\;|G_{i}^{w}(c_{i-1})|_{t,2p},\;|G_{i}^{y}|_{t,2p}$
similarly (where the last three should include $\mathbb{E}_{Z}$ in
addition). We define
\[
\left|G_{i}^{w}\right|_{t,2p}^{2p}=\mathbb{E}_{C}\left[\left|G_{i}^{w}\left(C_{i}\right)\right|_{t,2p}^{2p}+\left|G_{i}^{w}\left(C_{i-1}\right)\right|_{t,2p}^{2p}\right],\quad|G|_{t,2p}^{2p}=\left|G^{y}\right|_{t,2p}^{2p}+\sum_{i=1}^{L}\left|G_{i}^{w}\right|_{t,2p}^{2p}.
\]
When we drop the subscript $2p$, we implicitly assume that $2p=2$.

\paragraph*{Step 1: truncated process $R^{B}$.}

For each threshold $B>1$, let 
\begin{align*}
\mathbb{B}^{B}(t,c_{i-1},c_{i}) & =\mathbb{I}(\max(|w_{i+1}(c_{i})|_{t},|w_{i}(c_{i})||_{t},|w_{i}(c_{i-1})|_{t},|w_{i-1}(c_{i-1})|_{t})\le B)\\
\mathbb{B}^{B}(t,c_{i}) & =\mathbb{I}(\max(|w_{i+1}(c_{i})|_{t},|w_{i}(c_{i})|_{t})\le B).
\end{align*}
We define $R^{B}$ as the solution to 
\[
\partial_{t}R_{i}^{B}(G,t,c_{i-1},c_{i})]=-\Delta F_{i}(R^{B})(G,t,c_{i-1},c_{i})\cdot\mathbb{B}^{B}(t,c_{i-1},c_{i}).
\]
We can rewrite -- with an abuse of notations -- in the following
form:
\[
\partial_{t}R^{B}(G,t,\cdot,\cdot)=\frak{A}_{t}^{B}(R^{B}(G,t))+H(G,t),
\]
for $\frak{A}_{t}^{B}$ a linear operator. (One may easily recognize
that the first term corresponds to $\Delta F_{i}^{\left(1\right)}$
and the second one corresponds to $\Delta F_{i}^{\left(2\right)}$.)

By Lemma \ref{lem:MF_a_priori}, we have the bound: 
\begin{align*}
\left|\frak{A}_{t}^{B}(R^{B}(G,t))(c_{i-1},c_{i})\right| & \le K_{T}|R^{B}|_{t}|w_{i+1}(c_{i})|_{t}+K_{T}|R^{B}|_{t}|w_{i}(c_{i})|_{t}|w_{i+1}(c_{i})|_{t}\\
 & \qquad+K_{T}|R_{i}^{B}(c_{i})|_{t}|w_{i+1}(c_{i})|_{t}+K_{T}|R_{i+1}^{B}(c_{i})|_{t}+K_{T}|R_{i-1}^{B}(c_{i-1})|_{t}|w_{i}(c_{i})|_{t}.
\end{align*}
Note that by Lemma \ref{lem:MF_a_priori},
\[
\mathbb{E}_{C_{i}}\left[|w_{i}(C_{i})|_{t}^{2}|w_{i+1}(C_{i})|_{t}^{2}\right]\le K_{T}.
\]
Thus, 
\[
\mathbb{E}_{C}\left[\left|\frak{A}_{t}^{B}(R^{B}(G,t))(C_{i-1},C_{i})\right|^{2}\right]\le K_{T}B^{2}|R^{B}|_{t}^{2}.
\]
Existence and uniqueness of $R^{B}$ follows immediately from boundedness
of $\frak{A}^{B}$. Also $R^{B}$ is a bounded linear functional in
$G$.

\paragraph*{Step 2: bounds for $R^{B}$.}

We obtain a bound on $|R^{B}|_{t}$. Observe that
\[
\left|H\left(G,t\right)\left(c_{i-1},c_{i}\right)\right|\leq K_{T}|w_{i+1}\left(c_{i}\right)|_{t}\left|G^{y}\right|_{t}+K\mathbb{E}_{Z}\left[\left|G_{i}^{w}\left(t,c_{i-1},c_{i},X\right)\right|\right]
\]
and as such,
\begin{align*}
\partial_{t}(|R_{i}^{B}|_{t}^{2}) & =\partial_{t}\mathbb{E}_{C}\left[\left|R_{i}^{B}(G,t,C_{i-1},C_{i})\right|^{2}\right]\\
 & \leq2|R_{i}^{B}|_{t}\mathbb{E}_{C}\left[\left|\partial R_{i}^{B}(G,t,C_{i-1},C_{i})\right|^{2}\right]^{1/2}\\
 & \le K_{T}|R^{B}|_{t}^{2}+K_{T}B|R_{i}^{B}|_{t}^{2}+K_{T}|G|_{t}^{2}.
\end{align*}
This implies 
\[
\partial_{t}(|R^{B}|_{t}^{2})\le K_{T}B|R^{B}|_{t}^{2}+K_{T}|G|_{t}^{2},
\]
so by Gronwall's lemma, for any $t\leq T$,
\[
|R^{B}|_{t}\le\exp(K_{T}B)\sup_{s\leq T}|G|_{s}=\exp(K_{T}B)\left\Vert G\right\Vert _{T,2}.
\]

Now let us zoom into $R_{i+1}^{B}(c_{i})$ and $R_{i}^{B}(c_{i})$.
From the bounds on $\left|\frak{A}_{t}^{B}(R^{B}(G,t))(c_{i},c_{i+1})\right|$
and $\left|H\left(G,t\right)\left(c_{i},c_{i+1}\right)\right|$, we
deduce:
\begin{align*}
\partial_{t}(|R_{i+1}^{B}(c_{i})|_{t}^{2}) & =2\mathbb{E}_{C_{i+1}}\left[R_{i+1}^{B}(G,t,c_{i},C_{i+1})\partial_{t}R_{i+1}^{B}(G,t,c_{i},C_{i+1})\right]\\
 & \le K_{T}|R_{i+1}^{B}(c_{i})|_{t}\left(|R^{B}|_{t}+|R_{i}^{B}(c_{i})|_{t}+\left|G^{y}\right|_{t}+|G_{i+1}^{w}(c_{i})|_{t}\right),
\end{align*}
which yields
\[
\partial_{t}|R_{i+1}^{B}(c_{i})|_{t}\le K_{T}\left(\exp(K_{T}B)\left\Vert G\right\Vert _{T,2}+|R_{i}^{B}(c_{i})|_{t}+\left|G^{y}\right|_{t}+|G_{i+1}^{w}(c_{i})|_{t}\right).
\]
Similarly:
\begin{align*}
\partial_{t}(|R_{i}^{B}(c_{i})|_{t}^{2}) & =2\mathbb{E}_{C_{i-1}}\left[R_{i}^{B}(G,t,C_{i-1},c_{i})\partial_{t}R_{i}^{B}(G,t,C_{i-1},c_{i})\right]\\
 & \le K_{T}|R_{i}^{B}(c_{i})|_{t}\Big((1+|w_{i}(c_{i})|_{t})(1+|w_{i+1}(c_{i})|_{t})|R^{B}|_{t}+|w_{i+1}(c_{i})|_{t}|R_{i}^{B}(c_{i})|_{t}\\
 & \qquad+|R_{i+1}^{B}(c_{i})|_{t}+|G_{i}^{w}(c_{i})|_{t}+|w_{i+1}(c_{i})|_{t}|G^{y}|_{t}\Big)\\
 & \le K_{T}|R_{i}^{B}(c_{i})|_{t}\left(B^{2}|R^{B}|_{t}+B|R_{i}^{B}(c_{i})|_{t}+|R_{i+1}^{B}(c_{i})|_{t}+|G_{i}^{w}(c_{i})|_{t}+|w_{i+1}(c_{i})|_{t}|G^{y}|_{t}\right),
\end{align*}
which yields
\begin{align*}
\partial_{t}|R_{i}^{B}(c_{i})|_{t} & \leq K_{T}\left(\exp(K_{T}B)\left\Vert G\right\Vert _{T,2}+B|R_{i}^{B}(c_{i})|_{t}+|R_{i+1}^{B}(c_{i})|_{t}+|G_{i}^{w}(c_{i})|_{t}+B|G^{y}|_{t}\right).
\end{align*}
Therefore,
\begin{align*}
\partial_{t}\left(|R_{i}^{B}(c_{i})|_{t}+|R_{i+1}^{B}(c_{i})|_{t}\right) & \leq K_{T}\left(\exp(K_{T}B)\left\Vert G\right\Vert _{T,2}+B\left(|R_{i}^{B}(c_{i})|_{t}+|R_{i+1}^{B}(c_{i})|_{t}\right)+|G_{i}^{w}(c_{i})|_{t}+|G_{i+1}^{w}(c_{i})|_{t}+B|G^{y}|_{t}\right).
\end{align*}
By Gronwall's lemma,
\[
|R_{i}^{B}(c_{i})|_{t}+|R_{i+1}^{B}(c_{i})|_{t}\leq\left(\left\Vert G\right\Vert _{T,2}+\int_{0}^{t}\left(|G_{i}^{w}(c_{i})|_{s}+|G_{i+1}^{w}(c_{i})|_{s}+|G^{y}|_{s}\right)ds\right)\exp\left(K_{T}B\right).
\]

\paragraph*{Step 3: taking $B\to\infty$.}

Next we compare $R^{B}$ and $R^{B'}$ for $B'>B$. Let us define
\[
\Delta R_{i}(G,t,c_{i-1},c_{i})=R_{i}^{B'}(G,t,c_{i-1},c_{i})-R_{i}^{B}(G,t,c_{i-1},c_{i}).
\]
We have that if $\mathbb{B}^{B}(t,c_{i-1},c_{i})=1$, then 
\[
\partial_{t}\Delta R_{i}(G,t,c_{i-1},c_{i})=-\Delta F_{i}^{\left(1\right)}(\Delta R)(G,t,c_{i-1},c_{i}).
\]
If $\mathbb{B}^{B'}(t,c_{i-1},c_{i})=1$ but $\mathbb{B}^{B}(t,c_{i-1},c_{i})=0$,
we have 
\[
\partial_{t}\Delta R_{i}(G,t,c_{i-1},c_{i})=-\Delta F_{i}(R^{B'})(G,t,c_{i-1},c_{i}).
\]
Finally, if $\mathbb{B}^{B'}(t,c_{i-1},c_{i})=0$, then 
\begin{align*}
\partial_{t}\Delta R_{i}(G,t,c_{i-1},c_{i}) & =0.
\end{align*}
With $\mathbb{D}_{B,B'}(t,c_{i-1},c_{i})=\mathbb{I}(\mathbb{B}^{B'}(t,c_{i-1},c_{i})=1,\;\mathbb{B}^{B}(t,c_{i-1},c_{i})=0)$,
we obtain:
\begin{align*}
 & \partial_{t}\mathbb{E}_{C}\left[\left|\Delta R_{i}(G,t,C_{i-1},C_{i})\right|^{2}\mathbb{D}_{B,B'}(t,C_{i-1},C_{i})\right]\\
 & \le K\mathbb{E}_{C}\left[\Delta R_{i}(G,t,C_{i-1},C_{i})\mathbb{D}_{B,B'}(t,C_{i-1},C_{i})\cdot\partial_{t}\Delta R_{i}(G,t,C_{i-1},C_{i})\mathbb{D}_{B,B'}(t,C_{i-1},C_{i})\right]\\
 & \le K\mathbb{E}_{C}\left[\left|\Delta R_{i}(G,t,C_{i-1},C_{i})\right|^{2}\mathbb{D}_{B,B'}(t,C_{i-1},C_{i})\right]^{1/2}\mathbb{E}_{C}\left[\left|\partial_{t}\Delta R_{i}(G,t,C_{i-1},C_{i})\right|^{2}\mathbb{D}_{B,B'}(t,C_{i-1},C_{i})\right]^{1/2}\\
 & \le K_{T}\mathbb{E}_{C}\left[\left|\Delta R_{i}(G,t,C_{i-1},C_{i})\right|^{2}\mathbb{D}_{B,B'}(t,C_{i-1},C_{i})\right]^{1/2}\cdot\mathbb{E}_{C}\bigg[\mathbb{D}_{B,B'}(t,C_{i-1},C_{i})\Big(|R^{B'}|_{t}^{2}|w_{i+1}(C_{i})|_{t}^{2}+|R^{B'}|_{t}^{2}|w_{i}(C_{i})|_{t}^{2}|w_{i+1}(C_{i})|_{t}^{2}\\
 & \qquad+|R_{i}^{B'}(C_{i})|_{t}^{2}|w_{i+1}(C_{i})|_{t}^{2}+|R_{i+1}^{B'}(C_{i})|_{t}^{2}+|R_{i}^{B'}(C_{i-1})|_{t}^{2}|w_{i}(C_{i})|_{t}^{2}+|w_{i+1}(C_{i})|_{t}^{2}|G^{y}|_{t}^{2}+|G_{i}^{w}(C_{i-1},C_{i}|_{t}^{2}\Big)\bigg]^{1/2}\\
 & \le K_{T}\mathbb{E}_{C}\left[\left|\Delta R_{i}(G,t,C_{i-1},C_{i})\right|^{2}\mathbb{D}_{B,B'}(t,C_{i-1},C_{i})\right]^{1/2}\cdot\mathbb{E}_{C}\bigg[\mathbb{D}_{B,B'}(t,C_{i-1},C_{i})\bigg((B'^{2}+B'^{4})|R^{B'}|_{t}^{2}\\
 & \qquad+B'^{2}|R_{i}^{B'}(C_{i})|_{t}^{2}+|R_{i+1}^{B'}(C_{i})|_{t}^{2}+B'^{2}|R_{i}^{B'}(C_{i-1})|_{t}^{2}+B'^{2}|G^{y}|_{t}^{2}+|G_{i}^{w}(C_{i-1},C_{i}|_{t}^{2}\bigg)\bigg]^{1/2}.
\end{align*}
Using Lemma \ref{lem:MF_a_priori} and the bounds in Step 2, we get:
\[
\mathbb{E}_{C}\left[\mathbb{D}_{B,B'}(t,C_{i-1},C_{i})\left(|R^{B'}|_{t}^{2}(B'^{2}+B'^{4})\right)\right]\le\left\Vert G\right\Vert _{T,2}^{2}\exp\left(K_{T}B'-K_{T}B^{2}\right).
\]
Similarly,
\begin{align*}
 & \mathbb{E}_{C}\left[\mathbb{D}_{B,B'}(t,C_{i-1},C_{i})\left(B'^{2}|R_{i}^{B'}(C_{i})|_{t}^{2}+\|R_{i+1}^{B'}(C_{i})\|^{2}+B'^{2}|R_{i}^{B'}(C_{i-1})|_{t}^{2}\right)\right]\\
 & \le\exp\left(K_{T}B'\right)\mathbb{E}_{C}\left[\mathbb{D}_{B,B'}(t,C_{i-1},C_{i})\left(\left\Vert G\right\Vert _{T,2}^{2}+\int_{0}^{t}\left(|G_{i}^{w}(C_{i})|_{s}^{2}+|G_{i+1}^{w}(C_{i})|_{s}^{2}+|G_{i-1}^{w}(C_{i-1})|_{s}^{2}+|G_{i}^{w}(C_{i-1})|_{s}^{2}+|G^{y}|_{s}^{2}\right)ds\right)\right]\\
 & \leq\exp\left(K_{T}B'-K_{T}\frac{\epsilon B^{2}}{2+\epsilon}\right)\left\Vert G\right\Vert _{T,2+\epsilon}^{2},\\
 & \mathbb{E}_{C}\left[\mathbb{D}_{B,B'}(t,C_{i-1},C_{i})\left(B'^{2}|G^{y}|_{t}^{2}+|G_{i}^{w}(C_{i-1},C_{i}|_{t}^{2}\right)\right]\\
 & \leq\exp\left(K_{T}B'-K_{T}\frac{\epsilon B^{2}}{2+\epsilon}\right)\left\Vert G\right\Vert _{T,2+\epsilon}^{2}.
\end{align*}
We thus get:
\[
\partial_{t}\mathbb{E}_{C}\left[\left|\Delta R_{i}(G,t,C_{i-1},C_{i})\right|^{2}\mathbb{D}_{B,B'}(t,C_{i-1},C_{i})\right]\le\left\Vert G\right\Vert _{T,2+\epsilon}^{2}\exp\left(K_{T}B'-K_{T}\frac{\epsilon B^{2}}{2+\epsilon}\right).
\]
The term where $\mathbb{B}^{B}(t,c_{i-1},c_{i})=1$ can be bounded
as in the bound for $R^{B}$:
\[
\partial_{t}\mathbb{E}_{C}\left[\left|\Delta R_{i}(G,t,C_{i-1},C_{i})\right|^{2}\mathbb{B}^{B}(t,C_{i-1},C_{i})\right]\le K_{T}B|\Delta R|_{t}^{2}.
\]
The last two displays give:
\[
\partial_{t}|\Delta R|_{t}^{2}\le K_{T}B|\Delta R|_{t}^{2}+\left\Vert G\right\Vert _{T,2+\epsilon}^{2}\exp\left(K_{T}B'-K_{T}\frac{\epsilon B^{2}}{2+\epsilon}\right).
\]
Hence, 
\[
|\Delta R|_{t}^{2}\le\left\Vert G\right\Vert _{T,2+\epsilon}^{2}\exp\left(K_{T}BB'-K_{T}\frac{\epsilon B^{3}}{2+\epsilon}\right).
\]
In particular, for all $B$ sufficiently large, we have $|\Delta R|_{t}^{2}\le\left\Vert G\right\Vert _{T,2+\epsilon}^{2}\exp(-K_{T}\epsilon B^{2}/(2+\epsilon))$
for all $B'\le2B$. Thus, we can easily deduce that $R^{B}$ converges
in $L^{2}$ to a limit, which is the process $R$, as $B\to\infty$.
Since $|\Delta R|_{t}^{2}$ decays exponentially with $B^{2}$ while
$|R^{B}|_{t}\le\exp(K_{T}B)\left\Vert G\right\Vert _{T,2}$, we deduce
that:
\[
\sup_{t\leq T}|R|_{t}^{2}\le K_{T,\epsilon}\|G\|_{T,2+\epsilon}^{2}.
\]
We can also deduce for fixed $c_{i-1}$, $c_{i}$, $R^{B}(G,t,c_{i-1},c_{i})$
converges, as $B\to\infty$, to $R(G,t,c_{i-1},c_{i})$, and $R$
satisfies Eq. (\ref{eq:ODE_2ndMF-alt}).

\paragraph*{Step 4: uniqueness.}

Next, we show uniqueness of $R$. Assume that $R,R'$ are two solutions
to the equation in $L^{2}$, we have $U=R-R'$ satisfies 
\[
\partial_{t}U_{i}(G,t,c_{i-1},c_{i})=-\Delta F_{i}^{\left(1\right)}(U)(G,t,c_{i-1},c_{i})
\]
We then have 
\begin{align*}
\left|\partial_{t}U_{i}(G,t,c_{i-1},c_{i})\right| & \le K|U|_{t}|w_{i+1}(c_{i})|_{t}+K|U|_{t}|w_{i}(c_{i})|_{t}w_{i+1}(c_{i})|_{t}\\
 & \qquad+K|U_{i}(c_{i})|_{t}|w_{i+1}(c_{i})|_{t}+K|U_{i+1}(c_{i})|_{t}+K|U_{i-1}(c_{i-1})|_{t}|w_{i}(c_{i})|_{t}.
\end{align*}
Let $\kappa=\max_{t\le T}(|R|_{t},|R'|_{t})<\infty$. Then 
\[
\partial_{t}|U_{i}(c_{i-1})|_{t}^{2}\le K|U_{i}(c_{i-1})|_{t}\left(\kappa+|U_{i-1}(c_{i-1})|_{t}+\mathbb{E}_{C}[|U_{i}(C_{i})|_{t}^{2}|w_{i+1}(C_{i})|_{t}^{2}]^{1/2}\right),
\]
\[
\partial_{t}|U_{i}(c_{i})|_{t}^{2}\le K|U_{i}(c_{i})|_{t}\left(\kappa(1+\max(|w_{i+1}(c_{i})|_{t},|w_{i}(c_{i})|_{t}))^{2}+|w_{i+1}(c_{i})|_{t}|U_{i}(c_{i})|_{t}+|U_{i+1}(c_{i})|_{t}\right).
\]
In particular 
\[
\partial_{t}|U_{L}(c_{L})|_{t}^{2}\le K|U_{L}(c_{L})|_{t}\left(\kappa(1+|w_{L}(c_{L})|_{t})+|U_{L}(c_{L})|_{t}\right),
\]
from which we obtain $|U_{L}(c_{L})|_{t}\le K_{\kappa,T}(1+\sup_{s\le T}|w_{L}(c_{L})|_{s})$
for all $t\le T$ and $c_{L}$. Note that $\mathbb{E}_{C_{L}}[K_{\kappa,T}(1+\sup_{s\le T}|w_{L}(C_{L})|_{s})^{2}]<\infty$.
Assume the bound on $|U_{i+1}(c_{i+1})|_{t}$ such that $\mathbb{E}_{C_{i+1}}\left[|U_{i+1}(C_{i+1})|_{t}^{2}|w_{i+2}(C_{i+1})|_{t}^{2}\right]<\infty$,
we have 
\[
\partial_{t}|U_{i}(c_{i})|_{t}^{2}\le K|U_{i}(c_{i})|_{t}\left(\kappa(1+\max(|w_{i+1}(c_{i})|_{t},|w_{i}(c_{i})|_{t}))^{2}+|w_{i+1}(c_{i})|_{t}|U_{i}(c_{i})|_{t}+|U_{i+1}(c_{i})|_{t}\right),
\]
\begin{align*}
\partial_{t}|U_{i+1}(c_{i})|_{t}^{2} & \le K|U_{i+1}(c_{i})|_{t}\left(\kappa+|U_{i}(c_{i})|_{t}+\mathbb{E}[|U_{i+1}(C_{i+1})|_{t}^{2}|w_{i+2}(C_{i+1})|_{t}^{2}]^{1/2}\right)\\
 & \le K|U_{i+1}(c_{i})|_{t}\left(\kappa+|U_{i}(c_{i})|_{t}+K_{\kappa,T}\right).
\end{align*}
Thus, $|U_{i+1}(c_{i})|_{t}\le K_{\kappa,T}(\int_{0}^{t}|U_{i}(c_{i})|_{s}ds+1)$.
Hence, 
\[
\partial_{t}|U_{i}(c_{i})|_{t}\le\kappa\left(1+\max(|w_{i+1}(c_{i})|_{t},|w_{i}(c_{i})|_{t}))^{2}+|w_{i+1}(c_{i})|_{t}|U_{i}(c_{i})|_{t}+\int_{0}^{t}|U_{i}(c_{i})|_{s}ds\right).
\]
From this we obtain
\[
|U_{i}(c_{i})|_{t}\le K_{\kappa,T}\exp(K_{\kappa,T}|w_{i+1}(c_{i})|_{t})(1+\max(|w_{i+1}(c_{i})|_{t},|w_{i}(c_{i})|_{t}))^{2}).
\]
Note that we still have 
\begin{align*}
\mathbb{E}_{C_{i}}\left[|U_{i}(C_{i})|_{t}^{2}|w_{i}(C_{i})|_{t}^{2}\right] & \le K_{\kappa,T}\mathbb{E}_{C_{i}}\left[\exp(K_{\kappa,T}|w_{i+1}(C_{i})|_{t})(1+\max(|w_{i+1}(c_{i})|_{t},|w_{i}(c_{i})|_{t}))^{2})^{4}\right]\\
 & <\infty,
\end{align*}
by Lemma \ref{lem:MF_a_priori}. By induction, we obtain bounds 
\[
\max_{t\le T}\max(|U_{i+1}(c_{i})|_{t},|U_{i}(c_{i})|_{t})\le\exp\left(K_{\kappa,T}(1+\max(|w_{i+1}(c_{i})|_{t},|w_{i}(c_{i})|_{t}))\right),
\]
for all $i$ and $c_{i}$. From here we can easily obtain a similar
bound on $U_{i}(t,c_{i-1},c_{i})$, 
\[
\max_{t\le T}|U_{i}(t,c_{i-1},c_{i})|\le\exp\left(K_{\kappa,T}(1+\max(|w_{i}(c_{i-1})|_{t},|w_{i-1}(c_{i-1})|_{t},|w_{i+1}(c_{i})|_{t},|w_{i}(c_{i})|_{t}))\right).
\]
This also implies all moments of $U$ are finite again by Lemma \ref{lem:MF_a_priori}.

Using those bounds, we have, for any $B>0$, that 
\begin{align*}
\partial_{t}|U|_{t}^{2} & \le K|U|_{t}^{2}+K|U|_{t}\sum_{i}\left(\mathbb{E}_{C_{i}}\left[|U_{i}(C_{i})|_{t}^{2}|w_{i+1}(C_{i})|_{t}^{2}\right]\right)^{1/2}\\
 & \le K(1+B)|U|_{t}^{2}+K|U|_{t}\sum_{i}\left(\mathbb{E}_{C_{i}}\left[|U_{i}(C_{i})|_{t}^{2}|w_{i+1}(C_{i})|_{t}^{2}\mathbb{I}(|w_{i+1}(C_{i})|_{t}\ge B)\right]\right)^{1/2}.\\
 & \le K(1+B)|U|_{t}^{2}+K\sum_{i}\mathbb{E}_{C_{i}}\left[|U_{i}(C_{i})|_{t}^{2}|w_{i+1}(C_{i})|_{t}^{2}\mathbb{I}(|w_{i+1}(C_{i})|_{t}\ge B)\right].
\end{align*}
Our bound on $|U_{i}(C_{i})|_{t}$ together with Lemma \ref{lem:MF_a_priori}
gives 
\[
\mathbb{E}_{C_{i}}\left[|U_{i}(C_{i})|_{t}^{2}|w_{i+1}(C_{i})|_{t}^{2}\mathbb{I}(|w_{i+1}(C_{i})|_{t}\ge B)\right]\le\exp(K_{\kappa,T}B-cB^{2}).
\]
This implies 
\[
\partial_{t}|U|_{t}^{2}\le K(1+B)|U|_{t}^{2}+\exp(K_{\kappa,T}B-cB^{2}).
\]
Thus, 
\[
\sup_{t\le T}|U|_{t}^{2}\le\exp(K(1+B)T)\exp(K_{\kappa,T}B-cB^{2}).
\]
Sending $B$ to infinity, we immediately obtain that $|U|_{t}=0$
for all $t\le T$. Thus, the solution $R$ is unique.
\end{proof}
Repeating the proof of Theorem \ref{thm:exist-R} identically gives
the following.
\begin{lem}
\label{lem:lp-R}If $\|G\|_{T,2p+\epsilon}<\infty$, then $R_{i}(G,t,\cdot,\cdot)\in L^{2p}(\Omega_{i-1}\times\Omega_{i})$,
and 
\[
\|R(G)\|_{T,2p}^{2p}\le K_{T,\epsilon}\|G\|_{T,2p+\epsilon}^{2p}.
\]
\end{lem}

\section{Connecting finite-width neural network fluctuations with the limit
system: Proof of Theorem \ref{thm:2nd_order_MF}\label{Appendix:prop_chaos}}

Let us recall that
\[
{\bf R}_{i}(t,j_{i-1},j_{i})=\sqrt{N}({\bf w}_{i}(t,j_{i-1},j_{i})-w_{i}(t,C_{i-1}(j_{i-1}),C_{i}(j_{i})).
\]
Our goal is to prove Theorem \ref{thm:2nd_order_MF}. A major technical
difficulty here lies in establishing suitable a priori moment bounds
for $\mathbf{R}$ that are independent of $N$. This task suffers
from similar unboundedness issues that are encountered in the proof
of Theorem \ref{thm:exist-R}. Again this is a problem unique to the
multilayer structure. Yet it requires a delicate argument that is
different from the one in the proof of Theorem \ref{thm:exist-R}.
In particular, Lemma \ref{lem:trunc-bound-w} below plays a crucial
role and computes very high moments (of order increasing with $N$)
of the neural network's weights under GD evolution and randomized
initialization. It is also a result of independent interest, which
we expect to have applications beyond the current pursuit.

In the following, we denote 
\[
\|{\bf R}_{i}(t)-R_{i}(t)\|_{2p}^{2p}=\sum_{i=1}^{L}\mathbb{E}_{J}\left[\left|{\bf R}_{i}(t,J_{i-1},J_{i})-R_{i}(\tilde{G},t,C_{i-1}(J_{i-1}),C_{i}(J_{i}))\right|^{2p}\right].
\]
For brevity, let us write
\[
\tilde{w}_{i}(t,j_{i-1},j_{i})=w_{i}(t,C_{i-1}(j_{i-1}),C_{i}(j_{i})).
\]
Similar to the development in Appendix \ref{Appendix:Gaussian}, define
\begin{align*}
\|\tilde{w}_{i}(t,j_{i})\|_{2p}^{2p} & =\mathbb{E}_{J_{i-1}}\left[\left|w_{i}(t,C_{i-1}(J_{i-1}),C_{i}(j_{i}))\right|^{2p}\right],\\
\|{\bf w}_{i}(t,j_{i})\|_{2p}^{2p} & =\mathbb{E}_{J_{i-1}}\left[\left|{\bf w}_{i}(t,J_{i-1},j_{i})\right|^{2p}\right],\\
\|\tilde{w}_{i+1}(t,j_{i})\|_{2p}^{2p} & =\mathbb{E}_{J_{i+1}}\left[\left|w_{i}(t,C_{i}(j_{i}),C_{i+1}(J_{i+1}))\right|^{2p}\right],\\
\|{\bf w}_{i+1}(t,j_{i})\|_{2p}^{2p} & =\mathbb{E}_{J_{i+1}}\left[\left|{\bf w}_{i}(t,j_{i},J_{i+1})\right|^{2p}\right],\\
\|\tilde{w}_{i}(t)\|_{2p}^{2p} & =\mathbb{E}_{J_{i}}\left[\|\tilde{w}_{i}(t,J_{i})\|_{2p}^{2p}\right],\\
\|{\bf w}_{i}(t)\|_{2p}^{2p} & =\mathbb{E}_{J_{i}}\left[\|\mathbf{w}_{i}(t,J_{i})\|_{2p}^{2p}\right].
\end{align*}
We define similarly $\|{\bf R}_{i}(t,j_{i})\|_{2p}$, $\|R_{i}(t,j_{i})\|_{2p}$,
$\|{\bf R}_{i+1}(t,j_{i+1})\|_{2p}$, $\|R_{i+1}(t,j_{i+1})\|_{2p}$,
$\|{\bf R}_{i}(t)\|_{2p}$, $\|R_{i}(t)\|_{2p}$ (where quantities
involving $R$ should be computed w.r.t. $R(\tilde{G},t)$). We also
let
\[
\|{\bf R}(t)\|_{2p}^{2p}=\sum_{i=1}^{L}\|{\bf R}_{i}(t)\|_{2p}^{2p},\qquad\|R(t)\|_{2p}^{2p}=\sum_{i=1}^{L}\|R_{i}(t)\|_{2p}^{2p}.
\]
In all of these, when we drop the subscript $2p$, we implicitly take
$2p=2$. For $B>0$, define
\begin{align*}
\tilde{\mathbb{B}}_{i}^{B,2k}(t,j_{i}) & =\mathbb{I}(\|\tilde{w}_{i}(t,j_{i})\|_{2k},\|\tilde{w}_{i+1}(t,j_{i})\|_{2k}\le B),\\
{\bf B}_{i}^{B,2k}(t,j_{i}) & =\mathbb{I}(\|{\bf w}_{i}(t,j_{i})\|_{2k},\|{\bf w}_{i+1}(t,j_{i})\|_{2k}\le B),\\
\mathbb{B}^{B,2k}(t,j_{i}) & =\tilde{\mathbb{B}}_{i}^{B,2k}(t,j_{i}){\bf B}_{i}^{B,2k}(t,j_{i}),\\
\mathbb{B}^{B,2k}(t,j_{i-1},j_{i}) & =\mathbb{B}^{B,2k}(t,j_{i-1})\mathbb{B}^{B,2k}(t,j_{i}).
\end{align*}
We drop the superscripts $B$ and $2k$ if they are clear from context.
Note that these are random variables due to the randomness of sampling
$\{C_{i}(j_{i}):\;j_{i}\in[N_{i}],\;i\in[L]\}$, whose expectation
is denoted by $\mathbf{E}$.

\subsection{High moments of neural network's weights under GD}

We obtain the following estimate, which is important for our development.
\begin{lem}
\label{lem:trunc-bound-w}For $k\ge1$ and $B>0$, we have:
\[
{\bf E}\mathbb{E}_{J}\left[\left|w_{i}(t,C_{i-1}(J_{i-1}),C_{i}(J_{i}))\right|^{k}\right]=\mathbb{E}_{C}\left[\left|w_{i}(t,C_{i-1},C_{i})\right|^{k}\right],
\]
and for all $N\ge(kL)^{6}$, 
\[
{\bf E}\mathbb{E}_{J}\left[\left|{\bf w}_{i}(t,J_{i-1},J_{i})\right|^{k}\right]\le k^{k/2}K_{T}^{kL}.
\]
As an immediate corollary, for a fixed constant $C$ and $B\le N^{K_{C}}$,
\[
{\bf E}\left[{\bf B}_{i}^{B,2C}(t,j_{i})\right]\le\exp(-K_{T,C}B^{2}).
\]
\end{lem}

\begin{proof}
The first equality is trivial. For the second bound, we consider $i\geq2$;
the case $i=1$ can be done identically. We have 
\begin{align*}
\partial_{t}\mathbb{E}_{J}\left[\left|{\bf w}_{i}(t,J_{i-1},J_{i})\right|^{k}\right] & =k\mathbb{E}_{J}\left[{\bf w}_{i}(t,J_{i-1},J_{i})^{k-1}\partial_{t}{\bf w}_{i}(t,J_{i-1},J_{i})\right]\\
 & \le k\mathbb{E}_{J}\left[\left|{\bf w}_{i}(t,J_{i-1},J_{i})\right|^{k}\right]^{(k-1)/k}\mathbb{E}_{J}\left[\left|\partial_{t}{\bf w}_{i}(t,J_{i-1},J_{i})\right|^{k}\right]^{1/k}.
\end{align*}
Note that 
\begin{align*}
\mathbb{E}_{J}\left[\left|\partial_{t}{\bf w}_{i}(t,J_{i-1},J_{i})\right|^{k}\right] & =\mathbb{E}_{J}\left[\left|\mathbb{E}_{Z}\left[\partial_{2}{\cal L}(Y,{\bf y})\frac{\partial\hat{{\bf y}}(t,X)}{\partial{\bf w}_{i}(J_{i-1},J_{i})}\right]\right|^{k}\right]\\
 & \le K\mathbb{E}_{J}\left[\mathbb{E}_{Z}\left[\left|\frac{\partial\hat{{\bf y}}(t,X)}{\partial{\bf w}_{i}(J_{i-1},J_{i})}\right|^{2}\right]^{k/2}\right]\\
 & \le K^{k}\mathbb{E}_{J}\left[\left|{\bf w}_{i+1}(t,J_{i},J_{i+1})\right|{}^{k}\right]\prod_{j\ge i+2}\mathbb{E}_{J}\left[{\bf w}_{j}(t,J_{j-1},J_{j})^{2}\right]^{k/2}.
\end{align*}
Thus,
\begin{align*}
\partial_{t}\mathbb{E}_{J}\left[\left|{\bf w}_{i}(t,J_{i-1},J_{i})\right|{}^{k}\right] & \le Kk\mathbb{E}_{J}\left[\left|{\bf w}_{i}(t,J_{i-1},J_{i})\right|^{k}\right]^{(k-1)/k}\mathbb{E}_{J}\left[\left|{\bf w}_{i+1}(t,J_{i},J_{i+1})\right|^{k}\right]^{1/k}\prod_{j\ge i+2}\mathbb{E}_{J}\left[{\bf w}_{j}(t,J_{j-1},J_{j})^{2}\right]^{1/2}.
\end{align*}
In particular,
\[
\partial_{t}\mathbb{E}_{J}[\left|{\bf w}_{L}(t,J_{L-1},J_{L})\right|^{k}]^{1/k}\le K_{T},
\]
which implies
\[
\mathbb{E}_{J}[\left|{\bf w}_{L}(t,J_{L-1},J_{L})\right|^{k}]^{1/k}\le\mathbb{E}_{J}[\left|{\bf w}_{L}(0,J_{L-1},J_{L})\right|^{k}]^{1/k}+Kt.
\]
Then inductively, we obtain that 
\begin{align*}
 & \sup_{t\le T}\mathbb{E}_{J}[\left|{\bf w}_{i}(t,J_{i-1},J_{i})\right|{}^{k}]^{1/k}\\
 & \le\mathbb{E}_{J}[\left|{\bf w}_{i}(0,J_{i-1},J_{i})\right|{}^{k}]^{1/k}+p_{i}\Big(\left(\mathbb{E}_{J}\left[{\bf w}_{j}(0,J_{j-1},J_{j})^{2}\right]^{1/2}\right)_{j\ge i+1},\left(\mathbb{E}_{J}\left[|{\bf w}_{j}(0,J_{j-1},J_{j})|^{k}\right]^{1/k}\right)_{j\ge i+1},T\Big),
\end{align*}
where $p_{i}$ is a polynomial of degree at most $L-i$, and furthermore,
in each monomial in $p_{i}$ there is at most one term of degree one
from the variables $\left(\mathbb{E}_{J}\left[|{\bf w}_{j}(0,J_{j-1},J_{j})|^{k}\right]^{1/k}\right)_{j\ge i+1}$.
This allows us to bound 
\begin{align*}
{\bf E}\mathbb{E}_{J}[|{\bf w}_{i}(t,J_{i-1},J_{i})|^{k}] & \le K^{k}{\bf E}\mathbb{E}_{J}[|{\bf w}_{i}(0,J_{i-1},J_{i})|^{k}]\\
 & \quad+K_{T}^{k}\sum_{j\ge i+1}{\bf E}\Big[\mathbb{E}_{J}[|{\bf w}_{i}(0,J_{i-1},J_{i})|^{k}]\sum_{j'\ge i+1}\mathbb{E}_{J}\left[|{\bf w}_{j'}(0,J_{j'-1},J_{j'})|^{2}\right]^{kL/2}\Big].
\end{align*}
We have
\begin{align*}
 & {\bf E}\left[\mathbb{E}_{J}[|{\bf w}_{i}(0,J_{i-1},J_{i})|^{k}]\mathbb{E}_{J}\left[|{\bf w}_{j'}(0,J_{j'-1},J_{j'})|^{2}\right]^{kL/2}\right]\\
 & \le\left({\bf E}\left[\mathbb{E}_{J}[|{\bf w}_{i}(0,J_{i-1},J_{i})|^{k}]^{2}\right]\right)^{1/2}\left({\bf E}\left[\mathbb{E}_{J}\left[{\bf w}_{j'}(0,J_{j'-1},J_{j'})^{2}\right]^{kL}\right]\right)^{1/2}\\
 & \le\left({\bf E}\left[\mathbb{E}_{J}[{\bf w}_{i}(0,J_{i-1},J_{i})^{2k}]\right]\right)^{1/2}\left({\bf E}\left[\mathbb{E}_{J}\left[{\bf w}_{j'}(0,J_{j'-1},J_{j'})^{2}\right]^{kL}\right]\right)^{1/2}\\
 & \le K^{k}k^{k/2}\left({\bf E}\left[\mathbb{E}_{J}\left[{\bf w}_{j'}(0,J_{j'-1},J_{j'})^{2}\right]^{kL}\right]\right)^{1/2}.
\end{align*}
Let us analyze the term in the last display:
\[
\mathbb{E}_{J'}\left[{\bf w}_{j'}(0,J_{j'-1},J_{j'})^{2}\right]^{kL}=N^{-2kL}\sum_{\alpha_{1},\dots,\alpha_{kL}\in[N_{j'-1}],\beta_{1},\dots,\beta_{kL}\in[N_{j'}]}\prod_{t=1}^{kL}{\bf w}_{j'}(0,\alpha_{t},\beta_{t})^{2}.
\]
Furthermore, if each index appears at most once among $\alpha,\beta$,
then 
\begin{align*}
{\bf E}\left[\prod_{t=1}^{kL}{\bf w}_{j'}(0,\alpha_{t},\beta_{t})^{2}\right] & =\prod_{t=1}^{kL}\mathbb{E}_{C}\left[w_{j'}(0,C_{j'-1},C_{j'})^{2}\right]^{kL}\\
 & \le K^{kL}.
\end{align*}
On the other hand, if some index appears at least twice, let $d(a)$
be the number of times $a$ appears among $\{\alpha_{s}\}$, and we
similarly define $d(b)$. We have by Finner's inequality\footnote{Helmut Finner, \textquotedbl A generalization of Holder's inequality
and some probability inequalities\textquotedbl , The Annals of probability
(1992), pp. 1893-{}-1901.} that
\begin{align*}
{\bf E}\left[\prod_{t=1}^{kL}{\bf w}_{j'}(0,\alpha_{t},\beta_{t})^{2}\right] & \le\prod_{t=1}^{kL}{\bf E}\left[{\bf w}_{j'}(0,\alpha_{t},\beta_{t})^{2\max(d(\alpha_{t}),d(\beta_{t}))}\right]^{1/\max(d(\alpha_{t}),d(\beta_{t}))}\\
 & \le K^{k(1+L)}\prod_{t=1}^{kL}\max(d(\alpha_{t}),d(\beta_{t})).
\end{align*}
Note that $\sum_{a}d(a)=\sum_{b}d(b)=kL$. Fixing the multisets $\{d(a)\}=\{d(a):\;a\in[N_{j'-1}]\},\;\{d(b)\}=\{d(b):\;b\in[N_{j'}]\}$,
let $M_{a}$ be the number of positive values in $\{d(a)\}$ and $M_{b}$
the number of positive values of $\{d(b)\}$. Let $H_{a}$ be the
number of values at least $2$ in $\{d(a)\}$ and similarly for $H_{b}$.
The number of unlabeled bipartite graphs for which the degrees on
one part are given by $\{d(a)\}$ and the degrees of the other part
are given by $\{d(b)\}$ is at most $(H_{b}+1)^{H_{a}}(H_{a}+1)^{H_{b}}$
(each half-edge from a vertex of degree at least $2$ on the $a$-side
is chosen to match to a half-edge of a vertex of degree at least $2$
on the $b$-side or a vertex of degree $1$, for a total of $H_{b}+1$
choices). After choosing the unlabeled bipartite graph, there are
at most $N^{M_{a}+M_{b}}$ ways to label the vertices of the graph
with indices in $[N_{j'-1}]$ or $[N_{j'}]$. Let $Q$ be the number
of edges between vertices both having degree $1$. We have at most
$(kL)!/Q!$ ways to assign index $t\in[kL]$ to edges of the labeled
bipartite graph. We have that $Q\ge kL-\sum_{a}d(a)\mathbb{I}(d(a)\ge2)-\sum_{b}d(b)\mathbb{I}(d(b)\ge2)$,
while 
\[
M_{a}=\sum_{a}\mathbb{I}(d(a)\ge1)\le\sum_{a}\mathbb{I}(d(a)=1)+\frac{1}{2}\sum_{a}d(a)\mathbb{I}(d(a)\ge2)=kL-\frac{1}{2}\sum_{a}d(a)\mathbb{I}(d(a)\ge2),
\]
and similarly $M_{b}\le kL-\frac{1}{2}\sum_{b}d(b)\mathbb{I}(d(b)\ge2)$,
so $M_{a}+M_{b}\le2kL-\frac{1}{2}\left(\sum_{a}d(a)\mathbb{I}(d(a)\ge2)+\sum_{b}d(b)\mathbb{I}(d(b)\ge2)\right)$.
Thus, the number of edges between vertices both having degree $1$
is at least $2(M_{a}+M_{b})-3kL$. We also have $H_{a}\le\frac{1}{2}\sum_{a}d(a)\mathbb{I}(d(a)\ge2)\le kL-M_{a}$
and $H_{b}\le kL-M_{b}$. Finally, note that 
\[
\prod_{t=1}^{kL}\max(d(\alpha_{t}),d(\beta_{t}))\le\left(\frac{kL}{\min(1,H_{a})}\right)^{H_{a}}\left(\frac{kL}{\min(1,H_{b})}\right)^{H_{b}}.
\]
Thus, using convention that $x!=1$ if $x\le0$ and using that the
number of choices of $\{d(a)\}$ given $M_{a}=m_{a}$ is at most $\binom{kL-1}{m_{a}-1}$,
we get 
\begin{align*}
\sum_{\alpha,\beta}\prod_{t=1}^{kL}\max(d(\alpha_{t}),d(\beta_{t})) & \le\sum_{\Gamma}\frac{(kL)!}{(2m_{a}+2m_{b}-3kL)!}N^{m_{a}+m_{b}}(h_{b}+1)^{h_{a}}(h_{a}+1)^{h_{b}}\left(\frac{kL}{h_{a}}\right)^{h_{a}}\left(\frac{kL}{h_{b}}\right)^{h_{b}}\binom{kL-1}{m_{a}-1}\binom{kL-1}{m_{b}-1}\\
 & \le\sum_{\Gamma}\frac{(kL)!}{(2m_{a}+2m_{b}-3kL)!}N^{m_{a}+m_{b}}\left(\frac{2(kL)^{2}}{\min(1,h_{a})}\right)^{h_{a}}\left(\frac{2(kL)^{2}}{\min(1,h_{b})}\right)^{h_{b}}\binom{kL-1}{m_{a}-1}\binom{kL-1}{m_{b}-1}\\
 & \le K^{kL}\sum_{u=0}^{2kL}N^{2kL-u}\left(\frac{4(kL)^{4}}{\min(1,u)}\right)^{u}\\
 & \le K^{kL}N^{2kL}\sum_{u=0}^{2kL}\left(\frac{4(kL)^{4}}{N\min(1,u)}\right)^{u},
\end{align*}
where we use the shorthands $\Gamma=\left\{ m_{a},m_{b}\in[kL],h_{a},h_{b}\le kL/2,h_{a}\le kL-m_{a},h_{b}\le kL-m_{b}\right\} $.
Hence, for $N>(kL)^{6}$, we have 
\begin{align*}
\sum_{\alpha,\beta}\prod_{t=1}^{kL}\max(d(\alpha_{t}),d(\beta_{t})) & \le K^{kL}N^{2kL}.
\end{align*}
In particular, 
\[
{\bf E}\mathbb{E}_{J}\left[{\bf w}_{j'}(0,J_{j'-1},J_{j'})^{2}\right]^{kL}\le K^{kL}.
\]
Therefore,
\begin{align*}
{\bf E}\mathbb{E}_{J}[{\bf w}_{i}(t,J_{i-1},J_{i})^{k}] & \le k^{k/2}K_{T}^{kL}.
\end{align*}
\end{proof}
\begin{rem}
\label{rem:trunc-bound-w-remark}The same argument shows that for
any $\ell\geq1$ and $N\geq(K_{\ell}kL)^{16}$,
\[
{\bf E}\mathbb{E}_{J}\left[{\bf w}_{j}(0,J_{j-1},J_{j})^{2\ell}\right]^{kL}\le K_{\ell}^{kL}.
\]
\end{rem}

\subsection{A priori estimates at the fluctuation level}
\begin{lem}
\label{lem:a-priori-tilde_minus_MF}We have for any $t\leq T$:
\[
{\bf E}\mathbb{E}_{J}\left[\left(\sqrt{N}(\tilde{H}_{i}(t,C_{i}(J_{i}),x)-H_{i}(t,C_{i}(J_{i}),x)\right)^{2k}\right]\le K_{T,k},
\]
\[
{\bf E}\mathbb{E}_{J}\left[\left(\sqrt{N}\left(\frac{\partial\tilde{y}(t,x)}{\partial\tilde{H}_{i}(C_{i}(J_{i}))}-\frac{\partial\hat{y}(t,x)}{\partial H_{i}(C_{i}(J_{i}))}\right)\right)^{2k}\right]\le K_{T,k},
\]
\[
{\bf E}\mathbb{E}_{J}\left[\left(\sqrt{N}\left(\frac{\partial\tilde{y}(t,x)}{\partial w_{i}(C_{i-1}(J_{i-1}),C_{i}(J_{i}))}-\frac{\partial\hat{y}(t,x)}{\partial w_{i}(C_{i-1}(J_{i-1}),C_{i}(J_{i}))}\right)\right)^{2k}\right]\le K_{T,k},
\]
\[
{\bf E}\left[\left(\sqrt{N}\left(\partial_{2}{\cal L}(y,\tilde{y}(t,x))-\partial_{2}{\cal L}(y,\hat{y}(t,x))\right)\right)^{2k}\right]\le K_{T,k},
\]
\end{lem}

\begin{proof}
Let us drop $t$ and $x$ from the notations for brevity. We also
recall Lemma \ref{lem:MF_a_priori}. We prove claim by claim:
\begin{itemize}
\item We have: 
\[
\mathbf{E}\mathbb{E}_{J}\left[\left(\sqrt{N}(\tilde{H}_{1}(C_{1}(J_{1}))-H_{1}(C_{1}(J_{1}))\right)^{2k}\right]=0
\]
and using induction,
\begin{align*}
 & \mathbf{E}\mathbb{E}_{J}\left[\left(\sqrt{N}(\tilde{H}_{i}(C_{i}(J_{i}))-H_{i}(C_{i}(J_{i}))\right)^{2k}\right]\\
 & =\mathbf{E}\mathbb{E}_{J_{i}}\bigg[\bigg(\sqrt{N}\mathbb{E}_{J_{i-1}}\left[w_{i}\left(C_{i-1}\left(J_{i-1}\right),C_{i}\left(J_{i}\right)\right)\varphi_{i-1}\left(\tilde{H}_{i-1}\left(C_{i-1}\left(J_{i-1}\right)\right)\right)\right]\\
 & \qquad\qquad-\sqrt{N}\mathbb{E}_{C_{i-1}}\left[w_{i}\left(C_{i-1},C_{i}\left(J_{i}\right)\right)\varphi_{i-1}\left(H_{i-1}\left(C_{i-1}\right)\right)\right]\bigg)^{2k}\bigg]\\
 & \leq K_{k}\mathbf{E}\mathbb{E}_{J_{i}}\left[\mathbb{E}_{C_{i-1}}\left[\left|w_{i}\left(C_{i-1},C_{i}\left(J_{i}\right)\right)\varphi_{i-1}\left(H_{i-1}\left(t,C_{i-1}\right)\right)\right|^{2}\right]^{k}\right]\\
 & \quad+K_{k}\mathbf{E}\mathbb{E}_{J_{i}}\left[\left(\sqrt{N}\mathbb{E}_{J_{i-1}}\left[w_{i}\left(C_{i-1}\left(J_{i-1}\right),C_{i}\left(J_{i}\right)\right)\left(\varphi_{i-1}\left(\tilde{H}_{i-1}\left(C_{i-1}\left(J_{i-1}\right)\right)\right)-\varphi_{i-1}\left(H_{i-1}\left(C_{i-1}\left(J_{i-1}\right)\right)\right)\right)\right]\right)^{2k}\right]\\
 & \leq K_{T,k}+K_{k}\mathbf{E}\mathbb{E}_{J_{i}}\left[\mathbb{E}_{J_{i-1}}\left[\left|w_{i}\left(C_{i-1}\left(J_{i-1}\right),C_{i}\left(J_{i}\right)\right)\right|^{2k}\right]\mathbb{E}_{J_{i-1}}\left[N\left(\tilde{H}_{i-1}\left(C_{i-1}\left(J_{i-1}\right)\right)-H_{i-1}\left(C_{i-1}\left(J_{i-1}\right)\right)\right)^{2k}\right]\right]\\
 & \leq K_{T,k}.
\end{align*}
\item Next we have:
\[
\left(\sqrt{N}\left(\frac{\partial\tilde{y}}{\partial\tilde{H}_{L}(1)}-\frac{\partial\hat{y}}{\partial H_{i}(1)}\right)\right)^{2k}\leq K_{k}\left(\sqrt{N}\left(\tilde{H}_{L}(1)-H_{i}(1)\right)\right)^{2k}\leq K_{T,k}.
\]
Using backward induction:
\begin{align*}
 & {\bf E}\mathbb{E}_{J}\left[\left(\sqrt{N}\left(\frac{\partial\tilde{y}}{\partial\tilde{H}_{i}(C_{i}(J_{i}))}-\frac{\partial\hat{y}}{\partial H_{i}(C_{i}(J_{i}))}\right)\right)^{2k}\right]\\
 & \leq K_{k}{\bf E}\mathbb{E}_{J_{i}}\left[\mathbb{E}_{C_{i+1}}\left[\left|\frac{\partial\hat{y}}{\partial H_{i+1}(C_{i+1})}w_{i+1}\left(C_{i}\left(J_{i}\right),C_{i+1}\right)\right|^{2}\right]^{k}\right]\\
 & \quad+K_{k}{\bf E}\mathbb{E}_{J_{i}}\left[\mathbb{E}_{J_{i+1}}\left[\sqrt{N}\left|\frac{\partial\tilde{y}}{\partial\tilde{H}_{i+1}(C_{i+1}(J_{i+1}))}-\frac{\partial\hat{y}}{\partial H_{i+1}(C_{i+1}(J_{i+1}))}\right|\left|w_{i+1}\left(C_{i}\left(J_{i}\right),C_{i+1}\left(J_{i+1}\right)\right)\right|\right]^{2k}\right]\\
 & \quad+K_{k}{\bf E}\mathbb{E}_{J_{i}}\left[\mathbb{E}_{J_{i+1}}\left[\sqrt{N}\left|\frac{\partial\hat{y}}{\partial H_{i+1}(C_{i+1}(J_{i+1}))}\right|\left|w_{i+1}\left(C_{i}\left(J_{i}\right),C_{i+1}\left(J_{i+1}\right)\right)\right|\left|\tilde{H}_{i}\left(C_{i}\left(J_{i}\right)\right)-H_{i}\left(C_{i}\left(J_{i}\right)\right)\right|\right]^{2k}\right]\\
 & \leq K_{T,k}+K_{k}{\bf E}\mathbb{E}_{J_{i+1}}\left[\left(\sqrt{N}\left|\frac{\partial\tilde{y}}{\partial\tilde{H}_{i+1}(C_{i+1}(J_{i+1}))}-\frac{\partial\hat{y}}{\partial H_{i+1}(C_{i+1}(J_{i+1}))}\right|\right)^{4k}\right]^{1/2}\\
 & \quad+K_{k}{\bf E}\mathbb{E}_{J_{i}}\left[\left(\sqrt{N}\left|\tilde{H}_{i}\left(C_{i}\left(J_{i}\right)\right)-H_{i}\left(C_{i}\left(J_{i}\right)\right)\right|\right)^{8k}\right]^{1/3}\\
 & \leq K_{T,k}.
\end{align*}
\item For the third claim:
\begin{align*}
 & \left(\sqrt{N}\left(\frac{\partial\tilde{y}}{\partial w_{i}(c_{i-1},c_{i})}-\frac{\partial\hat{y}}{\partial w_{i}(c_{i-1},c_{i})}\right)\right)^{2k}\\
 & \leq K_{k}\left(\sqrt{N}\left(\frac{\partial\tilde{y}}{\partial\tilde{H}_{i}(c_{i})}-\frac{\partial\hat{y}}{\partial H_{i}(c_{i})}\right)\right)^{2k}+K_{k}\left(\frac{\partial\hat{y}}{\partial H_{i}(c_{i})}\right)^{2k}\left(\sqrt{N}\left(\tilde{H}_{i-1}\left(c_{i-1}\right)-H_{i-1}\left(c_{i-1}\right)\right)\right)^{2k}.
\end{align*}
Combining the previous claims, it is then easy to see that
\[
{\bf E}\mathbb{E}_{J}\left[\left(\sqrt{N}\left(\frac{\partial\tilde{y}}{\partial w_{i}(C_{i-1}(J_{i-1}),C_{i}(J_{i}))}-\frac{\partial\hat{y}}{\partial w_{i}(C_{i-1}(J_{i-1}),C_{i}(J_{i}))}\right)\right)^{2k}\right]\le K_{T,k}.
\]
\item The fourth claim is also immediate from the first claim by observing
that
\[
\sqrt{N}\left|\partial_{2}{\cal L}(\cdot,\tilde{y})-\partial_{2}{\cal L}(\cdot,\hat{y})\right|\le K\left|\sqrt{N}\left(\tilde{y}-\hat{y}\right)\right|\leq K\left|\sqrt{N}\left(\tilde{H}_{L}\left(1\right)-H_{L}\left(1\right)\right)\right|.
\]
\end{itemize}
\end{proof}
\begin{lem}[Lipschitz bounds]
\label{lem:lipschitz_NN}We have for any $t\leq T$:
\begin{align*}
 & \left(\sqrt{N}({\bf H}_{i}(t,j_{i},x)-\tilde{H}_{i}(t,C_{i}(j_{i}),x)\right)^{2k}\le\|{\bf R}_{i}(j_{i})\|_{2k}^{2k}+K_{k}\|{\bf R}\|_{2k}^{2k}(1+\|\tilde{w}_{i}(j_{i})\|_{2k}^{2k})\prod_{j=1}^{i-1}(1+\|\tilde{w}_{i}\|_{2k}^{2k}),\\
 & \left(\sqrt{N}\left(\frac{\partial\hat{{\bf y}}(t,x)}{\partial{\bf H}_{i}(j_{i})}-\frac{\partial\tilde{y}(t,x)}{\partial\tilde{H}_{i}(C_{i}(j_{i}))}\right)\right)^{2k}\\
 & \quad\le K_{k}\|\tilde{w}_{i+1}(j_{i})\|_{2k}^{2k}\bigg(\|R\|_{2k}^{2k}(1+\|\tilde{w}_{i}(j_{i})\|_{2k}^{2k})(1+\|\tilde{w}_{i+1}\|_{4k}^{2k})\prod_{j\le i-1}(1+\|\tilde{w}_{j}\|_{2k}^{2k})\prod_{j\ge i+2}(1+\|\tilde{w}_{j}\|_{4k}^{4k})\\
 & \quad\quad+\sum_{j\ge i+1}\mathbb{E}_{J_{j}}[\|R_{j}(J_{j})\|_{2}^{2k}\|\tilde{w}_{j+1}(J_{j})\|_{2k}^{2k}]\prod_{i+2\le j',j'\ne j+1}(1+\|\tilde{w}_{j'}\|_{2k}^{2k})+\|R_{i}(j_{i})\|_{2}^{2k}\prod_{i+2\le j'}(1+\|\tilde{w}_{j'}\|_{2k}^{2k})\bigg)\\
 & \quad\quad+\|{\bf R}_{i+1}(j_{i})\|_{2}^{2k}\prod_{j\ge i+2}(1+\|\tilde{w}_{j}\|_{2k}^{2k}),\\
 & \left(\sqrt{N}\left(\frac{\partial\hat{{\bf y}}(t,x)}{\partial{\bf w}_{i}(j_{i-1},j_{i})}-\frac{\partial\tilde{y}(t,x)}{\partial w_{i}(C_{i-1}(j_{i-1}),C_{i}(j_{i}))}\right)\right)^{2k}\\
 & \quad\le K_{k}\left(\sqrt{N}\left(\frac{\partial\hat{{\bf y}}(t,x)}{\partial{\bf H}_{i}(j_{i})}-\frac{\partial\tilde{y}(t,x)}{\partial\tilde{H}_{i}(C_{i}(j_{i}))}\right)\right)^{2k}\\
 & \quad\quad+K_{k}\|\tilde{w}_{i+1}(j_{i})\|_{2}^{2k}\prod_{j\ge i+2}\|\tilde{w}_{j}\|_{2}^{2k}\cdot\left(\sqrt{N}({\bf H}_{i-1}(t,j_{i-1},x)-\tilde{H}_{i-1}(t,C_{i-1}(j_{i-1}),x)\right)^{2k},\\
 & \sqrt{N}\left(\partial_{2}{\cal L}(y,\hat{{\bf y}}(t,x))-\partial_{2}{\cal L}(y,\tilde{y}(t,x))\right)\le K_{k}\|{\bf R}\|_{2k}^{2k}\prod_{j=1}^{L}(1+\|\tilde{w}_{i}\|_{2k}^{2k}).\\
\end{align*}
In the above, we have dropped the notational dependency on $t$ on
the right-hand side for brevity.
\end{lem}

\begin{proof}
Let us consider the first claim:
\begin{align*}
 & \left(\sqrt{N}({\bf H}_{i}(t,j_{i},x)-\tilde{H}_{i}(t,C_{i}(j_{i}),x)\right)^{2k}\\
 & =\left(\sqrt{N}\mathbb{E}_{J_{i-1}}[{\bf w}_{i}(t,C_{i-1}(J_{i-1}),C_{i}(j_{i}))\varphi_{i-1}({\bf H}_{i-1}(t,C_{i-1}(J_{i-1}),x))-w_{i}(t,C_{i-1}(J_{i-1}),C_{i}(j_{i}))\varphi_{i-1}(\tilde{H}_{i-1}(t,C_{i-1}(J_{i-1}),x))]\right)^{2k}\\
 & \le K^{k}\left(\mathbb{E}_{J_{i-1}}\left[\left(\sqrt{N}({\bf w}_{i}(t,C_{i-1}(J_{i-1}),C_{i}(j_{i}))-w_{i}(t,C_{i-1}(J_{i-1}),C_{i}(j_{i})))\right)^{2}\right]\right)^{k}\\
 & \qquad+K^{k}\left(\mathbb{E}_{J_{i-1}}\left[\left(\sqrt{N}({\bf H}_{i-1}(t,C_{i-1}(J_{i-1}),x)-\tilde{H}_{i-1}(t,C_{i-1}(J_{i-1}),x))\right)^{2}\right]\mathbb{E}_{J_{i-1}}\left[w_{i}(t,C_{i-1}(J_{i-1}),C_{i}(j_{i}))^{2}\right]\right)^{k}\\
 & \le K^{k}\mathbb{E}_{J_{i-1}}\left[\left(\sqrt{N}({\bf w}_{i}(t,C_{i-1}(J_{i-1}),C_{i}(j_{i}))-w_{i}(t,C_{i-1}(J_{i-1}),C_{i}(j_{i})))\right)^{2}\right]^{k}\\
 & \qquad+K^{k}\mathbb{E}_{J_{i-1}}\left[\left(\sqrt{N}({\bf H}_{i-1}(t,C_{i-1}(J_{i-1}),x)-\tilde{H}_{i-1}(t,C_{i-1}(J_{i-1}),x))\right)^{2}\right]^{k}\mathbb{E}_{J_{i-1}}\left[w_{i}(t,C_{i-1}(J_{i-1}),C_{i}(j_{i}))^{2}\right]^{k}.
\end{align*}
Furthermore, 
\begin{align*}
\left(\sqrt{N}({\bf H}_{1}(t,j_{1},x)-\tilde{H}_{i}(t,C_{1}(j_{1}),x)\right)^{2k} & \le K^{k}\left(\sqrt{N}(\varphi_{1}({\bf w}_{1}(t,C_{1}(j_{1}))\cdot x)-\varphi_{1}(w_{1}(t,C_{1}(j_{1}))\cdot x)\right)^{2k}\\
 & \le K^{k}\left(\sqrt{N}\left|{\bf w}(t,C_{1}(j_{1}))-w_{1}(t,C_{1}(j_{1}))\right|\cdot|x|\right)^{2k}\\
 & \le K^{k}\left(\sqrt{N}\left|{\bf w}(t,C_{1}(j_{1}))-w_{1}(t,C_{1}(j_{1}))\right|\right)^{2k}.
\end{align*}
Thus, by induction, we obtain that 
\[
\left(\sqrt{N}({\bf H}_{i}(t,j_{i},x)-\tilde{H}_{i}(t,C_{i}(j_{i}),x)\right)^{2k}\le\|{\bf R}_{i}(t,j_{i})\|_{2k}^{2k}+K_{k}\|{\bf R}(t)\|_{2k}^{2k}(1+\|\tilde{w}_{i}(j_{i})\|_{2k}^{2k})\prod_{j=1}^{i-1}(1+\|\tilde{w}_{i}\|_{2k}^{2k}).
\]
The remaining claims can be obtained similarly. 
\end{proof}
\begin{lem}[A-priori moment estimate of $\mathbf{R}$]
\label{lem:a-priori-R_bold}We have for any $t\leq T$,
\[
{\bf E}\|{\bf R}(t)\|_{2k}^{2k}\le K_{k,T}.
\]
\end{lem}

\begin{proof}
In the following, let us drop the dependency on $t$ for brevity.
From the bounds in Lemmas \ref{lem:lipschitz_NN} and \ref{lem:a-priori-tilde_minus_MF},
we can obtain the estimate 
\begin{align*}
 & \left|\partial_{t}{\bf R}_{i}(j_{i-1},j_{i})\right|^{2k}\\
 & =\left|\sqrt{N}\mathbb{E}_{Z}\left[\partial_{2}{\cal L}\left(Y,\hat{\mathbf{y}}\left(X\right)\right)\frac{\partial\hat{{\bf y}}\left(X\right)}{\partial{\bf w}_{i}\left(j_{i-1},j_{i}\right)}-\partial_{2}{\cal L}\left(Y,\hat{y}\left(X\right)\right)\frac{\partial\hat{y}\left(X\right)}{\partial w_{i}\left(C_{i-1}\left(j_{i-1}\right),C_{i}\left(j_{i}\right)\right)}\right]\right|^{2k}\\
 & \leq K_{k}\left|\sqrt{N}\mathbb{E}_{Z}\left[\partial_{2}{\cal L}\left(Y,\hat{\mathbf{y}}\left(X\right)\right)\frac{\partial\hat{{\bf y}}\left(X\right)}{\partial{\bf w}_{i}\left(j_{i-1},j_{i}\right)}-\partial_{2}{\cal L}\left(Y,\tilde{y}\left(X\right)\right)\frac{\partial\tilde{y}\left(X\right)}{\partial w_{i}\left(C_{i-1}\left(j_{i-1}\right),C_{i}\left(j_{i}\right)\right)}\right]\right|^{2k}\\
 & \quad+K_{k}\left|\sqrt{N}\mathbb{E}_{Z}\left[\partial_{2}{\cal L}\left(Y,\tilde{y}\left(X\right)\right)\frac{\partial\tilde{y}\left(X\right)}{\partial w_{i}\left(C_{i-1}\left(j_{i-1}\right),C_{i}\left(j_{i}\right)\right)}-\partial_{2}{\cal L}\left(Y,\hat{y}\left(X\right)\right)\frac{\partial\hat{y}\left(X\right)}{\partial w_{i}\left(C_{i-1}\left(j_{i-1}\right),C_{i}\left(j_{i}\right)\right)}\right]\right|^{2k}\\
 & \leq K_{k}\left|\sqrt{N}\mathbb{E}_{Z}\left[\frac{\partial\hat{{\bf y}}\left(X\right)}{\partial{\bf w}_{i}\left(j_{i-1},j_{i}\right)}-\frac{\partial\tilde{y}\left(X\right)}{\partial w_{i}\left(C_{i-1}\left(j_{i-1}\right),C_{i}\left(j_{i}\right)\right)}\right]\right|^{2k}\\
 & \quad+K_{k}\left|\sqrt{N}\mathbb{E}_{Z}\left[\left(\partial_{2}{\cal L}\left(Y,\hat{\mathbf{y}}\left(X\right)\right)-\partial_{2}{\cal L}\left(Y,\tilde{y}\left(X\right)\right)\right)\frac{\partial\tilde{y}\left(X\right)}{\partial w_{i}\left(C_{i-1}\left(j_{i-1}\right),C_{i}\left(j_{i}\right)\right)}\right]\right|^{2k}\\
 & \quad+K_{k}\left|\sqrt{N}\mathbb{E}_{Z}\left[\frac{\partial\tilde{y}\left(X\right)}{\partial w_{i}\left(C_{i-1}\left(j_{i-1}\right),C_{i}\left(j_{i}\right)\right)}-\frac{\partial\hat{y}\left(X\right)}{\partial w_{i}\left(C_{i-1}\left(j_{i-1}\right),C_{i}\left(j_{i}\right)\right)}\right]\right|^{2k}\\
 & \quad+K_{k}\left|\sqrt{N}\mathbb{E}_{Z}\left[\left(\partial_{2}{\cal L}\left(Y,\tilde{y}\left(X\right)\right)-\partial_{2}{\cal L}\left(Y,\hat{y}\left(X\right)\right)\right)\frac{\partial\hat{y}\left(X\right)}{\partial w_{i}\left(C_{i-1}\left(j_{i-1}\right),C_{i}\left(j_{i}\right)\right)}\right]\right|^{2k}\\
 & \le K_{k}\|\tilde{w}_{i+1}(j_{i})\|_{2k}^{2k}(1+\|\tilde{w}_{i}(j_{i})\|_{2k}^{2k})\|{\bf R}\|_{2k}^{2k}\prod_{j\le L}(1+\|\tilde{w}_{j}\|_{4k}^{4k})\\
 & \quad+K_{k}\|\tilde{w}_{i+1}(j_{i})\|_{2k}^{2k}\|{\bf R}_{i}(j_{i})\|_{2k}^{2k}\prod_{j\le L}(1+\|\tilde{w}_{j}\|_{2k}^{2k})+K_{k}\|\tilde{w}_{i+1}(j_{i})\|_{2k}^{2k}\sum_{j\ge i+1}\mathbb{E}_{J_{j}}[\|{\bf R}_{j}(J_{j})\|_{2}^{2k}\|\tilde{w}_{j+1}(J_{j})\|_{2k}^{2k}]\prod_{j'\le L}(1+\|\tilde{w}_{j'}\|_{2k}^{2k})\\
 & \quad+K_{k}\|\tilde{w}_{i+1}(j_{i})\|_{2k}^{2k}\|{\bf R}_{i+1}(j_{i})\|_{2}^{2k}\prod_{j\ge i+2}(1+\|\tilde{w}_{j}\|_{2k}^{2k})+K_{k}\|{\bf R}_{i-1}(t,j_{i-1})\|_{2}^{2k}\\
 & \quad+K_{k}\|\tilde{w}_{i+1}(j_{i})\|_{2k}^{2k}\|{\bf R}\|_{2k}^{2k}(1+\|\tilde{w}_{i-1}(j_{i-1})\|_{2}^{2k})\prod_{j=1}^{L}(1+\|\tilde{w}_{j}\|_{2}^{2k})+A_{k}^{2k}\left(t,j_{i-1},j_{i}\right),
\end{align*}
where $A_{k}\left(t,j_{i-1},j_{i}\right)$ does not depend on $\mathbf{R}$
and satisfies the moment bounds in Lemma \ref{lem:a-priori-tilde_minus_MF},
i.e. $\mathbf{E}\mathbb{E}_{J}\left[A_{k}^{2k}\left(t,J_{i-1},J_{i}\right)\right]\leq K_{T,k}$.

For each $k$, consider the ODE, initialized at zero:
\[
\partial_{t}X_{i}(t,j_{i-1},j_{i})=K_{k}F[X](t,j_{i-1},j_{i})+A_{k}\left(t,j_{i-1},j_{i}\right),
\]
where
\begin{align*}
F[X](t,j_{i-1},j_{i}) & =\bigg(\|\tilde{w}_{i+1}(j_{i})\|_{2k}^{2k}(1+\|\tilde{w}_{i}(j_{i})\|_{2k}^{2k})\|X\|_{2k}^{2k}\prod_{j\le L}(1+\|\tilde{w}_{j}\|_{4k}^{4k})\bigg)^{1/2k}\\
 & \quad+\bigg(\|\tilde{w}_{i+1}(j_{i})\|_{2k}^{2k}\|X_{i}(j_{i})\|_{2k}^{2k}\prod_{j\le L}(1+\|\tilde{w}_{j}\|_{2k}^{2k})\bigg)^{1/2k}\\
 & \quad+\bigg(\|\tilde{w}_{i+1}(j_{i})\|_{2k}^{2k}\sum_{j\ge i+1}\mathbb{E}_{J_{j}}[\|X_{j}(J_{j})\|_{2k}^{2k}\|\tilde{w}_{j+1}(J_{j})\|_{2k}^{2k}]\prod_{j'\le L}(1+\|\tilde{w}_{j'}\|_{2k}^{2k})\bigg)^{1/2k}\\
 & \quad+\bigg(\|\tilde{w}_{i+1}(j_{i})\|_{2k}^{2k}\|X_{i+1}(j_{i})\|_{2k}^{2k}\prod_{j\ge i+2}(1+\|\tilde{w}_{j}\|_{2k}^{2k})\bigg)^{1/2k}+K_{k}\|X_{i-1}(t,j_{i-1})\|_{2k}\\
 & \quad+\bigg(\|\tilde{w}_{i+1}(j_{i})\|_{2k}^{2k}\|X\|_{2k}^{2k}(1+\|\tilde{w}_{i-1}(j_{i-1})\|_{2}^{2k})\prod_{j=1}^{L}(1+\|\tilde{w}_{j}\|_{2}^{2k})\bigg)^{1/2k}.
\end{align*}
The key property of $X$ is that for all $t,j_{i-1},j_{i}$, we have:
\[
|{\bf R}_{i}(t,j_{i-1},j_{i})|\le|X_{i}(t,j_{i-1},j_{i})|.
\]
For $B>0$, define 
\[
\partial_{t}X_{i}^{B}(t,j_{i-1},j_{i})=K_{k}F[X^{B}](t,j_{i-1},j_{i})\mathbb{B}^{B}(t,j_{i-1},j_{i})+A_{k}\left(t,j_{i-1},j_{i}\right).
\]
We have the estimate:
\[
\partial_{t}\|X^{B}(t)\|_{2k}^{2k}\le K_{k}\prod_{j'\le L}(1+\|\tilde{w}_{j'}^{B}(t)\|_{4k}^{4k})(1+B^{2k})\|X^{B}(t)\|_{2k}^{2k}+\Vert A_{k}(t)\Vert_{2k}^{2k}.
\]
Here, we denote $\|\tilde{w}^{B}\|_{p}=\min(B,\|\tilde{w}\|_{p})$.
Furthermore, for solutions $X^{B},\;X'^{B}$ to the equation, by the
triangle inequality, we have
\[
\left|\partial_{t}\left(X^{B}-X'^{B}\right)(t,j_{i-1},j_{i})\right|\leq K_{k}F[\left(X^{B}-X'^{B}\right)](t,j_{i-1},j_{i})\mathbb{B}^{B}(t,j_{i-1},j_{i}).
\]
Existence and uniqueness of $X^{B}$ follows easily from the above
Lipschitz estimate, and we furthermore have the bound 
\[
\|X^{B}(t)\|_{2k}\le K_{k}\exp(K_{k}B\prod_{j'\le L}(1+\sup_{s\le t}\|\tilde{w}_{j'}^{B}(t)\|_{4k}^{4k})^{1/2k}t)\cdot\Vert A_{k}(t)\Vert_{2k},
\]
and similar to the proof of Theorem \ref{thm:exist-R}, the bounds
(for sufficiently large $B$)
\[
\|X_{i}^{B}(t,j_{i-1})\|_{2k}\le K(1+\|\tilde{w}_{i}(t,j_{i-1})\|_{2k}^{2k})(1+\|\tilde{w}_{i-1}(t,j_{i-1})\|_{2k}^{2k})\exp\Big(K_{k}B\prod_{j'\le L}(1+\sup_{s\le t}\|\tilde{w}_{j'}^{B}(t)\|_{4k}^{4k})^{1/2k}t\Big)\Vert A_{k}(t)\Vert_{2k},
\]
\[
\|X_{i}^{B}(t,j_{i})\|_{2k}\le K(1+\|\tilde{w}_{i}(t,j_{i})\|_{2k}^{2k})(1+\|\tilde{w}_{i+1}(t,j_{i})\|_{2k}^{2k})\exp\Big(K_{k}B\prod_{j'\le L}(1+\sup_{s\le t}\|\tilde{w}_{j'}^{B}(t)\|_{4k}^{4k})^{1/2k}t\Big)\Vert A_{k}(t)\Vert_{2k}.
\]
Notice that $X^{B}$ is positive and monotonically increasing in $B$.
For $B'>B$ and $\mathbb{D}_{B,B'}(t,j_{i-1},j_{i})=\mathbb{I}(\mathbb{B}^{B'}(t,j_{i-1},j_{i})=1,\mathbb{B}^{B}(t,j_{i-1},j_{i})=0)$,
we have for $\Delta X=X^{B'}-X^{B}$:
\begin{align*}
 & \partial_{t}\mathbb{E}_{J}[|\Delta X(t,J_{i-1},J_{i})|^{2k}\mathbb{D}_{B,B'}(t,J_{i-1},J_{i})]\\
 & \le K_{k}\left(\mathbb{E}_{J}[|\Delta X(t,J_{i-1},J_{i})|^{2k}\mathbb{D}_{B,B'}(t,J_{i-1},J_{i})]\right)^{(2k-1)/(2k)}\\
 & \quad\times\bigg[\bigg(\prod_{j'\le L}(1+\|\tilde{w}_{j'}(t)\|_{4k}^{4k})(1+B'^{2k})\|X^{B'}(t)\|_{2k}^{2k}+\Vert A_{k}(t)\Vert_{2k}^{2k}\bigg)\mathbb{E}_{J}[\mathbb{D}_{B,B'}(t,J_{i-1},J_{i})]\\
 & \qquad\qquad+K_{k}\mathbb{E}_{J}\bigg[\prod_{j'\le L}(1+\|\tilde{w}_{j'}(t)\|_{4k}^{4k})(1+B'^{2k})\|X^{B'}(t)\|_{2k}^{2k}\\
 & \qquad\qquad\qquad\qquad\times\Big(\|X_{i}^{B'}(t,J_{i})\|_{2k}^{2k}+\|X_{i+1}^{B'}(t,J_{i})\|_{2k}^{2k}+\|X_{i-1}^{B'}(t,J_{i-1})\|_{2k}^{2k}\Big)\mathbb{D}_{B,B'}(t,J_{i-1},J_{i})\bigg]\bigg]^{1/(2k)},
\end{align*}
while 
\begin{align*}
\partial_{t}\mathbb{E}[|\Delta X(t,J_{i-1},J_{i})|^{2k}\mathbb{B}^{B}(t,J_{i-1},J_{i})] & \le K_{k}\bigg[\prod_{j'\le L}(1+\|\tilde{w}_{j'}(t)\|_{4k}^{4k})(1+B'^{2k})\bigg]\mathbb{E}[|\Delta X(t,J_{i-1},J_{i})|^{2k}\mathbb{B}^{B}(t,J_{i-1},J_{i})],
\end{align*}
and if $\mathbb{B}^{B'}(t,\cdot)=0$ then $\partial_{t}\Delta X(t,\cdot)=0$.
Thus,
\begin{align*}
\partial_{t}\Vert\Delta X(t)\Vert_{2k}^{2k} & \le K_{k}\bigg[\prod_{j'\le L}(1+\|\tilde{w}_{j'}(t)\|_{4k}^{4k})(1+B'^{2k})\bigg]^{1/2k}\|\Delta X(t)\|_{2k}^{2k}\\
 & \quad+K_{k}\mathbb{E}[\mathbb{D}_{B,B'}(t,J_{i-1},J_{i})]\exp\Big(K_{T}B'\prod_{j'\le L}(1+\sup_{s\le t}\|\tilde{w}_{j'}^{B'}(s)\|_{4k}^{4k})^{1/2k}t\Big).
\end{align*}
This gives:
\[
\|\Delta X(t)\|_{2k}^{2k}\le K_{k}\int_{0}^{t}\mathbb{E}[\mathbb{D}_{B,B'}(s,J_{i-1},J_{i})]\exp(K_{T}B'\prod_{j'\le L}(1+\sup_{s\le t}\|\tilde{w}_{j'}^{B'}(s)\|_{4k}^{4k})^{1/2k})ds.
\]
Therefore, 
\[
{\bf E}\|\Delta X(t)\|_{2k}^{2k}\le K_{k,T}\Big(\int_{0}^{t}{\bf E}\mathbb{E}[\mathbb{D}_{B,B'}(s,J_{i-1},J_{i})]ds\Big)^{1/2}{\bf E}\Big[\exp(K_{T}B'\prod_{j'\le L}(1+\sup_{s\le t}\|\tilde{w}_{j'}^{B'}(s)\|_{4k}^{4k})^{1/2k})\Big]^{1/2}.
\]
Notice that
\[
\exp\Big(K_{T}B'\prod_{j'\le L}(1+\sup_{s\le t}\|\tilde{w}_{j'}^{B'}(s)\|_{4k}^{4k})^{1/2k}\Big)\le\exp\Big(K_{T}B'\prod_{j'\le L}(1+\|\tilde{w}_{j'}^{B'}(0)\|_{Kk}^{Kk})^{1/2k}\Big).
\]
Furthermore Lemma \ref{lem:trunc-bound-w} implies that 
\[
\Big(\int_{0}^{t}{\bf E}\mathbb{E}[\mathbb{D}_{B,B'}(s,J_{i-1},J_{i})]ds\Big)^{1/2}\le\exp(-K_{T}B^{2}),
\]
for all $B\le N^{1/16}$, while 
\[
{\bf E}\Big[\exp\Big(K_{T}B'\prod_{j'\le L}(1+\|\tilde{w}_{j'}^{B'}(0)\|_{Kk}^{Kk})^{1/2k}\Big)\Big]\le{\bf E}\Big[\exp\Big(K_{T}B'(1+\sum_{j'\le L}\|\tilde{w}_{j'}^{B'}(0)\|_{Kk}^{K})\Big)\Big].
\]
Using Remark \ref{rem:trunc-bound-w-remark}, and $\|\tilde{w}_{j'}^{B'}(0)\|_{Kk}^{K}\le B'{}^{K}$,
we have that if $N\ge(K_{T}B'^{C})^{16}$ for some sufficiently large
constant $C>0$, then 
\[
{\bf E}\Big[\exp\Big(K_{T}B'\prod_{j'\le L}(1+\|\tilde{w}_{j'}^{B'}(0)\|_{Kk}^{Kk})^{1/2k}\Big)\Big]\le\exp(K_{T}B').
\]
Combining the estimates, we obtain that for all $B\le K_{T}N^{c}$
for a sufficiently small constant $c>0$,
\[
{\bf E}\|X^{B}\|_{2k}^{2k}\le{\bf E}\|X^{K_{T}}\|_{2k}^{2k}+K_{k}\int_{K_{T}}^{B}\exp(K_{T}x)\exp(-Kx^{2})dx\le K_{T}.
\]

Finally, for $\gamma>0$ consider the event ${\cal E}$ that $\max_{i}\max_{j_{i}}\max(\|\tilde{w}_{i}(j_{i})\|_{Kk},\|\tilde{w}_{i+1}(j_{i})\|_{Kk},\|{\bf w}_{i}(j_{i})\|_{Kk},\|{\bf w}_{i+1}(j_{i})\|_{Kk})\le\gamma$.
This event has probability at least $1-KN\exp(-K_{k}\gamma)$ for
$\gamma\le N^{c_{k}}$ with sufficiently small $c_{k}$. We have for
$B=K_{T}N^{c}$ that 
\[
{\bf E}\left[\left|\|X\|_{2k}^{2k}-\|X^{B}\|_{2k}^{2k}\right|\mathbb{I}({\cal E})\right]\le\exp(-KB^{2})\exp(K_{T}\gamma^{K}).
\]
For $\gamma=N^{\epsilon}$ with $\epsilon$ sufficiently small compared
to $c$ and $c_{k}$, we then have 
\[
{\bf E}\left[\left|\|X\|_{2k}^{2k}-\|X^{B}\|_{2k}^{2k}\right|\mathbb{I}({\cal E})\right]\le1.
\]
This immediately implies 
\[
{\bf E}\left[\|{\bf R}\|_{2k}^{2k}\mathbb{I}({\cal E})\right]\le{\bf E}\left[\|X\|_{2k}^{2k}\mathbb{I}({\cal E})\right]\le K_{T}.
\]
On the other hand, 
\[
{\bf E}\left[\|{\bf R}\|_{2k}^{2k}(1-\mathbb{I}({\cal E}))\right]\le{\bf E}[\|{\bf R}\|_{2k}^{4k}]^{1/2}(1-{\bf P}({\cal E}))^{1/2}.
\]
Recall the definition of ${\bf R}$, and noting that the moments of
${\bf w}$ and $w$ are finite by Lemma \ref{lem:trunc-bound-w},
we have the trivial bound 
\[
{\bf E}[\|{\bf R}\|_{2k}^{4k}]^{1/2}\le N^{K_{k}}K_{T}.
\]
Thus 
\[
{\bf E}\left[\|{\bf R}\|_{2k}^{2k}(1-\mathbb{I}({\cal E}))\right]\le N^{K_{k}}K_{T}\exp(-c_{k}N^{2\epsilon})<1.
\]
We can thus conclude that 
\[
{\bf E}[\|{\bf R}\|_{2k}^{2k}]<K_{k,T}.
\]
\end{proof}

\subsection{Comparison bounds and concentration estimates}

Let us recall the notation $\mathbf{O}_{T}(\cdot)$ from Appendix
\ref{Appendix:Gaussian}.
\begin{lem}
\label{lem:a-priori-NN_minus_MF}We have for any $t\leq T$,
\[
{\bf E}\mathbb{E}_{J}\left[\left(\sqrt{N}({\bf H}_{i}(t,J_{i},x)-H_{i}(t,C_{i}(J_{i}),x)\right)^{2k}\right]\le K_{k,T},
\]
\[
{\bf E}\mathbb{E}_{J}\left[\left(\sqrt{N}\left(\frac{\partial\hat{{\bf y}}(t,x)}{\partial{\bf H}_{i}(J_{i})}-\frac{\partial\hat{y}(t,x)}{\partial H_{i}(C_{i}(J_{i}))}\right)\right)^{2k}\right]\le K_{k,T},
\]
\[
{\bf E}\mathbb{E}_{J}\left[\left(\sqrt{N}\left(\frac{\partial\hat{{\bf y}}(t,x)}{\partial{\bf w}_{i}(J_{i-1},J_{i})}-\frac{\partial\hat{y}(t,x)}{\partial w_{i}(C_{i-1}(J_{i-1}),C_{i}(J_{i}))}\right)\right)^{2k}\right]\le K_{k,T},
\]
\[
{\bf E}\left[\left(\sqrt{N}\left(\partial_{2}{\cal L}(y,\hat{{\bf y}}(t,x))-\partial_{2}{\cal L}(y,\hat{y}(t,x))\right)\right)^{2k}\right]\le K_{k,T}.
\]
\end{lem}

\begin{proof}
This is a consequence of Lemmas \ref{lem:a-priori-tilde_minus_MF},
\ref{lem:lipschitz_NN}, \ref{lem:a-priori-R_bold} and \ref{lem:trunc-bound-w}.
\end{proof}
\begin{lem}
\label{lem:dL_NN_minus_MF}We have for $t\leq T$,
\[
\sqrt{N}\left(\partial_{2}{\cal L}(y,\hat{{\bf y}}(t,x))-\partial_{2}{\cal L}(y,\hat{y}(t,x))\right)=\sqrt{N}(\hat{{\bf y}}(t,x)-\hat{y}(t,x))\partial_{2}^{2}{\cal L}(y,\hat{y}(t,x))+{\bf O}_{T}(1/\sqrt{N}).
\]
\end{lem}

\begin{proof}
We have 
\[
\sqrt{N}\left(\partial_{2}{\cal L}(y,\hat{{\bf y}}(t,x))-\partial_{2}{\cal L}(y,\hat{y}(t,x))\right)=\sqrt{N}(\hat{{\bf y}}(t,x)-\hat{y}(t,x))\partial_{2}^{2}{\cal L}(y,y'(t,x)),
\]
for some $y'(t,x)$ between $\hat{y}(t,x)$ and $\hat{{\bf y}}(t,x)$.
Hence, 
\begin{align*}
 & {\bf E}\mathbb{E}_{Z}\left[\left(\sqrt{N}\left(\partial_{2}{\cal L}(Y,\hat{{\bf y}}(t,X))-\partial_{2}{\cal L}(Y,\hat{y}(t,X))\right)-\sqrt{N}(\hat{{\bf y}}(t,X)-\hat{y}(t,X))\partial_{2}^{2}{\cal L}(Y,\hat{y}(t,X))\right)^{2}\right]\\
 & \le K{\bf E}\mathbb{E}_{X}\left[N(\hat{{\bf y}}(t,X)-\hat{y}(t,X))^{2}\cdot\left(\hat{{\bf y}}(t,X)-\hat{y}(t,X)\right)^{2}\right]\\
 & \le\frac{K}{N}{\bf E}\mathbb{E}_{X}\left[\left(\sqrt{N}(\hat{{\bf y}}(t,X)-\hat{y}(t,X))\right)^{4}\right]\\
 & \le\frac{K_{T}}{N},
\end{align*}
where the last step follows easily from Lemma \ref{lem:a-priori-NN_minus_MF}.
\end{proof}
\begin{lem}
\label{lem:decomp-y}We have for $t\leq T$,
\[
\sqrt{N}\left(\hat{{\bf y}}(t,x)-\tilde{y}(t,x)\right)=\sum_{i=1}^{L}\mathbb{E}_{J}\left[{\bf R}_{i}(t,J_{i-1},J_{i})\frac{\partial\hat{y}(t,x)}{\partial w_{i}(C_{i-1}(J_{i-1}),C_{i}(J_{i}))}\right]+{\bf O}_{T}(1/\sqrt{N}).
\]
\end{lem}

\begin{proof}
Following the same argument as Lemma \ref{lem:dL_NN_minus_MF}, we
get:
\begin{align*}
\sqrt{N}\left(\hat{{\bf y}}(t,x)-\tilde{y}(t,x)\right) & =\sum_{i=1}^{L}\mathbb{E}_{J}\left[{\bf R}_{i}(t,J_{i-1},J_{i})\frac{\partial\hat{{\bf y}}(t,x)}{\partial{\bf w}_{i}(C_{i-1}(J_{i-1}),C_{i}(J_{i}))}\right]+{\bf O}_{T}(1/\sqrt{N}).
\end{align*}
Again the same argument applied to each of the summands leads to:
\begin{align*}
 & {\bf E}\left[\left(\mathbb{E}_{J}\left[{\bf R}_{i}(t,J_{i-1},J_{i})\frac{\partial\hat{{\bf y}}(t,x)}{\partial{\bf w}_{i}(C_{i-1}(J_{i-1}),C_{i}(J_{i}))}\right]-\mathbb{E}_{J}\left[{\bf R}_{i}(t,J_{i-1},J_{i})\frac{\partial\hat{y}(t,x)}{\partial w_{i}(C_{i-1}(J_{i-1}),C_{i}(J_{i}))}\right]\right)^{2}\right]\\
 & \le\frac{1}{N}\left({\bf E}\mathbb{E}_{J}\left[{\bf R}_{i}(t,J_{i-1},J_{i})^{4}\right]\right)^{1/2}\left({\bf E}\mathbb{E}_{J}\left[\left(\sqrt{N}\left(\frac{\partial\hat{{\bf y}}(t,x)}{\partial{\bf w}_{i}(C_{i-1}(J_{i-1}),C_{i}(J_{i}))}-\frac{\partial\hat{y}(t,x)}{\partial w_{i}(C_{i-1}(J_{i-1}),C_{i}(J_{i}))}\right)\right)^{4}\right]\right)^{1/2}\\
 & \le K_{T}/N,
\end{align*}
by Lemmas \ref{lem:a-priori-R_bold} and \ref{lem:a-priori-NN_minus_MF}.
\end{proof}
\begin{lem}
\label{lem:R-JC}We have for $t\leq T$,
\[
\mathbb{E}_{J}\left[R_{i}(\tilde{G},t,C_{i-1}(J_{i-1}),C_{i}(J_{i}))\frac{\partial\hat{y}(t,x)}{\partial w_{i}(C_{i-1}(J_{i-1}),C_{i}(J_{i}))}\right]=\mathbb{E}_{C}\left[R_{i}(\tilde{G},t,C_{i-1},C_{i})\frac{\partial\hat{y}(t,x)}{\partial w_{i}(C_{i-1},C_{i})}\right]+{\bf O}_{T}(N^{-1/2}).
\]
\end{lem}

\begin{proof}
We verify the condition of Lemma \ref{lem:second-moment-dependency}
for ${\bf i}=(i-1,i)$, ${\bf i}'=\emptyset$, $S=\{C_{k}(j_{k}):\;j_{k}\in[N_{k}],\;k\in[L]\}$
and 
\[
f(S,c_{i-1},c_{i})=R_{i}(\tilde{G},t,c_{i-1},c_{i})\frac{\partial\hat{y}(t,x)}{\partial w_{i}(c_{i-1},c_{i})}.
\]
We have from Lemma \ref{lem:lp-R} and Theorem \ref{thm:clt-G}:
\[
{\bf E}\left[\mathbb{E}_{C}\left[\left|R_{i}(\tilde{G},t,C_{i-1},C_{i})\right|^{4}\right]\right]\le K_{T},
\]
and therefore, with Lemma \ref{lem:MF_a_priori},
\[
{\bf E}\left[\left|f(S^{j_{{\bf i}}},C_{{\bf i}}(j_{{\bf i}}))\right|^{2}\right]\leq K_{T}.
\]
Let $\tilde{G}^{j_{{\bf i}}}$ be obtained similar to $\tilde{G}$
but with $C_{{\bf i}}(j_{{\bf i}})$ replaced by an independent copy
$C_{{\bf i}}'(j_{{\bf i}})$. It remains to bound
\[
{\bf E}\left[\left|f(S,C_{{\bf i}}(j_{{\bf i}}))-f(S^{j_{{\bf i}}},C_{{\bf i}}(j_{{\bf i}}))\right|^{2}\right]={\bf E}\left[\left|R_{i}(\tilde{G},t,C_{{\bf i}}(j_{{\bf i}}))\frac{\partial\hat{y}(t,x)}{\partial w_{i}(C_{{\bf i}}(j_{{\bf i}}))}-R_{i}(\tilde{G}^{j_{\mathbf{i}}},t,C_{{\bf i}}(j_{{\bf i}}))\frac{\partial\hat{y}(t,x)}{\partial w_{i}(C_{{\bf i}}(j_{{\bf i}}))}\right|^{2}\right],
\]
which -- similar to the previous second moment bound -- amounts
to proving
\[
{\bf E}\left[\left|R_{i}(\tilde{G},t,C_{{\bf i}}(j_{{\bf i}}))-R_{i}(\tilde{G}^{j_{{\bf i}}},t,C_{{\bf i}}(j_{{\bf i}}))\right|^{4}\right]\leq K_{T}/N^{2}.
\]
The proof is complete upon having this bound.

By linearity from Theorem \ref{thm:clt-G},
\[
R_{i}(\tilde{G},t,c_{i-1},c_{i})-R_{i}(\tilde{G}^{j_{{\bf i}}},t,c_{i-1},c_{i})=R_{i}(\tilde{G}-\tilde{G}^{j_{{\bf i}}},t,c_{i-1},c_{i}),
\]
and we also have:
\begin{align*}
 & -\partial_{t}R_{i}(\tilde{G}-\tilde{G}^{j_{{\bf i}}},t,c_{i-1},c_{i})\\
 & =\mathbb{E}_{Z}\left[\frac{\partial\hat{y}(t,X)}{\partial w_{i}(c_{i-1},c_{i})}\cdot\partial_{2}^{2}{\cal L}(Y,\hat{y}(t,X))\cdot\left((\tilde{G}-\tilde{G}^{j_{{\bf i}}})^{y}(t,X)+\sum_{j=1}^{L}\mathbb{E}_{C_{j-1}',C_{j}'}\left\{ R_{j}(\tilde{G}-\tilde{G}^{j_{{\bf i}}},t,C_{j-1}',C_{j}')\frac{\partial\hat{y}(t,X)}{\partial w_{j}'(C_{j-1}',C_{j}')}\right\} \right)\right]\\
 & \quad+\mathbb{E}_{Z}\left[\partial_{2}{\cal L}(Y,\hat{y}(t,X))\cdot\left((\tilde{G}-\tilde{G}^{j_{{\bf i}}})_{i}^{w}(t,c_{i-1},c_{i},X)+\sum_{j=1}^{L}\mathbb{E}_{C_{j-1}',C_{j}'}\left\{ R_{j}(\tilde{G}-\tilde{G}^{j_{{\bf i}}},t,C_{j-1}',C_{j}')\frac{\partial^{2}\hat{y}(t,X)}{\partial w_{j}(C_{j-1}',C_{j}')\partial w_{i}(c_{i-1},c_{i})}\right\} \right)\right]\\
 & \quad+\mathbb{E}_{Z}\left[\partial_{2}{\cal L}(Y,\hat{y}(t,X))\cdot\mathbb{E}_{C_{i-1}'}\left\{ R_{i}(\tilde{G}-\tilde{G}^{j_{{\bf i}}},t,C_{i-1}',c_{i})\frac{\partial^{2}\hat{y}(t,X)}{\partial_{*}w_{i}(C_{i-1}',c_{i})\partial w_{i}(c_{i-1},c_{i})}\right\} \right]\\
 & \quad+\mathbb{E}_{Z}\left[\partial_{2}{\cal L}(Y,\hat{y}(t,X))\cdot\mathbb{E}_{C_{i+1}'}\left\{ R_{i}(\tilde{G}-\tilde{G}^{j_{{\bf i}}},t,c_{i},C_{i+1}')\frac{\partial^{2}\hat{y}(t,X)}{\partial_{*}w_{i+1}(c_{i},C_{i+1}')\partial w_{i}(c_{i-1},c_{i})}\right\} \right]\\
 & \quad+\mathbb{E}_{Z}\left[\partial_{2}{\cal L}(Y,\hat{y}(t,X))\cdot\mathbb{E}_{C_{i-2}'}\left\{ R_{i}(\tilde{G}-\tilde{G}^{j_{{\bf i}}},t,C_{i-2}',c_{i-1})\frac{\partial^{2}\hat{y}(t,X)}{\partial_{*}w_{i-1}(C_{i-2}',c_{i-1})\partial w_{i}(c_{i-1},c_{i})}\right\} \right].
\end{align*}
Note that 
\begin{align*}
\tilde{G}_{j}^{\tilde{H}}-\tilde{G}_{j}^{j_{{\bf i}},\tilde{H}} & =0,\qquad\forall j\le i-1,\\
(\tilde{G}_{i}^{\tilde{H}}-\tilde{G}_{i}^{j_{{\bf i}},\tilde{H}})(t,c_{i},x) & =\frac{1}{\sqrt{N}}\bigg(w_{i}(t,C_{i-1}(j_{i-1}),c_{i})\varphi_{i-1}\left(\tilde{H}_{i-1}(t,C_{i-1}(j_{i-1}),x)\right)\\
 & \qquad\qquad-w_{i}(t,C_{i-1}'(j_{i-1}),c_{i})\varphi_{i-1}\left(\tilde{H}_{i-1}^{j_{{\bf i}}}(t,C_{i-1}'(j_{i-1}),x)\right)\bigg),\\
(\tilde{G}_{i+1}^{\tilde{H}}-\tilde{G}_{i+1}^{j_{{\bf i}},\tilde{H}})(t,c_{i+1},x) & =\frac{1}{\sqrt{N}}\left(w_{i+1}(t,C_{i}(j_{i}),c_{i+1})\varphi_{i}\left(\tilde{H}_{i}(t,C_{i}(j_{i}),x)\right)-w_{i}(t,C_{i}'(j_{i}),c_{i+1})\varphi_{i}\left(\tilde{H}_{i}^{j_{{\bf i}}}(t,C_{i}'(j_{i}),x)\right)\right),\\
 & \qquad+\frac{1}{\sqrt{N}}\sum_{j_{i}'\ne j_{i}}w_{i+1}(t,C_{i}(j_{i}'),c_{i+1})\left(\varphi_{i}\left(\tilde{H}_{i}(t,C_{i}(j_{i}'),x)\right)-\varphi_{i}\left(\tilde{H}_{i}^{j_{{\bf i}}}(t,C_{i}(j_{i}'),x)\right)\right),\\
(\tilde{G}_{\ell+1}^{\tilde{H}}-\tilde{G}_{\ell+1}^{j_{{\bf i}},\tilde{H}})(t,c_{\ell+1},x) & =\frac{1}{\sqrt{N}}\sum_{j_{\ell}}w_{\ell+1}(t,C_{\ell}(j_{\ell}),c_{\ell+1})\left(\varphi_{\ell}\left(\tilde{H}_{\ell}(t,C_{\ell}(j_{\ell}),x)\right)-\varphi_{\ell}\left(\tilde{H}_{\ell}^{{\bf i}}(t,C_{\ell}(j_{\ell}),x)\right)\right),\qquad\forall\ell\ge i+1.
\end{align*}
We bound their expected $2p$-powers. We have that for $\ell\ge i+1$,
\begin{align*}
{\bf E}\mathbb{E}_{J_{\ell+1}}\left[\left|(\tilde{G}_{\ell+1}^{\tilde{H}}-\tilde{G}_{\ell+1}^{j_{{\bf i}},\tilde{H}})(t,C_{\ell+1}(J_{\ell+1}),x)\right|^{2p}\right] & \le K_{T,p}{\bf E}\mathbb{E}_{J_{\ell}}\left[\left|(\tilde{G}_{\ell}^{\tilde{H}}-\tilde{G}_{\ell}^{j_{{\bf i}},\tilde{H}})(t,C_{\ell}(J_{\ell}),x)\right|^{2p+2}\right],\\
{\bf E}\mathbb{E}_{C_{\ell+1}}\left[\left|(\tilde{G}_{\ell+1}^{\tilde{H}}-\tilde{G}_{\ell+1}^{j_{{\bf i}},\tilde{H}})(t,C_{\ell+1},x)\right|^{2p}\right] & \le K_{T,p}{\bf E}\mathbb{E}_{J_{\ell}}\left[\left|(\tilde{G}_{\ell}^{\tilde{H}}-\tilde{G}_{\ell}^{j_{{\bf i}},\tilde{H}})(t,C_{\ell}(J_{\ell}),x)\right|^{2p+2}\right].
\end{align*}
For $\ell=i$, by Lemma \ref{lem:MF_a_priori},
\begin{align*}
 & {\bf E}\mathbb{E}_{J_{i+1}}\left[\left|(\tilde{G}_{i+1}^{\tilde{H}}-\tilde{G}_{i+1}^{j_{{\bf i}},\tilde{H}})(t,C_{i+1}(J_{i+1}),x)\right|^{2p}\right]\\
 & \le K_{T,p}{\bf E}\mathbb{E}_{J_{i}}\left[\left|(\tilde{G}_{i}^{\tilde{H}}-\tilde{G}_{i}^{j_{{\bf i}},\tilde{H}})(t,C_{i}(J_{i}),x)\right|^{2p+2}\right]\\
 & \quad+K_{T,p}{\bf E}\mathbb{E}_{J_{i+1}}\left[\left|\frac{1}{\sqrt{N}}\left(w_{i+1}(t,C_{i}(j_{i}),C_{i+1}(J_{i+1}))\varphi_{i}\left(\tilde{H}_{i}(t,C_{i}(j_{i}),x)\right)-w_{i}(t,C_{i}'(j_{i}),C_{i+1}(J_{i+1}))\varphi_{i}\left(\tilde{H}_{i}^{j_{\mathbf{i}}}(t,C_{i}'(j_{i}),x)\right)\right)\right|^{2p}\right]\\
 & \le K_{T,p}{\bf E}\mathbb{E}_{J_{i}}\left[\left|(\tilde{G}_{i}^{\tilde{H}}-\tilde{G}_{i}^{j_{{\bf i}},\tilde{H}})(t,C_{i}(J_{i}),x)\right|^{2p+2}\right]+K_{T,p}N^{-p}.
\end{align*}
Similarly, 
\begin{align*}
{\bf E}\mathbb{E}_{C_{i+1}}\left[\left|(\tilde{G}_{i+1}^{\tilde{H}}-\tilde{G}_{i+1}^{j_{{\bf i}},\tilde{H}})(t,C_{i+1},x)\right|^{2p}\right] & \le KK_{T,p}{\bf E}\mathbb{E}_{J_{i}}\left[\left|(\tilde{G}_{i}^{\tilde{H}}-\tilde{G}_{i}^{j_{{\bf i}},\tilde{H}})(t,C_{i}(J_{i}),x)\right|^{2p+2}\right]+K_{T,p}N^{-p}.
\end{align*}
Finally at $\ell=i-1$, by the same argument:
\begin{align*}
{\bf E}\mathbb{E}_{J_{i}}\left[\left|(\tilde{G}_{i}^{\tilde{H}}-\tilde{G}_{i}^{j_{{\bf i}},\tilde{H}})(t,C_{i}(J_{i}),x)\right|^{2p}\right] & \le K_{T,p}N^{-p},\\
{\bf E}\mathbb{E}_{C_{i}}\left[\left|(\tilde{G}_{i}^{\tilde{H}}-\tilde{G}_{i}^{j_{{\bf i}},\tilde{H}})(t,C_{i}(C_{i}),x)\right|^{2p}\right] & \le K_{T,p}N^{-p}.
\end{align*}
As such, all these expected $2p$-powers are bounded by $K_{T,p}N^{-p}$.
We derive similar bounds for $\tilde{G}^{\partial\tilde{H}}-\tilde{G}^{j_{{\bf i}},\partial\tilde{H}}$,
$\tilde{G}^{w}-\tilde{G}^{j_{{\bf i}},w}$.

The estimates in Theorem \ref{thm:exist-R} imply that
\[
\mathbb{E}_{C_{j-1}',C_{j}'}\left\{ \left|R_{j}(\tilde{G}-\tilde{G}^{j_{\mathbf{i}}},t,C_{j-1}',C_{j}')\right|^{2}\right\} \le\inf_{B}\left\{ \exp(K_{T}B)\|\tilde{G}-\tilde{G}^{j_{\mathbf{i}}}\|_{T,2}^{2}+\exp(-K_{T}B^{2})\|\tilde{G}-\tilde{G}^{j_{\mathbf{i}}}\|_{T,4}^{2}\right\} \le K_{T}N^{-1}.
\]
Similarly, letting 
\begin{align*}
M_{1}(c_{i}) & =\sup_{t\leq T}\max(\|w_{i}(t,c_{i})\|_{2},\|w_{i+1}(t,c_{i})\|_{2},|w_{i}(t,C_{i-1}(j_{i-1}),c_{i})|,|w_{i}(t,C_{i-1}(j_{i-1}'),c_{i})|),\\
M_{2}(c_{i-1}) & =\sup_{t\leq T}\max(\|w_{i}(t,c_{i-1})\|_{2},\|w_{i-1}(t,c_{i-1})\|_{2},|w_{i}(t,c_{i-1},C_{i}(j_{i}))|,|w_{i}(t,c_{i-1},C_{i}(j_{i}'))|),
\end{align*}
we have 
\begin{align*}
\mathbb{E}_{C_{i-1}'}\left\{ R_{i}(\tilde{G}-\tilde{G}^{j_{{\bf i}}},t,C_{i-1}',c_{i})^{2}\right\}  & \le N^{-1}\exp(K_{T}M_{1}(c_{i})),\\
\mathbb{E}_{C_{i+1}'}\left\{ R_{i+1}(\tilde{G}-\tilde{G}^{j_{{\bf i}}},t,c_{i},C_{i+1}')^{2}\right\}  & \le N^{-1}\exp(K_{T}M_{2}(c_{i})),\\
\mathbb{E}_{C_{i-2}'}\left\{ R_{i-1}(\tilde{G}-\tilde{G}^{j_{{\bf i}}},t,C_{i-2}',c_{i-1})^{2}\right\}  & \le N^{-1}\exp(K_{T}M_{1}(c_{i-1})).
\end{align*}
All of these estimates give a final bound:
\[
\partial_{t}\left|R_{i}(\tilde{G}-\tilde{G}^{j_{{\bf i}}},t,c_{i-1},c_{i})\right|\leq N^{-1/2}\left[K_{T}+\exp(K_{T}M_{1}(c_{i}))+\exp(K_{T}M_{2}(c_{i}))+\exp(K_{T}M_{1}(c_{i-1}))\right].
\]
In particular, since it is initialized at zero, we have:
\[
\left|R_{i}(\tilde{G}-\tilde{G}^{j_{{\bf i}}},t,c_{i-1},c_{i})\right|^{4}\leq N^{-2}\left[K_{T}+\exp(K_{T}M_{1}(c_{i}))+\exp(K_{T}M_{2}(c_{i}))+\exp(K_{T}M_{1}(c_{i-1}))\right].
\]
Lemma \ref{lem:MF_a_priori} implies that $\mathbf{E}\left[\exp(K_{T}M_{1}(C_{i}(j_{i}))\right]\leq K_{T}$
and similarly for other terms. We thus get:
\[
{\bf E}\left[\left|R_{i}(\tilde{G},t,C_{{\bf i}}(j_{{\bf i}}))-R_{i}(\tilde{G}^{j_{{\bf i}}},t,C_{{\bf i}}(j_{{\bf i}}))\right|^{4}\right]\leq K_{T}/N^{2},
\]
as desired.
\end{proof}
\begin{lem}
\label{lem:R-JC_NN}We have for $t\leq T$ and any $B>1$,
\begin{align*}
 & \mathbb{E}_{J}\left[{\bf R}_{i}(t,J_{i-1},J_{i})\frac{\partial\hat{y}(t,x)}{\partial w_{i}(C_{i-1}(J_{i-1}),C_{i}(J_{i}))}\right]\\
 & =K_{T}\mathbb{E}_{J}\left[\left({\bf R}_{i}(t,J_{i-1},J_{i})-R_{i}(\tilde{G},t,C_{i-1}(J_{i-1}),C_{i}(J_{i}))\right)^{2}\right]^{1/2}+\mathbb{E}_{C}\left[R_{i}(\tilde{G},t,C_{i-1},C_{i})\frac{\partial\hat{y}(t,x)}{\partial w_{i}(C_{i-1},C_{i})}\right]+{\bf O}_{T}(N^{-1/2}).
\end{align*}
\end{lem}

\begin{proof}
We have:
\begin{align*}
 & \left|\mathbb{E}_{J}\left[\left({\bf R}_{i}(t,J_{i-1},J_{i})-R_{i}(\tilde{G},t,C_{i-1}(J_{i-1}),C_{i}(J_{i}))\right)\frac{\partial\hat{y}(t,x)}{\partial w_{i}(C_{i-1}(J_{i-1}),C_{i}(J_{i}))}\right]\right|^{2}\\
 & \le\mathbb{E}_{J}\left[\left({\bf R}_{i}(t,J_{i-1},J_{i})-R_{i}(\tilde{G},t,C_{i-1}(J_{i-1}),C_{i}(J_{i}))\right)^{2}\right]\mathbb{E}_{J}\left[\left|\frac{\partial\hat{y}(t,x)}{\partial w_{i}(C_{i-1}(J_{i-1}),C_{i}(J_{i}))}\right|^{2}\right].
\end{align*}
For $i=L$:
\[
\mathbb{E}_{J}\left[\left|\frac{\partial\hat{y}(t,x)}{\partial w_{i}(C_{i-1}(J_{i-1}),C_{i}(J_{i}))}\right|^{2}\right]\leq K
\]
and so the conclusion follows from Lemma \ref{lem:R-JC}.

For $i<L$, we have for any $B>1$:
\begin{align*}
\mathbb{E}_{J}\left[\left|\frac{\partial\hat{y}(t,x)}{\partial w_{i}(C_{i-1}(J_{i-1}),C_{i}(J_{i}))}\right|^{2}\right] & \leq K\mathbb{E}_{J}\mathbb{E}_{C_{i+1}}\left[\left|w_{i+1}(C_{i}(J_{i}),C_{i+1})\right|^{2}\right]\\
 & \leq KB+K\mathbb{E}_{J}\mathbb{E}_{C_{i+1}}\left[\left|w_{i+1}(C_{i}(J_{i}),C_{i+1})\right|^{2}\right]\mathbb{I}\left(\mathbb{E}_{J}\mathbb{E}_{C_{i+1}}\left[\left|w_{i+1}(C_{i}(J_{i}),C_{i+1})\right|^{2}\right]\geq B\right).
\end{align*}
Therefore, by Lemmas \ref{lem:MF_a_priori}, \ref{lem:a-priori-R_bold}
and Theorem \ref{thm:exist-R}:
\begin{align*}
 & \mathbf{E}\left[\left|\mathbb{E}_{J}\left[\left({\bf R}_{i}(t,J_{i-1},J_{i})-R_{i}(\tilde{G},t,C_{i-1}(J_{i-1}),C_{i}(J_{i}))\right)\frac{\partial\hat{y}(t,x)}{\partial w_{i}(C_{i-1}(J_{i-1}),C_{i}(J_{i}))}\right]\right|^{2}\right]\\
 & \le KB\mathbf{E}\mathbb{E}_{J}\left[\left({\bf R}_{i}(t,J_{i-1},J_{i})-R_{i}(\tilde{G},t,C_{i-1}(J_{i-1}),C_{i}(J_{i}))\right)^{2}\right]\\
 & \quad+K_{T}\mathbf{E}\left[\mathbb{I}\left(\mathbb{E}_{J}\mathbb{E}_{C_{i+1}}\left[\left|w_{i+1}(C_{i}(J_{i}),C_{i+1})\right|^{2}\right]\geq B^{2}\right)\right]^{1/4}\mathbf{E}\mathbb{E}_{J}\left[\left({\bf R}_{i}(t,J_{i-1},J_{i})\right)^{4}+\left(R_{i}(\tilde{G},t,C_{i-1}(J_{i-1}),C_{i}(J_{i}))\right)^{4}\right]^{1/2}\\
 & \le KB\mathbf{E}\mathbb{E}_{J}\left[\left({\bf R}_{i}(t,J_{i-1},J_{i})-R_{i}(\tilde{G},t,C_{i-1}(J_{i-1}),C_{i}(J_{i}))\right)^{2}\right]+e^{-K_{T}BN},
\end{align*}
for a suitable finite constant $B=B\left(T\right)$. The conclusion
again follows from Lemma \ref{lem:R-JC}.
\end{proof}
\begin{lem}
\label{lem:dy-dw-JC_NN}We have for $t\leq T$:
\begin{align*}
 & \sqrt{N}\left(\frac{\partial\hat{{\bf y}}(t,x)}{\partial{\bf w}(j_{i-1},j_{i})}-\frac{\partial\tilde{y}(t,x)}{\partial w(C_{i-1}(j_{i-1}),C_{i}(j_{i}))}\right)\\
 & =\mathbb{E}_{C}\bigg[\sum_{r=1}^{L}R_{r}\left(\tilde{G},t,C_{r-1},C_{r}\right)\frac{\partial^{2}\hat{y}\left(t,x\right)}{\partial w_{r}\left(C_{r-1},C_{r}\right)\partial w_{i}\left(C_{i-1}(j_{i-1}),C_{i}(j_{i})\right)}\bigg]\\
 & \quad+\mathbb{E}_{C}\bigg[R_{i-1}\left(\tilde{G},t,C_{i-2},C_{i-1}(j_{i-1})\right)\frac{\partial^{2}\hat{y}\left(t,x\right)}{\partial_{*}w_{i-1}\left(C_{i-2},C_{i-1}(j_{i-1})\right)\partial w_{i}\left(C_{i-1}(j_{i-1}),C_{i}(j_{i})\right)}\bigg]\\
 & \quad+\mathbb{E}_{C}\bigg[R_{i}\left(\tilde{G},t,C_{i-1},C_{i}(j_{i})\right)\frac{\partial^{2}\hat{y}\left(t,x\right)}{\partial_{*}w_{i}\left(C_{i-1},C_{i}(j_{i})\right)\partial w_{i}\left(C_{i-1}(j_{i-1}),C_{i}(j_{i})\right)}\bigg]\\
 & \quad+\mathbb{E}_{C}\bigg[R_{i+1}\left(\tilde{G},t,C_{i}(j_{i}),C_{i+1}\right)\frac{\partial^{2}\hat{y}\left(t,x\right)}{\partial_{*}w_{i+1}\left(C_{i}(j_{i}),C_{i+1}\right)\partial w_{i}\left(C_{i-1}(j_{i-1}),C_{i}(j_{i})\right)}\bigg]\\
 & \quad+K_{T}\mathbb{E}_{J}\left[\left({\bf R}_{i}(t,J_{i-1},J_{i})-R_{i}(\tilde{G},t,C_{i-1}(J_{i-1}),C_{i}(J_{i}))\right)^{2}\right]^{1/2}\\
 & \quad+{\bf O}_{T}(N^{-1/2}).
\end{align*}
\end{lem}

\begin{proof}
By Lemma \ref{lem:a-priori-NN_minus_MF}, we have:
\begin{align*}
 & {\bf E}\mathbb{E}_{J}\bigg[\bigg(\sqrt{N}\left(\frac{\partial\hat{{\bf y}}(t,x)}{\partial{\bf w}(J_{i-1},J_{i})}-\frac{\partial\tilde{y}(t,x)}{\partial w(C_{i-1}(J_{i-1}),C_{i}(J_{i}))}\right)\\
 & \quad-\sum_{j=1}^{L}\mathbb{E}_{J'}\bigg[{\bf R}_{j}(t,J_{j-1}',J_{j}')\frac{\partial^{2}\hat{y}(t,x)}{\partial w_{j}(C_{j-1}(J_{j-1}'),C_{j}(J_{j}'))\partial w_{i}(C_{i-1}(J_{i-1}),C_{i}(J_{i}))}\bigg]\\
 & \quad-\mathbb{E}_{J'}\bigg[\mathbf{R}_{i-1}\left(t,J_{i-2}',J_{i-1}\right)\frac{\partial^{2}\hat{y}\left(t,x\right)}{\partial_{*}w_{i-1}\left(C_{i-2}(J_{i-2}'),C_{i-1}(J_{i-1})\right)\partial w_{i}\left(C_{i-1}(J_{i-1}),C_{i}(J_{i})\right)}\bigg]\\
 & \quad-\mathbb{E}_{J'}\bigg[\mathbf{R}_{i}\left(t,J_{i-1}',J_{i}\right)\frac{\partial^{2}\hat{y}\left(t,x\right)}{\partial_{*}w_{i}\left(C_{i-1}(J_{i-1}'),C_{i}(J_{i})\right)\partial w_{i}\left(C_{i-1}(J_{i-1}),C_{i}(J_{i})\right)}\bigg]\\
 & \quad-\mathbb{E}_{J'}\bigg[\mathbf{R}_{i+1}\left(t,J_{i},J_{i+1}'\right)\frac{\partial^{2}\hat{y}\left(t,x\right)}{\partial_{*}w_{i+1}\left(C_{i}(J_{i}),C_{i+1}(J_{i+1}')\right)\partial w_{i}\left(C_{i-1}(J_{i-1}),C_{i}(J_{i})\right)}\bigg]\bigg)^{2}\bigg]\\
 & \le K_{T}N^{-1}.
\end{align*}
Following the argument in Lemmas \ref{lem:R-JC} and \ref{lem:R-JC_NN},
we obtain:
\begin{align*}
 & \mathbb{E}_{J'}\bigg[{\bf R}_{j}(t,J_{j-1}',J_{j}')\frac{\partial^{2}\hat{y}(t,x)}{\partial w_{j}(C_{j-1}(J_{j-1}'),C_{j}(J_{j}'))\partial w_{i}(C_{i-1}(J_{i-1}),C_{i}(J_{i}))}\bigg]\\
 & =\mathbb{E}_{C'}\bigg[R_{j}(\tilde{G},t,C_{j-1}',C_{j}')\frac{\partial^{2}\hat{y}(t,x)}{\partial w_{j}(C_{j-1}',C_{j}')\partial w_{i}(C_{i-1}(J_{i-1}),C_{i}(J_{i}))}\bigg]\\
 & \qquad+K_{T}\mathbb{E}_{J}\left[\left({\bf R}_{i}(t,J_{i-1},J_{i})-R_{i}(\tilde{G},t,C_{i-1}(J_{i-1}),C_{i}(J_{i}))\right)^{2}\right]^{1/2}+{\bf O}_{T}(N^{-1/2}),
\end{align*}
and similarly for the rest of the terms.
\end{proof}

\subsection{Proof of Theorem \ref{thm:2nd_order_MF}}
\begin{proof}[Proof of Theorem \ref{thm:2nd_order_MF}]
Recall that 
\begin{align*}
\partial_{t}{\bf R}_{i}(t,j_{i-1},j_{i}) & =\sqrt{N}\left(\partial_{t}\mathbf{w}_{i}(t,j_{i-1},j_{i})-\partial_{t}w_{i}(t,C_{i-1}(j_{i-1}),C_{i}(j_{i}))\right)\\
 & =-\mathbb{E}_{Z}\left[\sqrt{N}\left(\frac{\partial\hat{{\bf y}}(t,X)}{\partial{\bf w}_{i}(j_{i-1},j_{i})}-\frac{\partial\hat{y}(t,X)}{\partial w_{i}(C_{i-1}(j_{i-1}),C_{i}(j_{i}))}\right)\partial_{2}{\cal L}(Y,\hat{y}(t,X))\right]\\
 & \quad-\mathbb{E}_{Z}\left[\sqrt{N}\left(\partial_{2}{\cal L}(Y,\hat{{\bf y}}(t,X))-\partial_{2}{\cal L}(Y,\hat{y}(t,X))\right)\frac{\partial\hat{y}(t,X)}{\partial w_{i}(C_{i-1}(j_{i-1}),C_{i}(j_{i}))}\right]\\
 & \quad-\mathbb{E}_{Z}\left[\sqrt{N}\left(\partial_{2}{\cal L}(Y,\hat{{\bf y}}(t,X))-\partial_{2}{\cal L}(Y,\hat{y}(t,X))\right)\left(\frac{\partial\hat{{\bf y}}(t,X)}{\partial{\bf w}_{i}(j_{i-1},j_{i})}-\frac{\partial\hat{y}(t,X)}{\partial w_{i}(C_{i-1}(j_{i-1}),C_{i}(j_{i}))}\right)\right]\\
 & =-\mathbb{E}_{Z}\left[\sqrt{N}\left(\frac{\partial\hat{{\bf y}}(t,X)}{\partial{\bf w}_{i}(j_{i-1},j_{i})}-\frac{\partial\tilde{y}(t,X)}{\partial w_{i}(C_{i-1}(j_{i-1}),C_{i}(j_{i}))}\right)\partial_{2}{\cal L}(Y,\hat{y}(t,X))\right]\\
 & \quad-\mathbb{E}_{Z}\left[\tilde{G}_{i}^{w}\left(t,C_{i-1}(j_{i-1}),C_{i}(j_{i}),X\right)\partial_{2}{\cal L}(Y,\hat{y}(t,X))\right]\\
 & \quad-\mathbb{E}_{Z}\left[\sqrt{N}\left(\partial_{2}{\cal L}(Y,\hat{{\bf y}}(t,X))-\partial_{2}{\cal L}(Y,\hat{y}(t,X))\right)\frac{\partial\hat{y}(t,X)}{\partial w_{i}(C_{i-1}(j_{i-1}),C_{i}(j_{i}))}\right]\\
 & \quad-\mathbb{E}_{Z}\left[\sqrt{N}\left(\partial_{2}{\cal L}(Y,\hat{{\bf y}}(t,X))-\partial_{2}{\cal L}(Y,\hat{y}(t,X))\right)\left(\frac{\partial\hat{{\bf y}}(t,X)}{\partial{\bf w}_{i}(j_{i-1},j_{i})}-\frac{\partial\hat{y}(t,X)}{\partial w_{i}(C_{i-1}(j_{i-1}),C_{i}(j_{i}))}\right)\right].
\end{align*}
Comparing with Eq. (\ref{eq:ODE_2ndMF-alt}):
\begin{align*}
\partial_{t}{\bf R}_{i}(t,j_{i-1},j_{i})-\partial_{t}R_{i}(\tilde{G},t,C_{i-1}(j_{i-1}),C_{i}(j_{i})) & =-\mathbb{E}_{Z}\left[A_{1}\left(Z,j_{i-1},j_{i}\right)+A_{2}\left(Z,j_{i-1},j_{i}\right)+A_{3}\left(Z,j_{i-1},j_{i}\right)\right],
\end{align*}
in which
\begin{align*}
A_{1}\left(z,j_{i-1},j_{i}\right) & =\partial_{2}{\cal L}(Y,\hat{y}(t,x))\bigg\{\sqrt{N}\left(\frac{\partial\hat{{\bf y}}(t,x)}{\partial{\bf w}_{i}(j_{i-1},j_{i})}-\frac{\partial\tilde{y}(t,x)}{\partial w_{i}\left(C_{i-1}(j_{i-1}),C_{i}(j_{i})\right)}\right)\\
 & \quad-\mathbb{E}_{C}\bigg[\sum_{r=1}^{L}R_{r}\left(\tilde{G},t,C_{r-1},C_{r}\right)\frac{\partial^{2}\hat{y}\left(t,X\right)}{\partial w_{r}\left(C_{r-1},C_{r}\right)\partial w_{i}\left(C_{i-1}(j_{i-1}),C_{i}(j_{i})\right)}\bigg]\\
 & \quad-\mathbb{E}_{C}\bigg[R_{i-1}\left(\tilde{G},t,C_{i-2},C_{i-1}(j_{i-1})\right)\frac{\partial^{2}\hat{y}\left(t,X\right)}{\partial_{*}w_{i-1}\left(C_{i-2},C_{i-1}(j_{i-1})\right)\partial w_{i}\left(C_{i-1}(j_{i-1}),C_{i}(j_{i})\right)}\bigg]\\
 & \quad-\mathbb{E}_{C}\bigg[R_{i}\left(\tilde{G},t,C_{i-1},C_{i}(j_{i})\right)\frac{\partial^{2}\hat{y}\left(t,X\right)}{\partial_{*}w_{i}\left(C_{i-1},C_{i}(j_{i})\right)\partial w_{i}\left(C_{i-1}(j_{i-1}),C_{i}(j_{i})\right)}\bigg]\\
 & \quad-\mathbb{E}_{C}\bigg[R_{i+1}\left(\tilde{G},t,C_{i}(j_{i}),C_{i+1}\right)\frac{\partial^{2}\hat{y}\left(t,X\right)}{\partial_{*}w_{i+1}\left(C_{i}(j_{i}),C_{i+1}\right)\partial w_{i}\left(C_{i-1}(j_{i-1}),C_{i}(j_{i})\right)}\bigg]\bigg\},\\
A_{2}\left(z,j_{i-1},j_{i}\right) & =\frac{\partial\hat{y}(t,x)}{\partial w_{i}(C_{i-1}(j_{i-1}),C_{i}(j_{i}))}A_{2,1}\left(z\right),\\
A_{2,1}\left(z\right) & =\sqrt{N}\left(\partial_{2}{\cal L}(y,\hat{{\bf y}}(t,x))-\partial_{2}{\cal L}(y,\hat{y}(t,x))\right)-\left(\sum_{i=1}^{L}\mathbb{E}_{C}\left[R_{i}(\tilde{G},t,C_{i-1},C_{i})\frac{\partial\hat{y}(t,x)}{\partial w_{i}(C_{i-1},C_{i})}\right]+\tilde{G}^{y}\left(t,x\right)\right)\partial_{2}^{2}{\cal L}(y,\hat{y}(t,x))\\
A_{3}\left(z,j_{i-1},j_{i}\right) & =\sqrt{N}\left(\partial_{2}{\cal L}(y,\hat{{\bf y}}(t,x))-\partial_{2}{\cal L}(y,\hat{y}(t,x))\right)\left(\frac{\partial\hat{{\bf y}}(t,x)}{\partial{\bf w}_{i}(j_{i-1},j_{i})}-\frac{\partial\hat{y}(t,x)}{\partial w_{i}(C_{i-1}(j_{i-1}),C_{i}(j_{i}))}\right).
\end{align*}
To analyze $A_{2}$, we have from Lemmas \ref{lem:dL_NN_minus_MF},
\ref{lem:decomp-y} and \ref{lem:R-JC_NN}:
\[
A_{2,1}\left(z\right)=K_{T}\mathbb{E}_{J}\left[\left({\bf R}_{i}(t,J_{i-1},J_{i})-R_{i}(\tilde{G},t,C_{i-1}(J_{i-1}),C_{i}(J_{i}))\right)^{2}\right]^{1/2}+{\bf O}_{T}(N^{-1/2}).
\]
Furthermore, for any $B>1$, by Lemma \ref{lem:MF_a_priori}, Theorems
\ref{thm:clt-G} and \ref{thm:exist-R},
\begin{align*}
\mathbf{E}\mathbb{E}_{Z,J}\left[\left|A_{2}\left(Z,J_{i-1},J_{i}\right)\right|^{2}\right] & \leq K_{T}\mathbf{E}\mathbb{E}_{Z,J}\left[\left|A_{2}\left(Z\right)\right|^{2}\mathbb{E}_{C_{i+1}}\left[\left|w_{i+1}(C_{i}(J_{i}),C_{i+1})\right|^{2}\right]\right]\\
 & \leq K_{T}B^{2}\mathbf{E}\mathbb{E}_{Z}\left[\left|A_{2,1}\left(Z\right)\right|^{2}\right]+K_{T}\mathbf{E}\mathbb{E}_{Z}\left[\left|A_{2,1}\left(Z\right)\right|^{4}\right]^{1/2}\mathbf{E}\left[\mathbb{I}\left(\mathbb{E}_{J_{i}}\mathbb{E}_{C_{i+1}}\left[\left|w_{i+1}(C_{i}(J_{i}),C_{i+1})\right|^{2}\right]\geq B^{2}\right)\right]^{1/2}\\
 & \leq K_{T}B^{2}\mathbf{E}\mathbb{E}_{Z}\left[\left|A_{2,1}\left(Z\right)\right|^{2}\right]+e^{-K_{T}BN}.
\end{align*}
Hence with a suitable choice of $B$, 
\[
\mathbf{E}\mathbb{E}_{Z,J}\left[\left|A_{2}\left(Z,J_{i-1},J_{i}\right)\right|^{2}\right]\leq K_{T}\mathbb{E}_{J}\left[\left({\bf R}_{i}(t,J_{i-1},J_{i})-R_{i}(\tilde{G},t,C_{i-1}(J_{i-1}),C_{i}(J_{i}))\right)^{2}\right]+K_{T}N^{-1}.
\]
The treatment of $A_{1}$ is similar via Lemma \ref{lem:dy-dw-JC_NN}:
\[
\mathbf{E}\mathbb{E}_{Z,J}\left[\left|A_{1}\left(Z,J_{i-1},J_{i}\right)\right|^{2}\right]\leq K_{T}\mathbb{E}_{J}\left[\left({\bf R}_{i}(t,J_{i-1},J_{i})-R_{i}(\tilde{G},t,C_{i-1}(J_{i-1}),C_{i}(J_{i}))\right)^{2}\right]+K_{T}N^{-1}.
\]
Finally for $A_{3}$, by Lemmas \ref{lem:a-priori-NN_minus_MF},
\[
A_{3}=\boldsymbol{O}_{T}(N^{-1/2}).
\]
Therefore,
\begin{align*}
\partial_{t}{\bf E}\mathbb{E}_{J}\left[\left|{\bf R}_{i}(t,J_{i-1},J_{i})-R_{i}(\tilde{G},t,C_{i-1}(J_{i-1}),C_{i}(J_{i})\right|^{2}\right] & \leq{\bf E}\mathbb{E}_{J}\left[\left|\partial_{t}{\bf R}_{i}(t,J_{i-1},J_{i})-\partial_{t}R_{i}(\tilde{G},t,C_{i-1}(J_{i-1}),C_{i}(J_{i})\right|^{2}\right]\\
 & \leq K_{T}\mathbb{E}_{J}\left[\left({\bf R}_{i}(t,J_{i-1},J_{i})-R_{i}(\tilde{G},t,C_{i-1}(J_{i-1}),C_{i}(J_{i}))\right)^{2}\right]+K_{T}N^{-1},
\end{align*}
By Gronwall's inequality, we have 
\[
{\bf E}\mathbb{E}_{J}\left[\left({\bf R}_{i}(t,J_{i-1},J_{i})-R_{i}(\tilde{G},t,C_{i-1}(J_{i-1}),C_{i}(J_{i})\right)^{2}\right]\le\frac{K_{T}}{N}.
\]
This completes the proof.
\end{proof}

\section{CLT for the output fluctuation: Proof of Theorem \ref{thm:CLT}\label{Appendix:CLT_output}}

\subsection{Joint convergence in moment: Proof of Proposition \ref{prop:conv-G-implies-R}}
\begin{proof}[Proof of Proposition \ref{prop:conv-G-implies-R}]
In the following, let $\tilde{R}$ denote $R(\tilde{G},\cdot)$ and
-- with an abuse of notations -- $R$ denote $R(G,\cdot)$. By Theorem
\ref{thm:exist-R} --- and recalling the norms defined in its statement
--- there exists a sequence in $B$ of processes $\tilde{R}^{B}$
with $\tilde{R}^{B}(t)=\frak{L}_{t}^{B}(\tilde{G})$ where $\frak{L}_{t}$
is a linear map for all $t\le T$ with $\Vert\tilde{R}^{B}\Vert_{T,2}\leq\exp(K_{T}B)\|\tilde{G}\|_{T,2}$
and $\|\tilde{R}^{B}-\tilde{R}\|_{T,2}\le\|\tilde{G}\|_{T,2+\epsilon}\exp(-c\epsilon B^{2})$
for some $c>0$. We have: 
\begin{align*}
 & \mathbf{E}\left[\prod_{j}\langle f_{j},\tilde{G}_{i(j)}^{\alpha_{j}}\rangle_{t_{j}}^{\beta_{j}}\prod_{j'}\langle h_{j'},\tilde{R}_{i(j)}\rangle_{t_{j'}}^{\beta_{j'}}\right]\\
 & =\mathbf{E}\left[\prod_{j}\langle f_{j},\tilde{G}_{i(j)}^{\alpha_{j}}\rangle_{t_{j}}^{\beta_{j}}\prod_{j'}\langle h_{j'},\tilde{R}_{i(j)}^{B}+(\tilde{R}_{i(j')}^{B}-\tilde{R}_{i(j')})\rangle_{t_{j'}}^{\beta_{j'}}\right]\\
 & =\mathbf{E}\left[\prod_{j}\langle f_{j},\tilde{G}_{i(j)}^{\alpha_{j}}\rangle_{t_{j}}^{\beta_{j}}\prod_{j'}\left(\langle h_{j'},\tilde{R}_{i(j)}^{B}\rangle_{t_{j'}}+\|h_{j'}\|_{t_{j'}}\|\tilde{G}\|_{T,2+\epsilon}\exp(-c\epsilon B^{2})O_{j'}\left(1\right)\right)^{\beta_{j'}}\right]\\
 & =\mathbf{E}\left[\prod_{j}\langle f_{j},\tilde{G}_{i(j)}^{\alpha_{j}}\rangle_{t_{j}}^{\beta_{j}}\prod_{j'}\langle h_{j'},\tilde{R}_{i(j)}^{B}\rangle_{t_{j'}}^{\beta_{j'}}\right]+O_{D}(1)\exp(-c\epsilon B^{2})\prod_{j}\|f_{j}\|_{t_{j}}^{\beta_{j}}\|\tilde{G}_{i(j)}^{\alpha_{j}}\|_{t_{j}}^{\beta_{j}}\prod_{j'}\|h_{j'}\|_{t_{j'}}^{\beta_{j'}}\|\tilde{G}\|_{T,2+\epsilon}^{\beta_{j'}}\\
 & =\mathbf{E}\left[\prod_{j}\langle f_{j},\tilde{G}_{i(j)}^{\alpha_{j}}\rangle_{t_{j}}^{\beta_{j}}\prod_{j'}\langle(\frak{L}_{t}^{B})^{*}(h_{j'}),\tilde{G}_{i(j')}\rangle_{t_{j'}}^{\beta_{j'}}\right]+O_{D}(1)\exp(-c\epsilon B^{2})\prod_{j}\|f_{j}\|_{t_{j}}^{\beta_{j}}\|\tilde{G}_{i(j)}^{\alpha_{j}}\|_{t_{j}}^{\beta_{j}}\prod_{j'}\|h_{j'}\|_{t_{j'}}^{\beta_{j'}}\|\tilde{G}\|_{T,2+\epsilon}^{\beta_{j'}},
\end{align*}
where $(\frak{L}_{t}^{B})^{*}$ is the adjoint of $\frak{L}_{t}^{B}$,
which has operator norm at most $\exp(K_{T}B)$. A similar bound applies
to $\mathbf{E}\left[\prod_{j}\langle f_{j},G_{i(j)}^{\alpha_{j}}\rangle_{t_{j}}^{\beta_{j}}\prod_{j'}\langle h_{j'},R_{i(j')}\rangle_{t_{j'}}^{\beta_{j'}}\right]$.
Since $\tilde{G}$ converges $G$-polynomially in moment to $G$,
we have:
\begin{align*}
 & \left|\mathbf{E}\left[\prod_{j}\langle f_{j},\tilde{G}_{i(j)}^{\alpha_{j}}\rangle_{t_{j}}^{\beta_{j}}\prod_{j'}\langle h_{j'},\tilde{R}_{i(j)}^{B}\rangle_{t_{j'}}^{\beta_{j'}}\right]-\mathbf{E}\left[\prod_{j}\langle f_{j},G_{i(j)}^{\alpha_{j}}\rangle_{t_{j}}^{\beta_{j}}\prod_{j'}\langle h_{j'},R_{i(j)}^{B}\rangle_{t_{j'}}^{\beta_{j'}}\right]\right|\\
 & \leq N^{-1/8}O_{D}(\max_{j}\|f_{j}\|_{t_{j}}^{D},\max_{j'}\|(\frak{L}_{t}^{B})^{*}(h_{j'})\|_{t_{j'}}^{D})\\
 & \leq\exp(K_{T}B)N^{-1/8}O_{D}(\max_{j}\|f_{j}\|_{t_{j}}^{D},\max_{j'}\|h_{j'}\|_{t_{j'}}^{D}).
\end{align*}
This also shows 
\[
\|\tilde{G}_{i(j)}^{\alpha_{j}}\|_{t_{j}}^{\beta_{j}}\leq\|G_{i(j)}^{\alpha_{j}}\|_{t_{j}}^{\beta_{j}}+N^{-1/8}O_{D}(\max_{j}\|f_{j}\|_{t_{j}}^{D},\max_{j'}\|h_{j'}\|_{t_{j'}}^{D}).
\]
Then choosing $B=c_{0}\sqrt{\log N}$ for a suitable constant $c_{0}$,
for sufficiently large $N$, we have:
\begin{align*}
 & \left|\mathbf{E}\left[\prod_{j}\langle f_{j},\tilde{G}_{i(j)}^{\alpha_{j}}\rangle_{t_{j}}^{\beta_{j}}\prod_{j'}\langle h_{j'},\tilde{R}_{i(j)}\rangle_{t_{j'}}^{\beta_{j'}}\right]-\mathbf{E}\left[\prod_{j}\langle f_{j},G_{i(j)}^{\alpha_{j}}\rangle_{t_{j}}^{\beta_{j}}\prod_{j'}\langle h_{j'},R_{i(j)}\rangle_{t_{j'}}^{\beta_{j'}}\right]\right|\\
 & \le\exp(K_{T}B)N^{-1/8}O_{D}(\max_{j}\|f_{j}\|_{t_{j}}^{D},\max_{j'}\|h_{j'}\|_{t_{j'}}^{D})+O_{D}(1)\exp(-c\epsilon B^{2})\prod_{j}\|f_{j}\|_{t_{j}}^{\beta_{j}}\|G_{i(j)}^{\alpha_{j}}\|_{t_{j}}^{\beta_{j}}\prod_{j'}\|h_{j'}\|_{t_{j'}}^{\beta_{j'}}\|G\|_{T,2+\epsilon}^{\beta_{j'}}\\
 & \le N^{-1/8+o(1)}O_{D}(\max_{j}\|f_{j}\|_{t_{j}}^{D},\max_{j'}\|h_{j'}\|_{t_{j'}}^{D}).
\end{align*}
In particular, $(R(\tilde{G},\cdot),\tilde{G})$ converges $G$-polynomial-moment
and $R$-linear-moment to $(R(G,\cdot),G)$.
\end{proof}
\textsl{}

\subsection{CLT for the output function: Proof of Theorem \ref{thm:CLT}}
\begin{proof}[Proof of Theorem \ref{thm:CLT}]
By Lemma \ref{lem:decomp-y}, Lemma \ref{lem:R-JC_NN} and Theorem
\ref{thm:2nd_order_MF}, we have:
\[
\sqrt{N}(\hat{{\bf y}}(t,x)-\tilde{y}(t,x))=\sum_{i=1}^{L}\mathbb{E}_{C}\left[R_{i}(\tilde{G},t,C_{i-1},C_{i})\frac{\partial\hat{y}(t,x)}{\partial w_{i}(C_{i-1},C_{i})}\right]+{\bf O}_{T}(1/\sqrt{N}).
\]
This implies that 
\[
\underbrace{\sqrt{N}(\hat{{\bf y}}(t,x)-\hat{y}(t,x))}_{\equiv U(z)}=\underbrace{\sum_{i=1}^{L}\mathbb{E}_{C}\left[R_{i}(\tilde{G},t,C_{i-1},C_{i})\frac{\partial\hat{y}(t,x)}{\partial w_{i}(C_{i-1},C_{i})}\right]+\tilde{G}^{y}(t,x)}_{\equiv V(\tilde{G},z)}+{\bf O}_{T}(1/\sqrt{N}).
\]
For $U=U(Z)$ and $V=V(\tilde{G},Z)$, we have for $m\geq1$ and $1$-bounded
$h$,
\begin{align*}
\left|{\bf E}\left[\langle h,U^{m}\rangle_{Z}-\langle h,V^{m}\rangle_{Z}\right]\right| & \leq{\bf E}\mathbb{E}_{Z}\left[\left|U^{m}-V^{m}\right|\right]\\
 & \leq K_{m}{\bf E}\mathbb{E}_{Z}\left[\left(\left|U\right|^{m-1}+\left|V\right|^{m-1}\right)\left|U-V\right|\right]\\
 & \leq K_{m}{\bf E}\mathbb{E}_{Z}\left[\left|U\right|^{2m-2}+\left|V\right|^{2m-2}\right]^{1/2}{\bf E}\mathbb{E}_{Z}\left[\left|U-V\right|^{2}\right]^{1/2}.
\end{align*}
Recalling Lemmas \ref{lem:MF_a_priori}, \ref{lem:a-priori-NN_minus_MF},
Theorems \ref{thm:exist-R} and \ref{thm:clt-G}, we obtain:
\[
\left|{\bf E}\left[\langle h,U^{m}\rangle_{Z}-\langle h,V^{m}\rangle_{Z}\right]\right|\leq K_{T,m}N^{-1/2}.
\]
Finally using Proposition \ref{prop:conv-G-implies-R} and Theorem
\ref{thm:clt-G}, we obtain:
\[
\left|{\bf E}\left[\langle h,U^{m}\rangle_{Z}-\langle h,V(\underline{G},Z)^{m}\rangle_{Z}\right]\right|\leq K_{T,m}N^{-1/8+o(1)}.
\]
This also proves the desired weak convergence.
\end{proof}
\textsl{}

\section{Asymptotic variance of the output fluctuation: Proof of Theorems
\ref{thm:variance-global-opt-init} and \ref{thm:variance-global-opt-fast}\label{Appendix:variance}}

\subsection{Variance decomposition of the limiting output fluctuation $\underline{G}^{y}$}

The goal of this section is to show that $\underline{G}^{y}\left(t,\cdot\right)$
lies in the span of $\frac{\partial\hat{y}\left(t,\cdot\right)}{\partial w_{i}(C_{i-1},C_{i})}$.
This is done by decomposing the covariance structure of $\underline{G}^{y}$
in a suitable layer-wise fashion.
\begin{prop}
\label{prop:variance-decomp}Let $f$ be any square-integrable function
of $X$ for which 
\[
\sum_{i}\mathbb{E}_{C}\left[\mathbb{E}_{X}\left[\frac{\partial\hat{y}(t,X)}{\partial w_{i}(C_{i-1},C_{i})}f(X)\right]^{2}\right]=0.
\]
Then almost surely, 
\[
\mathbb{E}_{X}\left[\underline{G}^{y}(t,X)f(X)\right]^{2}=0.
\]
\end{prop}

\begin{proof}
Let us drop the notational dependency on $t$. Let us also write
\[
\varphi_{i}\left(H_{i}(x,c_{i})\right)\equiv\varphi_{i}\left(x,c_{i}\right),\qquad\varphi_{i}'\left(H_{i}(x,c_{i})\right)\equiv\varphi_{i}'\left(x,c_{i}\right).
\]
Note that 
\[
\mathbf{E}\left[\mathbb{E}_{X}\left[\underline{G}^{y}(X)f(X)\right]^{2}\right]=\mathbb{E}_{X,X'}\left[\mathbf{E}[\underline{G}^{y}(X)\underline{G}^{y}(X')]f(X)f(X')\right].
\]
We recall from the statement of Theorem \ref{thm:clt-G}:
\begin{align*}
\underline{G}^{y}(x) & =\varphi_{L}'(x)\underline{G}_{L}^{\tilde{H}}(x),\\
\underline{G}_{i}^{\tilde{H}}(c_{i},x) & =\underline{G}_{i}^{H}(c_{i},x)+\mathbb{E}_{C_{i-1}}\left[w_{i}(C_{i-1},c_{i})\varphi_{i-1}'(x,C_{i-1})\underline{G}_{i-1}^{\tilde{H}}(x,C_{i-1})\right],\\
\underline{G}_{1}^{\tilde{H}}(c_{1},x) & =0,
\end{align*}
which thus yield
\[
\underline{G}^{y}(x)=\sum_{i=1}^{L}\varphi'_{L}(x)\mathbb{E}_{C}\bigg[\bigg(\prod_{j=i+1}^{L}w_{j}(C_{j-1},C_{j})\varphi'_{j-1}(x,C_{j-1})\bigg)\underline{G}_{i}^{H}(x,C_{i})\bigg],
\]
where we take $\prod_{j=L+1}^{L}=1$. We also recall:
\[
\mathbf{E}[\underline{G}_{i}^{H}(x,c_{i})\underline{G}_{i}^{H}(x',c_{i}')]=\mathbb{C}_{C_{i-1}}\left[w_{i}(C_{i-1},c_{i})\varphi_{i-1}\left(x,C_{i-1}\right);\;w_{i}(C_{i-1},c_{i}')\varphi_{i-1}\left(x',C_{i-1}\right)\right],
\]
and for $i\ne j$, 
\[
\mathbf{E}[\underline{G}_{i}^{H}(x,c_{i})\underline{G}_{j}^{H}(x',c_{j}')]=0.
\]
Combining these facts together, we obtain:
\begin{align*}
 & \mathbf{E}[\underline{G}^{y}(x)\underline{G}^{y}(x')]\\
 & =\sum_{i=1}^{L}\varphi_{L}'(x)\varphi_{L}'(x')\mathbb{E}_{C,C'}\bigg[\prod_{j=i+1}^{L}w_{j}(C_{j-1},C_{j})\varphi'_{j-1}(x,C_{j-1})w_{j}(C_{j-1}',C_{j}')\varphi_{j-1}'(x,C_{j-1}')\mathbb{E}\left[\underline{G}_{i}^{H}(x,C_{i})\underline{G}_{i}^{H}(x,C_{i}')\right]\bigg]\\
 & =\sum_{i=1}^{L}\varphi_{L}'(x)\varphi_{L}'(x')\mathbb{E}_{C_{i},\dots,C_{L},C_{i}',....,C_{L}'}\bigg[\prod_{j=i+1}^{L}w_{j}(C_{j-1},C_{j})\varphi'_{j-1}(x,C_{j-1})w_{j}(C_{j-1}',C_{j}')\varphi_{j-1}'(x,C_{j-1}')\\
 & \qquad\times\Big(\mathbb{E}_{C_{i-1}}\left[w_{i}(C_{i-1},C_{i})w_{i}(C_{i-1},C_{i}')\varphi_{i-1}(x,C_{i-1})\varphi_{i-1}(x',C_{i-1})\right]\\
 & \qquad\qquad-\mathbb{E}_{C_{i-1}}\left[w_{i}(C_{i-1},C_{i})\varphi_{i-1}(x,C_{i-1})\right]\mathbb{E}_{C_{i-1}}\left[w_{i}(C_{i-1},C_{i}')\varphi_{i-1}(x',C_{i-1})\right]\Big)\bigg]\\
 & =\sum_{i=1}^{L}\bigg\{\mathbb{E}_{C_{i-1}}\bigg[\varphi_{L}'(x)\mathbb{E}_{C_{i},\dots,C_{L}}\bigg[\prod_{j=i+1}^{L}w_{j}(C_{j-1},C_{j})\varphi'_{j-1}(x,C_{j-1})\cdot w_{i}(C_{i-1},C_{i})\varphi_{i-1}(x,C_{i-1})\bigg]\\
 & \qquad\times\varphi_{L}'(x')\mathbb{E}_{C_{i}',\dots,C_{L}'}\bigg[\prod_{j=i+1}^{L}w_{j}(C_{j-1}',C_{j}')\varphi'_{j-1}(x',C_{j-1}')\cdot w_{i}(C_{i-1},C_{i}')\varphi_{i-1}(x',C_{i-1})\bigg]\bigg]\\
 & \quad-\mathbb{E}_{C_{i-1}}\bigg[\varphi_{L}'(x)\mathbb{E}_{C_{i},\dots,C_{L}}\bigg[\prod_{j=i+1}^{L}w_{j}(C_{j-1},C_{j})\varphi'_{j-1}(x,C_{j-1})\cdot w_{i}(C_{i-1},C_{i})\varphi_{i-1}(x,C_{i-1})\bigg]\bigg]\\
 & \qquad\times\mathbb{E}_{C_{i-1}'}\bigg[\varphi_{L}'(x')\mathbb{E}_{C_{i}',\dots,C_{L}'}\bigg[\prod_{j=i+1}^{L}w_{j}(C_{j-1}',C_{j}')\varphi'_{j-1}(x',C_{j-1}')\cdot w_{i}(C_{i-1}',C_{i}')\varphi_{i-1}(x',C_{i-1}')\bigg]\bigg]\bigg\}\\
 & =\sum_{i=1}^{L}\bigg\{\mathbb{E}_{C_{i-1}}\left[\mathbb{E}_{C_{i}}\left[w_{i}(C_{i-1},C_{i})\frac{\partial\hat{y}(x)}{\partial w_{i}(C_{i-1},C_{i})}\right]\mathbb{E}_{C_{i}'}\left[w_{i}(C_{i-1},C_{i}')\frac{\partial\hat{y}(x')}{\partial w_{i}(C_{i-1},C_{i}')}\right]\right]\\
 & \qquad-\mathbb{E}_{C_{i-1}}\left[\mathbb{E}_{C_{i}}\left[w_{i}(C_{i-1},C_{i})\frac{\partial\hat{y}(x)}{\partial w_{i}(C_{i-1},C_{i})}\right]\right]\mathbb{E}_{C_{i-1}'}\left[\mathbb{E}_{C_{i}'}\left[w_{i}(C_{i-1}',C_{i}')\frac{\partial\hat{y}(x')}{\partial w_{i}(C_{i-1}',C_{i}')}\right]\right]\bigg\}.
\end{align*}
In particular, 
\begin{align*}
 & \mathbb{E}_{X,X'}\left[\mathbf{E}[\underline{G}^{y}(X)\underline{G}^{y}(X')]f(X)f(X')\right]\\
 & =\sum_{i=1}^{L}\left\{ \mathbb{E}_{C_{i-1}}\left[\mathbb{E}_{X,C_{i}}\left[w_{i}(C_{i-1},C_{i})\frac{\partial\hat{y}(X)}{\partial w_{i}(C_{i-1},C_{i})}f(X)\right]^{2}\right]-\left(\mathbb{E}_{X,C_{i-1},C_{i}}\left[w_{i}(C_{i-1},C_{i})\frac{\partial\hat{y}(X)}{\partial w_{i}(C_{i-1},C_{i})}f(X)\right]\right)^{2}\right\} .
\end{align*}
Of course $\mathbb{E}_{X,X'}\left[\mathbf{E}[\underline{G}^{y}(X)\underline{G}^{y}(X')]f(X)f(X')\right]\geq0$,
but we also have:
\[
\mathbb{E}_{C_{i-1}}\left[\mathbb{E}_{X,C_{i}}\left[w_{i}(C_{i-1},C_{i})\frac{\partial\hat{y}(X)}{\partial w_{i}(C_{i-1},C_{i})}f(X)\right]^{2}\right]\le\mathbb{E}_{C_{i-1}}\left[\mathbb{E}_{C_{i}}\left[\left|w_{i}(C_{i-1},C_{i})\right|^{2}\right]\mathbb{E}_{C_{i}}\left[\mathbb{E}_{X}\left[\frac{\partial\hat{y}(X)}{\partial w_{i}(C_{i-1},C_{i})}f(X)\right]^{2}\right]\right]=0.
\]
Therefore it must be that
\[
\mathbf{E}\left[\mathbb{E}_{X}\left[\underline{G}^{y}(X)f(X)\right]^{2}\right]=\mathbb{E}_{X,X'}\left[\mathbf{E}[\underline{G}^{y}(X)\underline{G}^{y}(X')]f(X)f(X')\right]=0.
\]
\end{proof}

\subsection{Variance reduction under GD: Proof of Theorem \ref{thm:variance-global-opt-init}}

We prove the variance reduction effect of GD in the case of idealized
initialization.
\begin{proof}[Proof of Theorem \ref{thm:variance-global-opt-init}]
By Theorem \ref{thm:CLT}, the fluctuation $\sqrt{N}(\hat{{\bf y}}(t,x)-\hat{y}(t,x))$
converges in polynomial-moment to 
\[
\hat{G}(t,x):=\sum_{i=1}^{L}\mathbb{E}_{C}\left[R_{i}(\underline{G},t,C_{i-1},C_{i})\frac{\partial\hat{y}(t,x)}{\partial w_{i}(C_{i-1},C_{i})}\right]+\underline{G}^{y}(t,x),
\]
where $\underline{G}$ is the limiting Gaussian process defined in
Theorem \ref{thm:clt-G}. Furthermore, by Theorem \ref{thm:clt-partial_t-G},
\[
\partial_{t}\hat{G}(t,x)=\sum_{i=1}^{L}\mathbb{E}_{C}\left[\partial_{t}R_{i}(\underline{G},t,C_{i-1},C_{i})\frac{\partial\hat{y}(t,x)}{\partial w_{i}(C_{i-1},C_{i})}+R_{i}(\underline{G},t,C_{i-1},C_{i})\partial_{t}\left(\frac{\partial\hat{y}(t,x)}{\partial w_{i}(C_{i-1},C_{i})}\right)\right]+\partial_{t}\underline{G}^{y}(t,x).
\]
By the assumption on the initialization, $\partial_{t}\left(\frac{\partial\hat{y}(t,x)}{\partial w_{i}(C_{i-1},C_{i})}\right)=0$
and $\partial_{t}\underline{G}^{y}=0$. Then using Eq. (\ref{eq:ODE_2ndMF-alt})
for $\partial_{t}R_{i}(\underline{G},t,c_{i-1},c_{i})$ and recalling
that $\mathbb{E}_{Z}\left[\partial_{2}{\cal L}\left(Y,\hat{y}\left(0,X\right)\right)\middle|X\right]=0$,
we obtain:
\begin{align*}
\partial_{t}\hat{G}(t,x) & =-\sum_{i=1}^{L}\mathbb{E}_{Z,C}\bigg[\frac{\partial\hat{y}(t,x)}{\partial w_{i}(C_{i-1},C_{i})}\frac{\partial\hat{y}(t,X)}{\partial w_{i}(C_{i-1},C_{i})}\cdot\partial_{2}^{2}{\cal L}(Y,\hat{y}(t,X))\\
 & \qquad\times\bigg(\underline{G}^{y}(t,X)+\sum_{j=1}^{L}\mathbb{E}_{C_{j-1}',C_{j}'}\bigg[R_{j}(\underline{G},t,C_{j-1}',C_{j}')\frac{\partial\hat{y}(t,X)}{\partial w_{j}'(C_{j-1}',C_{j}')}\bigg]\bigg)\bigg]\\
 & =-\partial_{2}^{2}{\cal L}\sum_{i=1}^{L}\mathbb{E}_{Z,C}\left[\frac{\partial\hat{y}(t,x)}{\partial w_{i}(C_{i-1},C_{i})}\frac{\partial\hat{y}(t,X)}{\partial w_{i}(C_{i-1},C_{i})}\cdot\hat{G}(t,X)\right],
\end{align*}
where we denote $\partial_{2}^{2}{\cal L}=\mathbb{E}_{Z}\left[\partial_{2}^{2}{\cal L}\left(Y,\hat{y}\left(0,X\right)\right)\middle|X\right]>0$
a finite constant. Hence, 
\begin{align*}
\partial_{t}V^{*}\left(t\right) & =\partial_{t}\mathbf{E}\mathbb{E}_{X}\big[\big|\hat{G}(t,X)\big|^{2}\big]\\
 & =-2\partial_{2}^{2}{\cal L}\mathbf{E}\mathbb{E}_{Z'}\left[\hat{G}(t,X')\left(\sum_{i=1}^{L}\mathbb{E}_{Z,C}\left[\frac{\partial\hat{y}(t,X')}{\partial w_{i}(C_{i-1},C_{i})}\frac{\partial\hat{y}(t,X)}{\partial w_{i}(C_{i-1},C_{i})}\cdot\hat{G}(t,X)\right]\right)\right]\\
 & =-2\partial_{2}^{2}{\cal L}\sum_{i=1}^{L}\mathbf{E}\mathbb{E}_{C}\left[\mathbb{E}_{Z}\left[\frac{\partial\hat{y}(t,X)}{\partial w_{i}(C_{i-1},C_{i})}\hat{G}(t,X)\right]^{2}\right],
\end{align*}
which is non-negative. This completes the proof of the first part.

Next, since we initialize the MF limit at a stationary point, we have
\[
\hat{y}(t,x)=\hat{y}(x),\qquad\frac{\partial\hat{y}(t,x)}{\partial w_{i}(c_{i-1},c_{i})}=\frac{\partial\hat{y}(x)}{\partial w_{i}(c_{i-1},c_{i})},
\]
independent of time $t$. Let $\frak{A}:L^{2}({\cal P})\to L^{2}({\cal P})$
be the linear operator defined by 
\[
(\frak{A}f)(x)=-\sum_{i}\mathbb{E}_{Z,C}\left[\frac{\partial\hat{y}(x)}{\partial w_{i}(C_{i-1},C_{i})}\frac{\partial\hat{y}(X)}{\partial w_{i}(C_{i-1},C_{i})}f(X)\right].
\]
As shown -- and assuming $\partial_{2}^{2}{\cal L}=1$ for simplicity
-- we can write:
\[
\partial_{t}\hat{G}(t,x)=(\frak{A}\hat{G}\left(t,\cdot\right))(x).
\]
Note that $\frak{A}$ is self-adjoint: 
\[
\mathbb{E}\left[g(X)(\frak{A}f)(X)\right]=-\mathbb{E}_{Z,Z',C}\left[g(X)f(X')\sum_{i}\frac{\partial\hat{y}(X)}{\partial w_{i}(C_{i-1},C_{i})}\frac{\partial\hat{y}(X')}{\partial w_{i}(C_{i-1},C_{i})}\right]=\mathbb{E}\left[(\frak{A}g)(X)f(X)\right],
\]
and $\frak{A}$ is compact, as it is a finite sum of integral operators
with Hilbert-Schmidt kernels. By the spectral theorem, the orthogonal
complement of $\ker(\frak{A})$ has a countable orthonormal basis
consisting of eigenfunctions of $\frak{A}$. Note that $\ker(\frak{A})$
is a closed subspace of $L^{2}({\cal P})$ which is separable, and
hence it is also separable. Similarly, the orthogonal complement of
$\ker(\frak{A})$ is a separable subspace of $L^{2}({\cal P})$. For
any eigenfunction $\phi$ of $\frak{A}$ with eigenvalue $\lambda$,
we have 
\[
(\frak{A}\phi)(x)=-\sum_{i}\frac{\partial\hat{y}(x)}{\partial w_{i}(C_{i-1},C_{i})}\mathbb{E}_{Z,C}\left[\frac{\partial\hat{y}(X)}{\partial w_{i}(C_{i-1},C_{i})}\phi(X)\right]=\lambda\phi(x).
\]
Thus, 
\[
\lambda=\mathbb{E}[\phi(x)(\frak{A}\phi)(x)]=-\sum_{i}\mathbb{E}_{C}\left[\mathbb{E}_{X}\left[\frac{\partial\hat{y}(X)}{\partial w_{i}(C_{i-1},C_{i})}\phi(X)\right]^{2}\right].
\]
Hence, all eigenvalues of $\frak{A}$ are non-positive. Note that
$\frak{A}f=0$ if and only if 
\[
\sum_{i}\mathbb{E}_{C}\left[\mathbb{E}_{X}\left[\frac{\partial\hat{y}(X)}{\partial w_{i}(C_{i-1},C_{i})}f(X)\right]^{2}\right]=0.
\]
By Proposition \ref{prop:variance-decomp}, almost surely, $\underline{G}^{y}\in\ker(\frak{A})^{\perp}$
(noting that $\ker(\frak{A})$ is separable). Hence so is $\hat{G}(t,\cdot)$.
For $f\in\ker(\frak{A})^{\perp}$, the solution to the equation $\partial_{t}f_{t}(x)=(\frak{A}f_{t})(x)$
is given as
\[
f_{t}(x)=\exp(t\frak{A})f_{0}(x).
\]
Let $f_{0}=\sum_{j}\tau_{j}\phi_{j}$ where $\phi_{j}$ are eigenfunctions
of $\frak{A}$ in ${\cal H}$ with eigenvalue $\lambda_{j}<0$, we
have 
\[
f_{t}(x)=\sum_{j}\exp(t\lambda_{j})\tau_{j}\phi_{j}(x).
\]
In particular, since $\sum_{j}\tau_{j}^{2}=\left\Vert f_{0}\right\Vert _{L^{2}({\cal P})}^{2}$
is finite, as $t\to\infty$,
\[
\mathbb{E}_{X}\left[\left|f_{t}(X)\right|^{2}\right]=\sum_{j}\exp(2t\lambda_{j})\tau_{j}^{2}\to0.
\]
This proves the second part of the theorem.
\end{proof}

\subsection{Variance reduction under GD: Proof of Theorem \ref{thm:variance-global-opt-fast}}

We prove the variance reduction effect of GD in the case of sufficiently
fast convergence to global optima. In particular, we recall from Theorem
\ref{thm:variance-global-opt-fast} that we assume for some $\eta>0$,
\begin{equation}
\int_{0}^{\infty}s^{2+\eta}\mathbb{E}[\partial_{2}{\cal L}(Y,\hat{y}(s,X))^{2}]ds<\infty.\label{eq:assump-fast-conv}
\end{equation}
The dependency on time will be important in this section, therefore
we remind that $K_{T}$ denotes a constant depending on time, and
$K$ denotes a constant that does not.
\begin{lem}
\label{lem:unif-bound-w}If $\int_{0}^{\infty}\left(\mathbb{E}[\partial_{2}{\cal L}(Y,\hat{y}(t,X))^{2}]\right)^{1/2}dt<\infty$,
then for any $p\ge1$, $\sup_{t\geq0}\sum_{i}\mathbb{E}_{C}[w_{i}(t,C_{i-1},C_{i})^{2p}]\le(Kp)^{p}$.
\end{lem}

\begin{proof}
We have 
\[
\partial_{t}w_{i}(t,C_{i-1},C_{i})=\mathbb{E}_{Z}\left[\partial_{2}{\cal L}(Y,\hat{y}(t,X))\frac{\partial\hat{y}(t,X)}{\partial w_{i}(C_{i-1},C_{i})}\right],
\]
and hence
\begin{align*}
\mathbb{E}_{C}\left[\left|\partial_{t}w_{i}(t,C_{i-1},C_{i})\right|^{2p}\right] & \le\mathbb{E}_{C}\left[\mathbb{E}_{Z}\left[\partial_{2}{\cal L}(Y,\hat{y}(t,X))^{2}\right]^{p}\mathbb{E}_{Z}\left[\left(\frac{\partial\hat{y}(t,X)}{\partial w_{i}(C_{i-1},C_{i})}\right)^{2}\right]^{p}\right]\\
 & \le\mathbb{E}_{Z}\left[\partial_{2}{\cal L}(Y,\hat{y}(t,X))^{2}\right]^{p}\mathbb{E}_{C,Z}\left[\left(\frac{\partial\hat{y}(t,X)}{\partial w_{i}(C_{i-1},C_{i})}\right)^{2p}\right]\\
 & \le\mathbb{E}_{Z}\left[\partial_{2}{\cal L}(Y,\hat{y}(t,X))^{2}\right]^{p}\cdot K\mathbb{E}_{C_{i+1},C_{i}}\left[w_{i+1}(t,C_{i},C_{i+1})^{2p}\right]\prod_{j\ge i+2}\mathbb{E}_{C_{j},C_{j-1}}\left[w_{j}(t,C_{j-1},C_{j})^{2}\right]^{p}.
\end{align*}
Let us consider $i>1$; the case $i=1$ is similar. We then have:
\begin{align*}
\partial_{t}\left(\mathbb{E}[w_{i}(t,C_{i-1},C_{i})^{2p}]\right) & =2p\mathbb{E}[w_{i}^{2p-1}(t,C_{i-1},C_{i})\partial_{t}w_{i}(t,C_{i-1},C_{i})]\\
 & \le2pK\left(\mathbb{E}[w_{i}(t,C_{i-1},C_{i})^{2p}]\right)^{(2p-1)/(2p)}\cdot(\mathbb{E}[\partial_{t}w_{i}(t,C_{i-1},C_{i})^{2p}])^{1/(2p)}\\
 & \le2pK\left(\mathbb{E}[w_{i}(t,C_{i-1},C_{i})^{2p}]\right)^{(2p-1)/(2p)}\left(\mathbb{E}[\partial_{2}{\cal L}(Y,\hat{y}(t,X))^{2}]\right)^{1/2}\\
 & \qquad\times\mathbb{E}_{C_{i+1},C_{i}}\left[w_{i+1}(t,C_{i},C_{i+1})^{2p}\right]^{1/2p}\prod_{j\ge i+2}\mathbb{E}_{C_{j},C_{j-1}}\left[w_{j}(t,C_{j-1},C_{j})^{2}\right]^{1/2}.
\end{align*}
In particular, for $i=L$, 
\begin{align*}
\partial_{t}\left(\mathbb{E}[w_{L}(t,C_{L-1},C_{L})^{2p}]\right) & =2pK\left(\mathbb{E}[w_{L}(t,C_{L-1},C_{L})^{2p}]\right)^{(2p-1)/(2p)}\left(\mathbb{E}[\partial_{2}{\cal L}(Y,\hat{y}(t,X))^{2}]\right)^{1/2},
\end{align*}
which implies that 
\[
\mathbb{E}[w_{L}(t,C_{L-1},C_{L})^{2p}]^{1/(2p)}\le Kp^{1/2}+K\int_{0}^{t}\left(\mathbb{E}[\partial_{2}{\cal L}(Y,\hat{y}(s,X))^{2}]\right)^{1/2}ds,
\]
and therefore,
\[
\sup_{t\geq0}\mathbb{E}[w_{L}(t,C_{L-1},C_{L})^{2p}]^{1/(2p)}\le Kp^{1/2}+K.
\]
(Here we note that $\mathbb{E}[w_{L}(0,C_{L-1},C_{L})^{2p}]^{1/(2p)}\le Kp^{1/2}$
by the sub-Gaussian initialization assumption.)

Next, by induction, assuming that 
\[
\sup_{i>h}\sup_{t}\mathbb{E}[w_{i}(t,C_{i-1},C_{i})^{2p}]^{1/(2p)}<Kp^{1/2}+K,
\]
we have 
\begin{align*}
\partial_{t}\left(\mathbb{E}[w_{h}(t,C_{h-1},C_{h})^{2p}]\right) & \le2pK\left(\mathbb{E}[w_{h}(t,C_{h-1},C_{h})^{2p}]\right)^{(2p-1)/(2p)}\left(\mathbb{E}[\partial_{2}{\cal L}(Y,\hat{y}(t,X))^{2}]\right)^{1/2}\\
 & \qquad\times\mathbb{E}_{C_{h+1},C_{h}}\left[w_{h+1}(t,C_{h},C_{h+1})^{2p}\right]^{1/2p}\prod_{j\ge h+2}\mathbb{E}_{C_{j},C_{j-1}}\left[w_{j}(t,C_{j-1},C_{j})^{2p}\right]\\
 & \le2pK\left(\mathbb{E}[w_{h}(t,C_{h-1},C_{h})^{2p}]\right)^{(2p-1)/(2p)}\left(\mathbb{E}[\partial_{2}{\cal L}(Y,\hat{y}(t,X))^{2}]\right)^{1/2}\mathbb{E}_{C_{h+1},C_{h}}\left[w_{h+1}(t,C_{h},C_{h+1})^{2p}\right]^{1/2p}.
\end{align*}
Thus, as above, 
\[
\mathbb{E}[w_{h}(t,C_{h-1},C_{h})^{2p}]^{1/(2p)}\le Kp^{1/2}+(Kp^{1/2}+K)\int_{0}^{t}\left(\mathbb{E}[\partial_{2}{\cal L}(Y,\hat{y}(s,X))^{2}]\right)^{1/2}ds,
\]
so we get
\[
\sup_{t\geq0}\mathbb{E}[w_{h}(t,C_{h-1},C_{h})^{2p}]^{1/(2p)}\le Kp^{1/2}+K.
\]
This proves the claim.
\end{proof}
Under the assumption (\ref{eq:assump-fast-conv}), the MF limit $w_{i}(t,\cdot,\cdot)$
converges in $L^{2}$ to a limit $\overline{w}_{i}$, as shown in
the corollary below. 
\begin{cor}
\label{cor:conv-w}Under the setting of Lemma \ref{lem:unif-bound-w},
we have 
\begin{align*}
\mathbb{E}\left[\left|w_{t}(t,C_{i-1},C_{i})-\overline{w}_{i}(C_{i-1},C_{i})\right|^{2k}\right] & \le K_{\delta}^{k}\int_{t}^{\infty}s^{(2k-1)(1+\delta)}\mathbb{E}_{Z}\left[\partial_{2}{\cal L}(Y,\hat{y}(s,X))^{2}\right]^{k}ds.
\end{align*}
\end{cor}

\begin{proof}
We have from the previous lemma that 
\[
\mathbb{E}\left[\left|\partial_{t}w_{i}(t,C_{i-1},C_{i})\right|\right]\le\mathbb{E}\left[\left|\partial_{t}w_{i}(t,C_{i-1},C_{i})\right|^{2p}\right]^{1/(2p)}\le K_{p}\mathbb{E}_{Z}\left[\partial_{2}{\cal L}(Y,\hat{y}(t,X))^{2}\right]^{1/2}.
\]
By Holder's inequality, 
\begin{align*}
\left|w_{t}(t,C_{i-1},C_{i})-\overline{w}_{i}(C_{i-1},C_{i})\right|^{2k} & \le\left(\int_{t}^{\infty}s^{-1-\delta}ds\right)^{2k-1}\left(\int_{t}^{\infty}s^{(2k-1)(1+\delta)}|\partial_{t}w_{i}(s,C_{i-1},C_{i})|^{2k}ds\right)\\
 & \le K_{\delta}^{k}\left(\int_{t}^{\infty}s^{(2k-1)(1+\delta)}|\partial_{t}w_{i}(s,C_{i-1},C_{i})|^{2k}ds\right).
\end{align*}
Thus, 
\begin{align*}
\mathbb{E}\left[\left|w_{t}(t,C_{i-1},C_{i})-\overline{w}_{i}(C_{i-1},C_{i})\right|^{2k}\right] & \le K_{\delta}^{k}\int_{t}^{\infty}s^{(2k-1)(1+\delta)}\mathbb{E}_{C_{i-1},C_{i}}\left[|\partial_{t}w_{i}(s,C_{i-1},C_{i})|^{2k}\right]ds\\
 & \le K_{\delta}^{k}\int_{t}^{\infty}s^{(2k-1)(1+\delta)}\mathbb{E}_{Z}\left[\partial_{2}{\cal L}(Y,\hat{y}(s,X))^{2}\right]^{k}ds.
\end{align*}
\end{proof}
This suggests that one should have convergence to a limit as $t\to\infty$
at the fluctuation level. Should we have that, we would obtain Theorem
\ref{thm:variance-global-opt-fast} as a consequence of Theorem \ref{thm:variance-global-opt-init}.
The caveat is that the convergence must satisfy certain \textit{uniform-in-time}
properties. This is the bulk of the work.

Now we let $\overline{G},\overline{w},\overline{R}$ be obtained by
plugging the infinite-time limit $\overline{w}=\left\{ \bar{w}_{i}\right\} $
into $w=\left\{ w_{i}\left(t,\cdot,\cdot\right)\right\} $. Recall
that $\underline{G}$ is the limiting Gaussian process defined in
Theorem \ref{thm:clt-G}. We then have:
\begin{align*}
\partial_{t}R_{i}(\underline{G},t,c_{i-1},c_{i}) & =-\mathbb{E}_{Z}\left[\frac{\partial\hat{y}(t,X)}{\partial w_{i}(c_{i-1},c_{i})}\cdot\partial_{2}^{2}{\cal L}(Y,\hat{y}(t,X))\cdot\left(\underline{G}^{y}(t,X)+\sum_{j=1}^{L}\mathbb{E}_{C_{j-1}',C_{j}'}\left\{ R_{j}(\underline{G},t,C_{j-1}',C_{j}')\frac{\partial\hat{y}(t,X)}{\partial w_{j}(C_{j-1}',C_{j}')}\right\} \right)\right]\\
 & \quad-\mathbb{E}_{Z}\left[\partial_{2}{\cal L}(Y,\hat{y}(t,X))\cdot\left(\underline{G}_{i}^{w}(t,c_{i-1},c_{i},X)+\sum_{j=1}^{L}\mathbb{E}_{C_{j-1}',C_{j}'}\left\{ R_{j}(\underline{G},t,C_{j-1}',C_{j}')\frac{\partial^{2}\hat{y}(t,X)}{\partial w_{j}(C_{j-1}',C_{j}')\partial w_{i}(c_{i-1},c_{i})}\right\} \right)\right]\\
 & \quad-\mathbb{E}_{Z}\left[\partial_{2}{\cal L}(Y,\hat{y}(t,X))\cdot\mathbb{E}_{C_{i-1}'}\left\{ R_{i}(\underline{G},t,C_{i-1}',c_{i})\frac{\partial^{2}\hat{y}(t,X)}{\partial_{*}w_{i}(C_{i-1}',c_{i})\partial w_{i}(c_{i-1},c_{i})}\right\} \right]\\
 & \quad-\mathbb{E}_{Z}\left[\partial_{2}{\cal L}(Y,\hat{y}(t,X))\cdot\mathbb{E}_{C_{i+1}'}\left\{ R_{i}(\underline{G},t,c_{i},C_{i+1}')\frac{\partial^{2}\hat{y}(t,X)}{\partial_{*}w_{i+1}(c_{i},C_{i+1}')\partial w_{i}(c_{i-1},c_{i})}\right\} \right]\\
 & \quad-\mathbb{E}_{Z}\left[\partial_{2}{\cal L}(Y,\hat{y}(t,X))\cdot\mathbb{E}_{C_{i-2}'}\left\{ R_{i}(\underline{G},t,C_{i-2}',c_{i-1})\frac{\partial^{2}\hat{y}(t,X)}{\partial_{*}w_{i-1}(C_{i-2}',c_{i-1})\partial w_{i}(c_{i-1},c_{i})}\right\} \right],\\
\partial_{t}\overline{R}_{i}(\overline{G},t,c_{i-1},c_{i}) & =-\mathbb{E}_{Z}\left[\frac{\partial\hat{y}(X)}{\partial\overline{w}_{i}(c_{i-1},c_{i})}\cdot\partial_{2}^{2}{\cal L}\cdot\left(\overline{G}^{y}(t,X)+\sum_{j=1}^{L}\mathbb{E}_{C_{j-1}',C_{j}'}\left\{ \overline{R}_{j}(\overline{G},t,C_{j-1}',C_{j}')\frac{\partial\hat{y}(X)}{\partial\overline{w}_{j}(C_{j-1}',C_{j}')}\right\} \right)\right],
\end{align*}
where $\partial_{2}^{2}{\cal L}>0$ denotes the positive constant
at global optima as in Assumption \ref{Assump:Assumption_2}. Whenever
the context is clear, we write $R$ for $R(\underline{G},\cdot)$.
We also recall from Theorem \ref{thm:CLT} that the limiting output
fluctuation $\hat{G}$ is described via $R(\underline{G},\cdot)$.

Let us rewrite these dynamics in the following form:
\[
\partial_{t}R=\frak{A}_{t}R+H_{t},\qquad\partial_{t}\overline{R}=\overline{\frak{A}}\overline{R}+\overline{H},
\]
which implies

\begin{align*}
\partial_{t}(R-\overline{R}) & =(\frak{A}_{t}-\overline{\frak{A}})R+(H_{t}-\overline{H})+\overline{\frak{A}}(R-\overline{R}),\\
(R-\overline{R})(s) & =\int_{s}^{\infty}\exp(-t\overline{\frak{A}})((\frak{A}_{t}-\overline{\frak{A}})R+(H_{t}-\overline{H}))dt,
\end{align*}
for linear operators $\frak{A}_{t}$ and $\overline{\frak{A}}$ defined
by 
\begin{align*}
\frak{A}_{t}S(c_{i-1},c_{i}) & =-\mathbb{E}_{Z}\left[\frac{\partial\hat{y}(t,X)}{\partial w_{i}(c_{i-1},c_{i})}\cdot\partial_{2}{\cal L}(Y,\hat{y}(t,X))\cdot\left(\sum_{j=1}^{L}\mathbb{E}_{C_{j-1}',C_{j}'}\left\{ S_{j}(C_{j-1}',C_{j}')\frac{\partial\hat{y}(t,X)}{\partial w_{j}'(C_{j-1}',C_{j}')}\right\} \right)\right]\\
 & \quad-\mathbb{E}_{Z}\left[\partial_{2}{\cal L}(Y,\hat{y}(t,X))\cdot\left(\sum_{j=1}^{L}\mathbb{E}_{C_{j-1}',C_{j}'}\left\{ S_{j}(C_{j-1}',C_{j}')\frac{\partial^{2}\hat{y}(t,X)}{\partial w_{j}(C_{j-1}',C_{j}')\partial w_{i}(c_{i-1},c_{i})}\right\} \right)\right]\\
 & \quad-\mathbb{E}_{Z}\left[\partial_{2}{\cal L}(Y,\hat{y}(t,X))\cdot\mathbb{E}_{C_{i-1}'}\left\{ S_{i}(C_{i-1}',c_{i})\frac{\partial^{2}\hat{y}(t,X)}{\partial_{*}w_{i}(C_{i-1}',c_{i})\partial w_{i}(c_{i-1},c_{i})}\right\} \right]\\
 & \quad-\mathbb{E}_{Z}\left[\partial_{2}{\cal L}(Y,\hat{y}(t,X))\cdot\mathbb{E}_{C_{i+1}'}\left\{ S_{i}(c_{i},C_{i+1}')\frac{\partial^{2}\hat{y}(t,X)}{\partial_{*}w_{i+1}(c_{i},C_{i+1}')\partial w_{i}(c_{i-1},c_{i})}\right\} \right]\\
 & \quad-\mathbb{E}_{Z}\left[\partial_{2}{\cal L}(Y,\hat{y}(t,X))\cdot\mathbb{E}_{C_{i-2}'}\left\{ S_{i}(C_{i-2}',c_{i-1})\frac{\partial^{2}\hat{y}(t,X)}{\partial_{*}w_{i-1}(C_{i-2}',c_{i-1})\partial w_{i}(c_{i-1},c_{i})}\right\} \right],\\
\overline{\frak{A}}S(c_{i-1},c_{i}) & =-\mathbb{E}_{Z}\left[\frac{\partial\hat{y}(X)}{\partial\overline{w}_{i}(c_{i-1},c_{i})}\cdot\partial_{2}^{2}{\cal L}\cdot\left(\sum_{j=1}^{L}\mathbb{E}_{C_{j-1}',C_{j}'}\left\{ S(C_{j-1}',C_{j}')\frac{\partial\hat{y}(X)}{\partial\overline{w}_{j}'(C_{j-1}',C_{j}')}\right\} \right)\right],\\
H_{t}(c_{i-1},c_{i}) & =-\mathbb{E}_{Z}\left[\frac{\partial\hat{y}(t,X)}{\partial w_{i}(c_{i-1},c_{i})}\cdot\partial_{2}^{2}{\cal L}(Y,\hat{y}(t,X))\cdot\underline{G}^{y}(t,X)\right]-\mathbb{E}_{Z}\left[\partial_{2}{\cal L}(Y,\hat{y}(t,X))\cdot\underline{G}_{i}^{w}(t,c_{i-1},c_{i},X)\right],\\
\overline{H}(c_{i-1},c_{i}) & =-\mathbb{E}_{Z}\left[\frac{\partial\hat{y}(X)}{\partial\overline{w}_{i}(c_{i-1},c_{i})}\cdot\partial_{2}^{2}{\cal L}\cdot\overline{G}^{y}(X)\right].
\end{align*}

\begin{lem}
\label{lem:bound-H}For any $\delta>0$ and $k\ge1$, we have the
estimate 
\begin{align*}
{\bf E}\|H_{t}-\overline{H}\|_{2k}^{2k} & :={\bf E}\bigg\{\sum_{i=1}^{L}\mathbb{E}_{C}\left[\left(\mathbb{E}_{Z}\left[\frac{\partial\hat{y}(t,X)}{\partial w_{i}(C_{i-1},C_{i})}\cdot\partial_{2}^{2}{\cal L}(Y,\hat{y}(t,X))\cdot\underline{G}^{y}(t,X)\right]-\mathbb{E}_{Z}\left[\frac{\partial\hat{y}(X)}{\partial\overline{w}_{i}(C_{i-1},C_{i})}\cdot\partial_{2}^{2}{\cal L}\cdot\overline{G}^{y}(X)\right]\right)^{2k}\right]\\
 & \qquad+\sum_{i=1}^{L}\mathbb{E}_{C}\left[\left(\mathbb{E}_{Z}\left[\partial_{2}{\cal L}(Y,\hat{y}(t,X))\cdot\underline{G}_{i}^{w}(t,c_{i-1},c_{i},X)\right]\right)^{2k}\right]\bigg\}\\
 & \le K_{k,\delta}\int_{t}^{\infty}s^{(2k-1)(1+\delta)}\mathbb{E}_{Z}\left[\partial_{2}{\cal L}(Y,\hat{y}(s,X))^{2}\right]^{k}ds.
\end{align*}
\end{lem}

\begin{proof}
We have 
\begin{align*}
 & \mathbb{E}_{C_{i-1},C_{i}}\left[\left(\mathbb{E}_{Z}\left[\frac{\partial\hat{y}(t,X)}{\partial w_{i}(C_{i-1},C_{i})}\cdot\partial_{2}^{2}{\cal L}(Y,\hat{y}(t,X))\cdot\underline{G}^{y}(t,X)\right]-\mathbb{E}_{Z}\left[\frac{\partial\hat{y}(X)}{\partial\overline{w}_{i}(C_{i-1},C_{i})}\cdot\partial_{2}^{2}{\cal L}\cdot\overline{G}^{y}(X)\right]\right)^{2k}\right]\\
 & \le K_{k}\mathbb{E}_{C_{i-1},C_{i}}\Bigg[\mathbb{E}_{Z}\left[\left(\frac{\partial\hat{y}(t,X)}{\partial w_{i}(C_{i-1},C_{i})}\cdot\partial_{2}^{2}{\cal L}(Y,\hat{y}(t,X))-\frac{\partial\hat{y}(X)}{\partial\overline{w}_{i}(C_{i-1},C_{i})}\cdot\partial_{2}^{2}{\cal L}\right)\cdot\underline{G}^{y}(t,X)\right]^{2k}\\
 & \qquad\qquad\qquad+\mathbb{E}_{Z}\left[\frac{\partial\hat{y}(X)}{\partial\overline{w}_{i}(C_{i-1},C_{i})}\cdot\partial_{2}^{2}{\cal L}\cdot\left(\underline{G}^{y}(t,X)-\overline{G}^{y}(X)\right)\right]^{2k}\Bigg]\\
 & \le K_{k}\mathbb{E}_{Z}\left[\underline{G}^{y}(t,X)^{2}\right]^{k}\mathbb{E}_{C_{i-1},C_{i}}\left[\mathbb{E}_{Z}\left[\left(\frac{\partial\hat{y}(t,X)}{\partial w_{i}(C_{i-1},C_{i})}\cdot\partial_{2}^{2}{\cal L}(Y,\hat{y}(t,X))-\frac{\partial\hat{y}(X)}{\partial\overline{w}_{i}(C_{i-1},C_{i})}\cdot\partial_{2}^{2}{\cal L}\right)^{2}\right]^{k}\right]\\
 & \qquad+K_{k}\mathbb{E}_{Z}\left[\left(\underline{G}^{y}(t,X)-\overline{G}^{y}(X)\right)^{2}\right].
\end{align*}
Note that $\sup_{t}\mathbb{E}_{Z}\left[\underline{G}^{y}(t,X)^{2}\right]^{k}<\infty$
since $W(t)$ is uniformly bounded in $L^{2p}$ for any $p\ge1$ by
Lemma \ref{lem:unif-bound-w}. Furthermore, 
\begin{align*}
 & \mathbb{E}_{C_{i-1},C_{i}}\left[\mathbb{E}_{Z}\left[\left(\frac{\partial\hat{y}(t,X)}{\partial w_{i}(C_{i-1},C_{i})}\cdot\partial_{2}^{2}{\cal L}(Y,\hat{y}(t,X))-\frac{\partial\hat{y}(X)}{\partial\overline{w}_{i}(C_{i-1},C_{i})}\cdot\partial_{2}^{2}{\cal L}\right)^{2}\right]^{k}\right]\\
 & \le K_{k}\mathbb{E}_{C_{i-1},C_{i},Z}\left[\left(\frac{\partial\hat{y}(t,X)}{\partial w_{i}(C_{i-1},C_{i})}-\frac{\partial\hat{y}(X)}{\partial\overline{w}_{i}(C_{i-1},C_{i})}\right)^{2k}\right]+\mathbb{E}_{C_{i-1},C_{i}}\left[\mathbb{E}_{Z}\left[\left(\frac{\partial\hat{y}(t,X)}{\partial w_{i}(C_{i-1},C_{i})}\cdot\left(\partial_{2}^{2}{\cal L}(Y,\hat{y}(t,X))-\partial_{2}^{2}{\cal L}\right)\right)^{2}\right]^{k}\right]\\
 & \le K_{k}\mathbb{E}_{C_{i-1},C_{i},Z}\left[\left(\frac{\partial\hat{y}(t,X)}{\partial w_{i}(C_{i-1},C_{i})}-\frac{\partial\hat{y}(X)}{\partial\overline{w}_{i}(C_{i-1},C_{i})}\right)^{2k}\right]+K_{k}\mathbb{E}_{Z}\left[\left(\partial_{2}^{2}{\cal L}(Y,\hat{y}(t,X))-\partial_{2}^{2}{\cal L}\right)^{2k}\right]\\
 & \le K_{k}\mathbb{E}_{C_{i-1},C_{i},Z}\left[\left(\frac{\partial\hat{y}(t,X)}{\partial w_{i}(C_{i-1},C_{i})}-\frac{\partial\hat{y}(X)}{\partial\overline{w}_{i}(C_{i-1},C_{i})}\right)^{2k}\right]+K_{k}\mathbb{E}_{Z}\left[\left(\hat{y}(t,X)-\overline{\hat{y}}(X)\right)^{2k}\right].
\end{align*}
Using uniform boundedness of $W(t)$ in $t$, we have 
\[
\mathbb{E}_{C_{i-1},C_{i},Z}\left[\left(\frac{\partial\hat{y}(t,X)}{\partial w_{i}(C_{i-1},C_{i})}-\frac{\partial\hat{y}(X)}{\partial\overline{w}_{i}(C_{i-1},C_{i})}\right)^{2k}\right]\le K_{k}\sum_{j}\mathbb{E}_{C_{j-1},C_{j}}\left[(w_{j}(t,C_{j-1},C_{j})-\overline{w}_{j}(C_{j-1},C_{j}))^{2k}\right],
\]
and 
\[
\mathbb{E}_{Z}\left[\left(\hat{y}(t,X)-\overline{\hat{y}}(X)\right)^{2k}\right]\le K_{k}\sum_{j}\mathbb{E}_{C_{j-1},C_{j}}\left[(w_{j}(t,C_{j-1},C_{j})-\overline{w}_{j}(C_{j-1},C_{j}))^{2k}\right].
\]
By Lemma \ref{cor:conv-w}, 
\begin{align*}
 & \mathbb{E}_{C_{i-1},C_{i}}\left[\mathbb{E}_{Z}\left[\left(\frac{\partial\hat{y}(t,X)}{\partial w_{i}(C_{i-1},C_{i})}\cdot\partial_{2}^{2}{\cal L}(Y,\hat{y}(t,X))-\frac{\partial\hat{y}(X)}{\partial\overline{w}_{i}(C_{i-1},C_{i})}\cdot\partial_{2}^{2}{\cal L}\right)^{2}\right]^{k}\right]\\
 & \le K_{\delta}\int_{t}^{\infty}s^{(2k-1)(1+\delta)}\mathbb{E}_{Z}\left[\partial_{2}{\cal L}(Y,\hat{y}(s,X))^{2}\right]^{k}ds.
\end{align*}

Next, we have 
\begin{align*}
{\bf E}\mathbb{E}_{Z}\left[\left(\underline{G}^{y}(t,X)-\overline{G}^{y}(X)\right)^{2}\right]^{k} & \le K_{k,\delta}\int_{t}^{\infty}s^{(2k-1)(1+\delta)}{\bf E}\mathbb{E}_{Z}\left[(\partial_{t}\underline{G}^{y}(s,X))^{2k}\right]ds.
\end{align*}
Using the covariance structure of $\underline{G}^{\partial_{t}H}$
and $\underline{G}^{\partial_{t}(\partial H)}$, we similarly deduce
that 
\[
{\bf E}\mathbb{E}_{Z}\left[(\partial_{t}\underline{G}^{y}(t,X))^{2k}\right]\le K_{k}\mathbb{E}_{Z}\left[\partial_{2}{\cal L}(Y,\hat{y}(t,X))^{2}\right]^{k}.
\]
Furthermore, 
\[
\sum_{i=1}^{L}{\bf E}\mathbb{E}_{C_{i-1},C_{i}}\left[\left(\mathbb{E}_{Z}\left[\partial_{2}{\cal L}(Y,\hat{y}(t,X))\cdot\underline{G}_{i}^{w}(t,c_{i-1},c_{i},X)\right]\right)^{2k}\right]\le K_{k}\mathbb{E}_{Z}\left[\partial_{2}{\cal L}(Y,\hat{y}(t,X))^{2}\right]^{k},
\]
by Lemma \ref{lem:unif-bound-w}. Hence, 
\[
\mathbf{E}\left[\|H_{t}-\overline{H}\|_{2k}^{2k}\right]\le K_{k,\delta}\int_{t}^{\infty}s^{(2k-1)(1+\delta)}\mathbb{E}_{Z}\left[\partial_{2}{\cal L}(Y,\hat{y}(s,X))^{2}\right]^{k}ds.
\]
\end{proof}
Note that under Assumption (\ref{eq:assump-fast-conv}), $\mathbb{E}_{Z}\left[\partial_{2}{\cal L}(Y,\hat{y}(s,X))^{2}\right]\to0$
as $s\to\infty$, and hence 
\[
\int_{0}^{\infty}s^{(2k-1)(1+\delta)}\mathbb{E}_{Z}\left[\partial_{2}{\cal L}(Y,\hat{y}(s,X))^{2}\right]^{k}ds\le K_{k}+\int_{0}^{\infty}s^{(2k-1)(1+\delta)}\mathbb{E}_{Z}\left[\partial_{2}{\cal L}(Y,\hat{y}(s,X))^{2}\right]ds.
\]
Thus, under the assumption, we have that for $k=1+\epsilon$ with
$\epsilon$ sufficiently small, and taking $\delta$ sufficiently
small, we conclude that ${\bf E}\|H_{t}-\overline{H}\|_{2k}^{2k}\to0$
as $t\to\infty$ for all $k=k(\eta)$ sufficiently close to $1$.
\begin{lem}
\label{lem:unif-bound-Rbar}We have $\sup_{t}\|\overline{R}(t)\|_{2}^{2}/\|\overline{H}\|_{2}^{2}<\infty$.
\end{lem}

\begin{proof}
We have that $\overline{R}$ is the solution to 
\[
\partial_{t}\overline{R}=\overline{\frak{A}}\overline{R}+\overline{H}.
\]
In the proof of Theorem \ref{thm:variance-global-opt-init}, we observe
that $\overline{\frak{A}}$ is self-adjoint and has nonpositive eigenvalues.
Hence, 
\[
\overline{R}/\|\overline{H}\|_{2}=\exp(t\overline{\frak{A}})\overline{H}/\|\overline{H}\|_{2}
\]
is bounded uniformly in time.
\end{proof}
\begin{lem}
\label{lem:bound-AR}We have the following estimate for $\delta$
sufficiently small (depending on $\eta$): 
\[
{\bf E}\|(\frak{A}_{t}-\overline{\frak{A}})R\|_{2}^{2}\le K_{\delta}\int_{t}^{\infty}s^{1+\delta}\mathbb{E}_{Z}\left[\partial_{2}{\cal L}(Y,\hat{y}(s,X))^{2}\right]ds+K_{\delta}\mathbb{E}_{Z}\left[\partial_{2}{\cal L}(Y,\hat{y}(t,X))^{2}\right].
\]
\end{lem}

\begin{proof}
We first follow the proof of Theorem \ref{thm:exist-R} to give uniform-in-time
estimates for $R(t)$ under the assumption (\ref{eq:assump-fast-conv}).
For a threshold $B$, define the linear operator 
\begin{align*}
\frak{A}_{t}^{B}S(c_{i-1},c_{i}) & =-\mathbb{E}_{Z}\left[\frac{\partial\hat{y}(t,X)}{\partial w_{i}(c_{i-1},c_{i})}\cdot\partial_{2}^{2}{\cal L}(Y,\hat{y}(t,X))\cdot\left(\sum_{j=1}^{L}\mathbb{E}_{C_{j-1}',C_{j}'}\left\{ S_{j}(C_{j-1}',C_{j}')\frac{\partial\hat{y}(t,X)}{\partial w_{j}(C_{j-1}',C_{j}')}\right\} \right)\right]\\
 & \quad+\tilde{\frak{A}}_{t}^{B}S(c_{i-1},c_{i}),\\
\tilde{\frak{A}}_{t}^{B}S(c_{i-1},c_{i}) & =-\Bigg\{\mathbb{E}_{Z}\left[\partial_{2}{\cal L}(Y,\hat{y}(t,X))\cdot\left(\sum_{j=1}^{L}\mathbb{E}_{C_{j-1}',C_{j}'}\left\{ S_{j}(C_{j-1}',C_{j}')\frac{\partial^{2}\hat{y}(t,X)}{\partial w_{j}(C_{j-1}',C_{j}')\partial w_{i}(c_{i-1},c_{i})}\right\} \right)\right]\\
 & \quad-\mathbb{E}_{Z}\left[\partial_{2}{\cal L}(Y,\hat{y}(t,X))\cdot\mathbb{E}_{C_{i-1}'}\left\{ S_{i}(C_{i-1}',c_{i})\frac{\partial^{2}\hat{y}(t,X)}{\partial_{*}w_{i}(C_{i-1}',c_{i})\partial w_{i}(c_{i-1},c_{i})}\right\} \right]\\
 & \quad-\mathbb{E}_{Z}\left[\partial_{2}{\cal L}(Y,\hat{y}(t,X))\cdot\mathbb{E}_{C_{i+1}'}\left\{ S_{i}(c_{i},C_{i+1}')\frac{\partial^{2}\hat{y}(t,X)}{\partial_{*}w_{i+1}(c_{i},C_{i+1}')\partial w_{i}(c_{i-1},c_{i})}\right\} \right]\\
 & \quad-\mathbb{E}_{Z}\left[\partial_{2}{\cal L}(Y,\hat{y}(t,X))\cdot\mathbb{E}_{C_{i-2}'}\left\{ S_{i}(C_{i-2}',c_{i-1})\frac{\partial^{2}\hat{y}(t,X)}{\partial_{*}w_{i-1}(C_{i-2}',c_{i-1})\partial w_{i}(c_{i-1},c_{i})}\right\} \right]\Bigg\}\mathbb{B}^{B}(t,c_{i-1},c_{i}),
\end{align*}
and 
\begin{align*}
H_{t}(c_{i-1},c_{i}) & =-\mathbb{E}_{Z}\left[\frac{\partial\hat{y}(t,X)}{\partial w_{i}(c_{i-1},c_{i})}\cdot\partial_{2}^{2}{\cal L}(Y,\hat{y}(t,X))\cdot\underline{G}^{y}(t,X)\right]-\mathbb{E}_{Z}\left[\partial_{2}{\cal L}(Y,\hat{y}(t,X))\cdot\underline{G}_{i}^{w}(t,c_{i-1},c_{i},X)\right],\\
\partial_{t}R^{B}(\underline{G},t,\cdot,\cdot) & =\frak{A}_{t}^{B}(R^{B}(\underline{G},t))+H(\underline{G},t).
\end{align*}
Following Theorem \ref{thm:exist-R}, we have the bound 
\begin{align*}
\left|\tilde{\frak{A}}_{t}^{B}(R^{B}(\underline{G},t))(c_{i-1},c_{i})\right| & \le K\Big(\|R^{B}(t)\|\|w_{i}(t,c_{i})\|\|w_{i+1}(t,c_{i})\|+\|R_{i}^{B}(t,c_{i})\|\|w_{i+1}(t,c_{i})\|\\
 & \qquad+\|R_{i+1}^{B}(t,c_{i})\|+\|R_{i}^{B}(t,c_{i-1})\|\|w_{i}(t,c_{i})\|\Big)\cdot\|\partial_{2}{\cal L}(Y,\hat{y}(t,X))\|\mathbb{B}^{B}(t,c_{i-1},c_{i}).
\end{align*}
Thus 
\[
\mathbb{E}_{C_{i-1},C_{i}}\left[\left|\tilde{\frak{A}}_{t}^{B}(R^{B}(\underline{G},t))(C_{i-1},C_{i})\right|^{2}\right]\le K\|R^{B}(t)\|^{2}(1+B^{2})\|\partial_{2}{\cal L}(Y,\hat{y}(t,X))\|^{2},
\]
where we have used uniform-in-time boundedness of $W(t)$ from Lemma
\ref{lem:unif-bound-w}. We then have 
\begin{align*}
 & \partial_{t}\|R^{B}(t)\|_{2}^{2}\\
 & \le-\sum_{i=1}^{L}\mathbb{E}_{Z,C_{i-1},C_{i}}\left[R_{i}^{B}(\underline{G},t,C_{i-1},C_{i})\frac{\partial\hat{y}(t,X)}{\partial w_{i}(C_{i-1},C_{i})}\cdot\partial_{2}^{2}{\cal L}(Y,y)\cdot\left(\sum_{j=1}^{L}\mathbb{E}_{C_{j-1}',C_{j}'}\left\{ R_{j}^{B}(\underline{G},t,C_{j-1}',C_{j}')\frac{\partial\hat{y}(t,X)}{\partial w_{j}'(C_{j-1}',C_{j}')}\right\} \right)\right]\\
 & \quad+\|R^{B}(t)\|\cdot K\|R^{B}(t)\|(1+B)\|\partial_{2}{\cal L}(Y,\hat{y}(t,X))\|+\|H(\underline{G},t)\|^{2}\\
 & \le-\partial_{2}^{2}{\cal L}\left(\sum_{j=1}^{L}\mathbb{E}_{C_{j-1}',C_{j}'}\left\{ R_{j}^{B}(\underline{G},t,C_{j-1}',C_{j}')\frac{\partial\hat{y}(t,X)}{\partial w_{j}'(C_{j-1}',C_{j}')}\right\} \right)^{2}+K\|w(t)\|_{K}^{K_{L}}\|\partial_{2}^{2}{\cal L}(Y,\hat{y}(t,X))-\partial_{2}^{2}{\cal L}\|\|R^{B}(t)\|_{2}^{2}\\
 & \quad+K\|R^{B}(t)\|^{2}(1+B)\|\partial_{2}{\cal L}(Y,\hat{y}(t,X))\|+\|H(\underline{G},t)\|^{2}.
\end{align*}
Thus, 
\begin{align*}
\partial_{t}\|R^{B}(\underline{G},t)\|_{2}^{2} & \le K(1+B)\|R^{B}(t)\|^{2}\|\partial_{2}{\cal L}(Y,\hat{y}(t,X))\|+K\|w(t)\|_{K}^{K_{L}}\|\partial_{2}^{2}{\cal L}(Y,\hat{y}(t,X))-\partial_{2}^{2}{\cal L}\|_{2}^{2}+\|H(\underline{G},t)\|^{2}.
\end{align*}
Using that 
\[
\int_{0}^{\infty}\|\partial_{2}^{2}{\cal L}(Y,\hat{y}(t,X))-\partial_{2}^{2}{\cal L}\|dt<\infty,
\]
\[
{\bf E}\int_{0}^{\infty}\|H(\underline{G},t)\|^{2}dt<\infty,
\]
\[
\int_{0}^{\infty}\|\partial_{2}{\cal L}(Y,\hat{y}(t,X))\|dt<\infty,
\]
as well as Lemma \ref{lem:unif-bound-w}, we can then conclude 
\[
\sup_{t}{\bf E}\|R^{B}(\underline{G},t)\|_{2}^{2}\le\exp(K(1+B)).
\]

Next, we compare $R^{B}$ and $R^{B'}$ as in Theorem \ref{thm:exist-R}
for $B'>B$. We have the bound, recalling the definition of $\mathbb{D}_{B,B'}$
in Theorem \ref{thm:exist-R}, 
\begin{align*}
 & \partial_{t}{\bf E}\|(R^{B}-R^{B'})(t)\|_{2}^{2}\\
 & \le-\sum_{i=1}^{L}{\bf E}\mathbb{E}_{Z,C_{i-1},C_{i}}\Bigg[(R_{i}^{B}-R_{i}^{B'})(\underline{G},t,C_{i-1},C_{i})\frac{\partial\hat{y}(t,X)}{\partial w_{i}(C_{i-1},C_{i})}\cdot\partial_{2}^{2}{\cal L}(Y,y)\\
 & \qquad\qquad\times\left(\sum_{j=1}^{L}\mathbb{E}_{C_{j-1}',C_{j}'}\left\{ (R_{j}^{B}-R_{j}^{B'})(\underline{G},t,C_{j-1}',C_{j}')\frac{\partial\hat{y}(t,X)}{\partial w_{j}'(C_{j-1}',C_{j}')}\right\} \right)\Bigg]\\
 & \quad+K{\bf E}\|R^{B}(t)\|^{2}(1+B)\|\partial_{2}{\cal L}(Y,\hat{y}(t,X))\|\\
 & \quad+K_{\delta}\sum_{i=1}^{L}{\bf E}\mathbb{E}_{C_{i-1},C_{i}}\left[\mathbb{D}_{B,B'}(t,C_{i-1},C_{i})\exp(K(1+B'))\int_{0}^{t}(\|\underline{G}_{i+1}(s,C_{i})\|^{2}+\|\underline{G}_{i}(s,C_{i})\|^{2})ds\right]\|\partial_{2}{\cal L}(Y,\hat{y}(t,X))\|\\
 & \le-\partial_{2}^{2}{\cal L}{\bf E}\left[\left(\sum_{j=1}^{L}\mathbb{E}_{C_{j-1}',C_{j}'}\left\{ (R_{j}^{B}-R_{j}^{B'})(\underline{G},t,C_{j-1}',C_{j}')\frac{\partial\hat{y}(t,X)}{\partial w_{j}'(C_{j-1}',C_{j}')}\right\} \right)^{2}\right]\\
 & \quad+K\|w(t)\|_{K}^{K_{L}}\|\partial_{2}^{2}{\cal L}(Y,\hat{y}(t,X))-\partial_{2}^{2}{\cal L}\|_{2}{\bf E}\|R^{B}-R^{B'}\|^{2}\\
 & \quad+K{\bf E}\|(R^{B}-R^{B'})(t)\|^{2}(1+B)\|\partial_{2}{\cal L}(Y,\hat{y}(t,X))\|+\exp(K(1+B')-cB^{2})\|\partial_{2}{\cal L}(Y,\hat{y}(t,X))\|.
\end{align*}
Again using our convergence assumption, we can conclude that 
\[
\sup_{t}{\bf E}\|(R^{B}-R^{B'})(t)\|^{2}\le\exp(K(1+B))\exp(K(1+B')-cB^{2}).
\]
In particular, $R^{B}$ converges (uniformly in time $t$) in $L^{2}$
to the limit $R$, whose norm is uniformly bounded in time.

An identical argument shows that upon having finite moment $\|H_{t}\|^{2k+\delta}$,
we have that $R^{B}$ converges (uniformly in time) in $L^{2k}$ to
a limit $R$, whose $L^{2k}$ norm is uniformly bounded in time. This
is guaranteed for all $k$ sufficiently close to $1$ (in terms of
$\eta$) by the remark following Lemma \ref{lem:bound-H}. 

From Holder's Inequality, we have the following estimate for sufficiently
small $\delta$:
\[
\|\tilde{\frak{A}}_{t}R(t)\|_{2}^{2}\le K_{\delta}\|R(t)\|_{2+\delta}^{2}\|\partial_{2}{\cal L}(Y,\hat{y}(t,X))\|^{2},
\]
and 
\[
\|((\frak{A}_{t}-\tilde{A}_{t})-\overline{\frak{A}})R(t)\|_{2}^{2}\le K_{\delta}\|R(t)\|_{2}^{2}\int_{t}^{\infty}s^{1+\delta}\mathbb{E}_{Z}\left[\partial_{2}{\cal L}(Y,\hat{y}(s,X))^{2}\right]ds.
\]
Thus, 
\[
{\bf E}\|(\frak{A}_{t}-\overline{\frak{A}})R(t)\|_{2}^{2}\le K_{\delta}\int_{t}^{\infty}s^{1+\delta}\mathbb{E}_{Z}\left[\partial_{2}{\cal L}(Y,\hat{y}(s,X))^{2}\right]ds+\mathbb{E}_{Z}\left[\partial_{2}{\cal L}(Y,\hat{y}(t,X))^{2}\right].
\]
\end{proof}
\begin{lem}
We have for $\delta$ sufficiently small in $\eta$ that 
\[
{\bf E}\|R(t)-\overline{R}\|_{2}^{2}\le K_{\delta}\left(\int_{t}^{\infty}(s^{2+\delta}-t^{2+\delta}+1)\mathbb{E}_{Z}\left[\partial_{2}{\cal L}(Y,\hat{y}(s,X))^{2}\right]ds\right).
\]
\end{lem}

\begin{proof}
We have 
\[
\partial_{t}{\bf E}\|R(t)-\overline{R}\|_{2}^{2}=2{\bf E}\langle R(t)-\overline{R},(\frak{A}_{t}-\overline{\frak{A}})R(t)+(H_{t}-\overline{H})+\overline{\frak{A}}(R(t)-\overline{R})\rangle.
\]
We now bound 
\[
{\bf E}\langle R(t)-\overline{R},H_{t}-\overline{H}\rangle\le K_{\delta}\left(\int_{t}^{\infty}s^{1+\delta}\mathbb{E}_{Z}\left[\partial_{2}{\cal L}(Y,\hat{y}(s,X))^{2}\right]ds\right)^{1/2}\cdot({\bf E}\|R(t)-\overline{R}\|_{2}^{2})^{1/2},
\]
We also have 
\[
\langle R(t)-\overline{R},\overline{\frak{A}}(R(t)-\overline{R})\rangle\le0.
\]
From Lemma \ref{lem:bound-AR},

\begin{align*}
{\bf E}\left[\langle R(t)-\overline{R},(\frak{A}_{t}-\overline{\frak{A}})\overline{R}\rangle\right] & \le K{\bf E}\left[\|R(t)-\overline{R}\|_{2}^{2}\right]^{1/2}{\bf E}\left[\|(\frak{A}_{t}-\overline{\frak{A}})\overline{R}\|_{2}^{2}\right]^{1/2}\\
 & \le K{\bf E}\left[\|R(t)-\overline{R}\|_{2}^{2}\right]^{1/2}\left(K_{\delta}\int_{t}^{\infty}s^{1+\delta}\mathbb{E}_{Z}\left[\partial_{2}{\cal L}(Y,\hat{y}(s,X))^{2}\right]ds+K_{\delta}\mathbb{E}_{Z}\left[\partial_{2}{\cal L}(Y,\hat{y}(t,X))^{2}\right]\right)^{1/2}.
\end{align*}
Thus, 
\[
\partial_{t}{\bf E}\|R(t)-\overline{R}\|_{2}^{2}\le K_{\delta}\left(\mathbb{E}_{Z}\left[\partial_{2}{\cal L}(Y,\hat{y}(t,X))^{2}\right]+\int_{t}^{\infty}s^{1+\delta}\mathbb{E}_{Z}\left[\partial_{2}{\cal L}(Y,\hat{y}(s,X))^{2}\right]ds\right)^{1/2}({\bf E}\|R(t)-\overline{R}\|_{2}^{2})^{1/2}.
\]
In particular, 
\[
\partial_{t}({\bf E}\|R(t)-\overline{R}\|_{2}^{2})^{1/2}\le K_{\delta}\left(\mathbb{E}_{Z}\left[\partial_{2}{\cal L}(Y,\hat{y}(t,X))^{2}\right]+\int_{t}^{\infty}s^{1+\delta}\mathbb{E}_{Z}\left[\partial_{2}{\cal L}(Y,\hat{y}(s,X))^{2}\right]ds\right).
\]
Hence, 
\[
{\bf E}\|R(t)-\overline{R}\|_{2}^{2}\le K_{\delta}\left(\int_{t}^{\infty}(s^{2+\delta}-t^{2+\delta}+1)\mathbb{E}\left[\partial_{2}{\cal L}(Y,\hat{y}(s,X))^{2}\right]ds\right).
\]
\end{proof}
We can now finish the proof of Theorem \ref{thm:variance-global-opt-fast}.
\begin{proof}[Proof of Theorem \ref{thm:variance-global-opt-fast}]
We have from Theorem \ref{thm:CLT}:
\[
\lim_{N\to\infty}{\bf E}\mathbb{E}_{X}\left[\left(\sqrt{N}(\hat{{\bf y}}(t,X)-\hat{y}(t,X))\right)^{2}\right]={\bf E}\mathbb{E}_{Z}\left[\left(\sum_{i=1}^{L}\mathbb{E}_{C}\left[R_{i}(\underline{G},t,C_{i-1},C_{i})\frac{\partial\hat{y}(t,X)}{\partial w_{i}(C_{i-1},C_{i})}\right]+\underline{G}^{y}(t,X)\right)^{2}\right].
\]
For $\delta>0$ sufficiently small in $\eta$, 
\begin{align*}
 & \Bigg|{\bf E}\mathbb{E}_{Z}\left[\left(\sum_{i=1}^{L}\mathbb{E}_{C}\left[R_{i}(\underline{G},t,C_{i-1},C_{i})\frac{\partial\hat{y}(t,X)}{\partial w_{i}(C_{i-1},C_{i})}\right]+\underline{G}^{y}(t,X)\right)^{2}\right]\\
 & \qquad-{\bf E}\mathbb{E}_{Z}\left[\left(\sum_{i=1}^{L}\mathbb{E}_{C}\left[\overline{R}_{i}(\overline{G},C_{i-1},C_{i})\frac{\partial\overline{\hat{y}}(X)}{\partial w_{i}(C_{i-1},C_{i})}\right]+\overline{G}^{y}(X)\right)^{2}\right]\Bigg|\\
 & \le K{\bf E}\|R(t)-\overline{R}\|_{2}^{2}+K\|w(t)-\overline{w}\|_{2}^{2}\\
 & \le K_{\delta}\left(\int_{t}^{\infty}(s^{2+\delta}-t^{2+\delta}+1)\mathbb{E}_{Z}\left[\partial_{2}{\cal L}(Y,\hat{y}(s,X))^{2}\right]ds\right).
\end{align*}
The conclusion immediately follows from the assumption 
\[
\int_{0}^{\infty}s^{2+\eta}\mathbb{E}[\partial_{2}{\cal L}(Y,\hat{y}(s,X))^{2}]ds<\infty,
\]
together with Theorem \ref{thm:variance-global-opt-init}.
\end{proof}

\section{Experimental details\label{Appendix:experiments}}

We give a full description of the experimental setup in Section \ref{sec:Numerical-illustration}.

\paragraph*{Overall setup.}

We consider a training set of 100 randomly chosen MNIST images of
digits 0, 4, 5 and 9, where we encode the label $y=+1$ if the image
is digit 5 or 9 and $y=-1$ otherwise. This small scale allows us
to run full-batch GD in reasonable time and avoid stochastic algorithms
which create extra fluctuations that are not the focus of our study,
as mentioned in Section \ref{sec:Introduction}.

We use a neural network of 3 layers and $\tanh$ activations in the
hidden layers (i.e. $\varphi_{1}=\varphi_{2}=\tanh$ and $\varphi_{3}={\rm identity}$).
We include the bias in the first layer's weight, and there is no bias
elsewhere. The network has constant widths $N$ which we vary in the
plots of Fig. \ref{fig:illustration}. We train it with full-batch
(discrete-time) GD, a Huber loss ${\cal L}(y,y')={\rm Huber}(y-y')$
and a constant learning rate $7\times10^{-3}$. This learning rate
is sufficiently small so that we obtain smooth evolution plots in
Fig. \ref{fig:illustration}(a), suggesting a behavior close to the
continuous-time GD. We also define 
\[
{\rm Huber}(u)=\begin{cases}
u^{2}/2, & |u|\leq1,\\
|u|-1/2, & {\rm otherwise}.
\end{cases}
\]
All the training is run with one NVIDIA Tesla P100 GPU.

\paragraph*{Initialization.}

In Fig. \ref{fig:illustration}(a)-(c), we initialize the network
randomly as follows. We generate $\mathbf{w}_{1}(0,\cdot,\cdot)\sim\mathcal{N}(0,1/785)$,
$\mathbf{w}_{2}(0,\cdot,\cdot)\sim\mathcal{N}(0.1,1)$ and $\mathbf{w}_{3}(0,\cdot,\cdot)\sim\mathcal{N}(0.1,1)$
all independently. In Fig. \ref{fig:illustration}(e), we first train
a big network which has width $N^{*}=5000$ and has the same configuration
as described for Fig. \ref{fig:illustration}(a)-(c). We stop the
training of this width-$N^{*}$ network at iteration $10^{5}$, which
ensures that it has found a global optimum. Then we consider a network
which has width $N\in\{50,100,200,400\}$ and is initialized as follows:
\begin{itemize}
\item Let $\mathbf{w}_{i}^{*}(T_{e},\cdot,\cdot)$ be the $i$-th layer's
weight of the width-$N^{*}$ network at $T_{e}$ corresponding to
iteration $10^{5}$.
\item For each $i=1,2$, we draw at random $\ell_{i,1},...,\ell_{i,N}$
from $[N^{*}]$ independently. Let $\ell_{0,\cdot}=\ell_{3,\cdot}=1$.
\item We set the initial weight $\mathbf{w}_{i}(0,j_{i-1},j_{i})=\mathbf{w}_{i}^{*}(T_{e},\ell_{i-1,j_{i-1}},\ell_{i,j_{i}})$
for the width-$N$ network.
\end{itemize}
In other words, the width-$N$ network at initialization is a result
of subsampling the neurons of the width-$N^{*}$ network at time $T_{e}$.
The width-$N^{*}$ network is essentially an approximation of the
MF limit. Ideally when $N^{*}\to\infty$ then $N\to\infty$, following
\cite{nguyen2020rigorous}, we have that the width-$N$ network is
initialized at a global optimum. Here this holds approximately at
finite widths, but is sufficient for our purpose.

\paragraph*{Estimation for the plots.}

For Fig. \ref{fig:illustration}(a), we run once for each $N$ and
plot the training loss over time.

For Fig. \ref{fig:illustration}(b), we draw at random a single test
image $x_{{\rm test}}$ (which is not among the 100 training images
and happens to be a digit 5 in this plot). For each $N$, we train
the network till time $T_{b}$, which corresponds to iteration $10^{3}$,
and we repeat $M=1000$ times with $M$ different randomization seeds.
Let $\hat{\mathbf{y}}_{N}^{k}(T_{b},x_{{\rm test}})$ be the prediction
of $x_{{\rm test}}$ of the $k$-th repeat, for $k\in[M]$. Let 
\[
\hat{\mathbf{r}}_{N}^{k}(T_{b},x_{{\rm test}})=\sqrt{N}\bigg(\hat{\mathbf{y}}_{N}^{k}(T_{b},x_{{\rm test}})-\frac{1}{M}\sum_{k'=1}^{M}\hat{\mathbf{y}}_{N}^{k'}(T_{b},x_{{\rm test}})\bigg).
\]
Then we plot the histogram of $\big\{\hat{\mathbf{r}}_{N}^{k}(T_{b},x_{{\rm test}})\big\}_{k\in[M]}$.

For Fig. \ref{fig:illustration}(c), we follow a procedure similar
to that of Fig. \ref{fig:illustration}(b), but with $M=10$. For
each $N$ and each $x$ from the training image set, we obtain $\hat{\mathbf{y}}_{N}^{k}(t,x)$
for $k\in[M]$. Then we compute:
\[
v_{N}(t)=\frac{1}{100M}\sum_{k=1}^{M}\sum_{x\in{\rm training\,set}}|\hat{\mathbf{r}}_{N}^{k}(t,x)|^{2}.
\]
We take $v_{N}(t)$ as an estimate for the output variance $\mathbf{E}\mathbb{E}_{Z}\big[\big|\sqrt{N}(\hat{{\bf y}}(t,X)-\hat{y}(t,X))\big|^{2}\big]$.

For Fig. \ref{fig:illustration}(e), we follow the same procedure
as Fig. \ref{fig:illustration}(c), except that we train a single
width-$N^{*}$ network that is used for all $M$ repeats and all $N\in\{50,100,200,400\}$.

For Fig. \ref{fig:illustration}(d), we recall the initialization
procedure for Fig. \ref{fig:illustration}(e). Here instead of using
one value of the terminal time $T_{e}$, we shall let it vary. Following
this procedure, for each $T_{e}$ and each $N$, we obtain an untrained
width-$N$ network. For each $T_{e}$, we repeat the procedure for
$M=10$ times, but we fix the same width-$N^{*}$ network (terminated
at time $T_{e}$) for all $M$ repeats. Let $\tilde{\mathbf{y}}_{N}^{k}(T_{e},x)$
be the prediction of the untrained width-$N$ network at the training
image $x$, in the $k$-th repeat with the terminal time $T_{e}$.
Using these predictions, one computes:
\begin{align*}
\tilde{v}_{N}(T_{e}) & =\frac{1}{100M}\sum_{k=1}^{M}\sum_{x\in{\rm training\,set}}|\tilde{\mathbf{r}}_{N}^{k}(T_{e},x)|^{2},\\
\tilde{\mathbf{r}}_{N}^{k}(T_{e},x) & =\sqrt{N}\bigg(\tilde{\mathbf{y}}_{N}^{k}(T_{e},x)-\frac{1}{M}\sum_{k'=1}^{M}\tilde{\mathbf{y}}_{N}^{k'}(T_{e},x)\bigg).
\end{align*}
We take $\tilde{v}_{N}(T_{e})$ as an estimate for $\mathbf{E}\mathbb{E}_{Z}[|\tilde{G}^{y}\left(T_{e},X\right)|^{2}]$.

\paragraph*{Reproducibility.}

The experiments can be repeated using the codes in this link: \url{https://github.com/npminh12/nn-clt}
\end{document}